%% file: arxiv.tex
\documentclass{article}

\usepackage[main, final, nonatbib]{neurips_2026}

\usepackage[utf8]{inputenc} 
\usepackage[T1]{fontenc}    
\usepackage{url}            
\usepackage{booktabs}       
\usepackage{amsfonts}       
\usepackage{nicefrac}       
\usepackage{microtype}      
\usepackage[utf8]{inputenc} 
\usepackage[T1]{fontenc}    
\usepackage[dvipsnames]{xcolor}
\usepackage{microtype}
\usepackage{graphicx}
\usepackage{subcaption}
\usepackage{booktabs} 
\usepackage{tabularx}
\usepackage{multirow}
\usepackage{colortbl} 

\usepackage{float}

\usepackage{amsmath}
\usepackage{amssymb}
\usepackage{mathtools}
\usepackage{amsthm}
\usepackage{breqn}

\usepackage{booktabs}
\usepackage{caption}
\usepackage{array}

\usepackage{url}

\usepackage{marvosym}
\usepackage{nicefrac}
\usepackage{optidef}
\usepackage[hypertexnames=false]{hyperref}
\hypersetup{
    colorlinks=true,
    citecolor=blue,
    linkcolor=blue,
    urlcolor=teal,
    breaklinks=true
}
\input{math_commands}
\newtcolorbox{algbox}[2][]{
  enhanced,
  breakable,
  float=H, 
  colback=white,
  colframe=black!20,
  arc=3pt,
  width=0.48\columnwidth,
  right=2cm, 
  overlay unbroken and first={
    \node[anchor=north east] 
      at (frame.north east) {\textbf{Algorithm #2}};
  },
  #1
}
\definecolor{burgundy}{rgb}{0.5, 0.0, 0.13}
\tcbset{mytheo/.style={
    fonttitle=\sffamily\bfseries,
    enhanced,
    colframe=burgundy,
    colback=burgundy!2!white,
    colbacktitle=burgundy,
    boxrule=0pt,
    borderline south={2pt}{-2pt}{burgundy},
    left*=1em,  
    right*=1em, 
    theorem style=standard,
    sharp corners,
    before skip=8pt,
    after skip=10pt,
    before upper={\setlength{\parindent}{1em}}, 
    before lower={\setlength{\parindent}{1em}}  
}}
\theoremstyle{plain}
\newtheorem{theorem}{Theorem}[section]
\newtheorem{proposition}[theorem]{Proposition}
\newtheorem{lemma}[theorem]{Lemma}
\newtheorem{corollary}[theorem]{Corollary}

\theoremstyle{definition}
\newtheorem{definition}[theorem]{Definition}

\newtheorem{assumption}[theorem]{Assumption}

\theoremstyle{remark}
\newtheorem{remark}[theorem]{Remark}
\newtheorem{example}[theorem]{Example}

\tcolorboxenvironment{example}{exampleboxstyle}

\NewDocumentEnvironment{Example}{o}
  {%
    \begingroup
    \pushQED{\qed}%
    \IfNoValueTF{#1}
      {\begin{example}}
      {\begin{example}[#1]}%
  }
  {%
    \popQED
    \end{example}%
    \endgroup
  }

\zcRefTypeSetup{equation}{
  Name-sg = Equation,
  name-sg = equation,
  Name-pl = Equations,
  name-pl = equations,
  refbounds = {,(,),}
}

\zcRefTypeSetup{theorem}{
  Name-sg = Theorem,
  name-sg = theorem,
  Name-pl = Theorems,
  name-pl = theorems
}

\zcRefTypeSetup{definition}{
  Name-sg = Definition,
  name-sg = definition,
  Name-pl = Definitions,
  name-pl = definitions
}

\zcRefTypeSetup{problem}{
  Name-sg = Problem,
  name-sg = problem,
  Name-pl = Problems,
  name-pl = problems
}

\zcRefTypeSetup{fact}{
  Name-sg = Fact,
  name-sg = fact,
  Name-pl = Facts,
  name-pl = facts
}

\zcRefTypeSetup{proposition}{
  Name-sg = Proposition,
  name-sg = proposition,
  Name-pl = Propositions,
  name-pl = propositions
}

\zcRefTypeSetup{lemma}{
  Name-sg = Lemma,
  name-sg = lemma,
  Name-pl = Lemmas,
  name-pl = lemmas
}

\zcRefTypeSetup{corollary}{
  Name-sg = Corollary,
  name-sg = corollary,
  Name-pl = Corollaries,
  name-pl = corollaries
}

\zcRefTypeSetup{assumption}{
  Name-sg = Assumption,
  name-sg = assumption,
  Name-pl = Assumptions,
  name-pl = assumptions
}

\zcRefTypeSetup{claim}{
  Name-sg = Claim,
  name-sg = claim,
  Name-pl = Claims,
  name-pl = claims
}

\zcRefTypeSetup{remark}{
  Name-sg = Remark,
  name-sg = remark,
  Name-pl = Remarks,
  name-pl = remarks
}

\zcRefTypeSetup{example}{
  Name-sg = Example,
  name-sg = example,
  Name-pl = Examples,
  name-pl = examples
}

\zcRefTypeSetup{subsubsection}{
  Name-sg = Subsubsection,
  name-sg = subsubsection,
  Name-pl = Subsubsections,
  name-pl = subsubsections
}

\zcRefTypeSetup{subsection}{
  Name-sg = {\S},
  name-sg = {\S},
  Name-pl = {\S},
  name-pl = {\S}
}

\zcRefTypeSetup{section}{
  Name-sg = {Section},
  name-sg = {section},
  Name-pl = {Sections},
  name-pl = {sections}
}

\zcRefTypeSetup{chapter}{
  Name-sg = {Chapter},
  name-sg = {chapter},
  Name-pl = {Chapters},
  name-pl = {chapters}
}

\zcRefTypeSetup{appendix}{
  Name-sg = Appendix,
  name-sg = appendix,
  Name-pl = Appendices,
  name-pl = appendices
}

\zcRefTypeSetup{table}{
  Name-sg = Table,
  name-sg = table,
  Name-pl = Tables,
  name-pl = tables
}

\zcRefTypeSetup{figure}{
  Name-sg = Figure,
  name-sg = figure,
  Name-pl = Figures,
  name-pl = figures
}

\zcRefTypeSetup{algorithm}{
  Name-sg = Algorithm,
  name-sg = algorithm,
  Name-pl = Algorithms,
  name-pl = algorithms
}

\zcRefTypeSetup{item}{
  Name-sg = item,
  name-sg = item,
  Name-pl = items,
  name-pl = items
}

\newcounter{phase}
\renewcommand{\thephase}{\arabic{phase}}

\zcRefTypeSetup{phase}{
  Name-sg = Phase,
  name-sg = Phase,
  Name-pl = Phases,
  name-pl = Phases,
}
\newcommand{\zphasefromto}[2]{%
  \zcref[sort=false,pairsep={~to~}]{#1,#2}%
}
\newcommand{\PhaseHeading}[1]{%
  \textbf{Phase~\thephase\ifblank{#1}{}{~(#1)}:}%
}

\makeatletter
\newcount\phase@listdepth
\let\phase@saveditem\relax

\newenvironment{phases}[1][]{%
  \begin{itemize}[
    label=\textbullet,
    labelindent=1.5em,
    leftmargin=*,
    labelsep=.75em,
    itemsep=.35\baselineskip,
    topsep=.4\baselineskip,
    parsep=0pt,
    #1
  ]%
  \phase@listdepth=\@listdepth\relax
  \let\phase@saveditem\item
  \def\item{%
    \ifnum\@listdepth=\phase@listdepth
      \expandafter\phase@item
    \else
      \expandafter\phase@saveditem
    \fi
  }%
}{%
  \end{itemize}%
}

\newcommand{\phase@item}{%
  \@ifnextchar[%
    {\phase@@item}%
    {\phase@@item[]}%
}

\def\phase@@item[#1]{%
  \phase@saveditem
  \refstepcounter{phase}%
  \PhaseHeading{#1}\nobreakspace\ignorespaces
}
\makeatother

\AddToHook{env/theorem/begin}{%
  \zcsetup{countertype={theorem=theorem}}}

\AddToHook{env/proposition/begin}{%
  \zcsetup{countertype={theorem=proposition}}}

\AddToHook{env/lemma/begin}{%
  \zcsetup{countertype={theorem=lemma}}}

\AddToHook{env/corollary/begin}{%
  \zcsetup{countertype={theorem=corollary}}}

\AddToHook{env/claim/begin}{%
  \zcsetup{countertype={theorem=claim}}}

\AddToHook{env/conjecture/begin}{%
  \zcsetup{countertype={theorem=conjecture}}}

\AddToHook{env/definition/begin}{%
  \zcsetup{countertype={theorem=definition}}}

\AddToHook{env/problem/begin}{%
  \zcsetup{countertype={theorem=problem}}}

\AddToHook{env/assumption/begin}{%
  \zcsetup{countertype={theorem=assumption}}}

\AddToHook{env/remark/begin}{%
  \zcsetup{countertype={theorem=remark}}}

\AddToHook{env/example/begin}{%
  \zcsetup{countertype={theorem=example}}}
\usepackage[
  backend=biber,
  style=authoryear,
  sorting=ynt,        
  sortcites=true,     
  maxcitenames=2,
  mincitenames=1,
  uniquelist=false,
  maxbibnames=99,
  minbibnames=1,
  giveninits=true,
  isbn=false,
  doi=false,
  eprint=true,
  url=true,
  date=year,
  uniquename=init,
  backref=true
]{biblatex}
\DefineBibliographyStrings{english}{          
  andothers = {et al\adddot}
}
\DeclareNameAlias{author}{given-family}
\DeclareNameAlias{editor}{given-family}
\DeclareNameAlias{translator}{given-family}

\renewbibmacro*{byeditor+others}{}
\renewbibmacro*{publisher+location+date}{%
  \usebibmacro{date}%
}
\renewbibmacro*{in:}{}

\AtEveryBibitem{%
  \clearname{editor}%
  \clearlist{publisher}%
  \clearlist{location}%
}

\title{Training-Induced Escape from Token Clustering in a Mean-Field Formulation of Transformers}
\author{
Noboru Isobe$^{1}$\quad Daisuke Inoue$^{2}$\quad Masaaki Imaizumi$^{1,3,4}$
\vspace{1em}
\\
$^1$RIKEN \quad $^2$Imperial College London \quad $^3$The University of Tokyo \quad $^4$Kyoto University
\vspace{1em}
\\
\texttt{noboru.isobe@riken.jp, dinoue@ic.ac.uk, imaizumi@g.ecc.u-tokyo.ac.jp}
}

\begin{document}

\maketitle

\input{contents/abst_cr}

\input{contents/intro2_cr}
\input{contents/setting_cr}
\input{contents/theory_cr}
\input{contents/numerical_cr}
\input{contents/related_works_cr}
\input{contents/concluding_cr}

\section*{Acknowledgments and Disclosure of Funding}
N.I. was supported by JST ACT-X (JPMJAX25C2) and JSPS KAKENHI (25K24910).
M.I. was supported by JSPS KAKENHI (24K02904), JST CREST (JPMJCR21D2), FOREST (JPMJFR216I), ASPIRE (JPMJAP2329), and BOOST (JPMJBY24A9).
The authors declare no competing interests.

\begin{refcontext}[sorting=nyt]
\printbibliography
\end{refcontext}
\newpage
\appendix

\input{contents/proofs_cr}
\input{contents/numerical_added_cr}

\end{document}

%% file: math_commands.tex
\usepackage{physics2}
\usephysicsmodule{ab}   
\usephysicsmodule{ab.braket}    
\usephysicsmodule{diagmat}  
\usephysicsmodule[showleft=2,showtop=2]{xmat}   
\usepackage{fixdif}
\usepackage{derivative}
\usepackage{amsmath, amssymb, bm, mathrsfs, bbm}
\usepackage{mathtools}
\makeatletter
\@ifundefined{increment}{%
  \DeclareFontEncoding{LS1}{}{}%
  \DeclareFontSubstitution{LS1}{stix2}{m}{n}%
  \DeclareSymbolFont{stixTwoOperatorsForIncrement}{LS1}{stix2}{m}{n}%
  \DeclareMathSymbol{\increment}{\mathord}{stixTwoOperatorsForIncrement}{"CA}%
}{}
\makeatother
\usepackage{svg}
\usepackage{subcaption}
\usepackage{enumitem}
\DeclareMathAlphabet{\mathcal}{OMS}{cmsy}{m}{n}
\usepackage[most, breakable, skins, theorems, hooks,xparse]{tcolorbox}

\usepackage{amsthm}
\usepackage[linesnumbered,ruled,vlined,noend]{algorithm2e}
\SetAlgoNlRelativeSize{-1}
\usepackage{wrapfig,wrapstuff}
\mathtoolsset{showonlyrefs,showmanualtags}

\usepackage{zref-clever}
\zcsetup{
  cap,
  noabbrev,
  nameinlink=true,
  lastsep={, and\nobreakspace},
  tlastsep={, and\nobreakspace}
}

\usepackage{array}

\definecolor{olive}{rgb}{0.6, 0.6, 0.2}
\definecolor{sand}{rgb}{0.8666666666666667, 0.8, 0.4666666666666667}
\definecolor{wine}{rgb}{0.5333333333333333, 0.13333333333333333, 0.3333333333333333}
\definecolor{deblue}{RGB}{11,132,147}
\definecolor{ocra}{RGB}{204, 119, 34}
\definecolor{customColor}{HTML}{C0C6D4}
\definecolor{lightBlue}{HTML}{E0ECF8}
\definecolor{verylightgray}{gray}{0.95} 
\definecolor{mygreen}{rgb}{0.9, 1.0, 0.9} 

\tcbuselibrary{breakable, skins}
\newtcolorbox{CatchyBox}[2][]{
    lower separated=false,
    colback=white!80!sand!90!ocra,
    colframe=white, fonttitle=\bfseries,
    colbacktitle=white!50!sand!90!ocra,
    coltitle=black,
    enhanced,
    breakable,
    attach boxed title to top left={xshift=.02\linewidth,yshift=-4mm},
    title=#2,#1}
\newtcolorbox{CatchyBox2}[2][]{
    lower separated=false,
    colback=lightBlue,
    colframe=white, fonttitle=\bfseries,
    colbacktitle=lightBlue,
    coltitle=black,
    enhanced,
    breakable,
    attach boxed title to top left={xshift=.02\linewidth,yshift=-4mm},
    title=#2,#1}
\newtcolorbox{CatchyBox_green}[2][]{
    lower separated=false,
    colback=mygreen,
    colframe=white, fonttitle=\bfseries,
    colbacktitle=mygreen,
    coltitle=black,
    enhanced,
    breakable,
    attach boxed title to top left={xshift=.02\linewidth,yshift=-4mm},
    title=#2,#1}

\newtcolorbox{CatchyBox3}[2][]{
    lower separated=false,
    colback=verylightgray,
    colframe=white, fonttitle=\bfseries,
    colbacktitle=gray,
    coltitle=black,
    enhanced,
    breakable,
    attach boxed title to top left={xshift=.02\linewidth,yshift=-4mm},
    title=#2,#1}

\newtcolorbox{SketchOfProof}[1][]{
  blanker,
  breakable,
  left=1em,
  before skip=5pt,
  after skip=10pt,
  borderline west={2pt}{0pt}{gray},
  before upper={\textit{\textbf{Proof sketch of #1.}}~},
  after upper={\hfill\tikz[xscale=0.25,yscale=0.25]{%
    \draw (0,0)--(1,0)--(1,1)--(0,1)--cycle;}%
  },
}

\tcbset{
  exampleboxstyle/.style={
    blanker,
    breakable,
    left=1em,
    before skip=5pt,
    after skip=10pt,
    borderline west={2pt}{0pt}{black!15!white}
  }
}
\usepackage{fontawesome5}
\newtcolorbox{codecontribution}[1][]{
    lower separated=false,
    colback=gray!10!white,
    colframe=gray!50!black,
    fonttitle=\sffamily\bfseries,
    enhanced,
    title={\small Contributions},
    left=1mm, 
    right=10mm, 
    top=2mm, 
    bottom=2mm, 
    #1
}
\expandafter\let\expandafter\oldproof\csname\string\proof\endcsname
\let\oldendproof\endproof
\renewenvironment{proof}[1][\proofname]{%
  \oldproof[\bfseries\itshape #1] 
}{\oldendproof}


\def\1{\bm{1}}

\def\eps{{\varepsilon}}

\DeclareMathAlphabet{\mathsfit}{\encodingdefault}{\sfdefault}{m}{sl}
\SetMathAlphabet{\mathsfit}{bold}{\encodingdefault}{\sfdefault}{bx}{n}

\def\gA{{\mathcal{A}}}

\def\gC{{\mathcal{C}}}

\def\gU{{\mathcal{U}}}

\newcommand{\Laplacian}{\triangle}
\newcommand{\wLaplacian}{\Laplacian_{\barrho}}

\newcommand{\Var}{\mathrm{Var}}

\DeclareMathOperator{\softmax}{softmax}

\usepackage{bbm}
\newcommand*{\eval}[1]{\left.{#1}\right|}

\DeclareFontFamily{U}{matha}{\hyphenchar\font45}
\DeclareFontShape{U}{matha}{m}{n}{
<-6> matha5 <6-7> matha6 <7-8> matha7
<8-9> matha8 <9-10> matha9
<10-12> matha10 <12-> matha12
}{}
\DeclareSymbolFont{matha}{U}{matha}{m}{n}

\DeclareFontFamily{U}{mathx}{\hyphenchar\font45}
\DeclareFontShape{U}{mathx}{m}{n}{
<-6> mathx5 <6-7> mathx6 <7-8> mathx7
<8-9> mathx8 <9-10> mathx9
<10-12> mathx10 <12-> mathx12
}{}
\DeclareSymbolFont{mathx}{U}{mathx}{m}{n}

\DeclareMathDelimiter{\vvvert} {0}{matha}{"7E}{mathx}{"17}%

\DeclarePairedDelimiterX{\normiii}[1]
{\vvvert}
{\vvvert}
{\ifblank{#1}{\:\cdot\:}{#1}}
\usepackage{appendix}
\usepackage{cancel}

\SetMathAlphabet{\mathcal}{normal}{OMS}{cmsy}{m}{n} 
\SetMathAlphabet{\mathcal}{bold}{OMS}{cmsy}{m}{n} 
\DeclareSymbolFont{EulerExtension}{U}{euex}{m}{n}
\DeclareMathSymbol{\euintop}{\mathop} {EulerExtension}{"52}
\DeclareMathSymbol{\euointop}{\mathop} {EulerExtension}{"48}
\let\intop\euintop
\let\ointop\euointop
\makeatletter
\protected\def\int{\DOTSI\intop\ilimits@}
\protected\def\oint{\DOTSI\ointop\ilimits@}
\makeatother
\DeclareSymbolFont{cmletters}{OML}{cmm}{m}{it}
\DeclareSymbolFont{cmsymbols}{OMS}{cmsy}{m}{n}
\DeclareSymbolFont{cmlargesymbols}{OMX}{cmex}{m}{n}
\DeclareMathSymbol{\myjmath}{\mathord}{cmletters}{"7C}
\DeclareMathSymbol{\myamalg}{\mathbin}{cmsymbols}{"71}
\DeclareMathSymbol{\mycoprod}{\mathop}{cmlargesymbols}{"60}
\let\jmath\myjmath

\numberwithin{equation}{section}

\allowdisplaybreaks[4]

\usepackage{makecell}
\usepackage{tikz}
\usetikzlibrary{decorations,decorations.pathmorphing}
\usetikzlibrary{shapes, arrows.meta, positioning,bending}
\usetikzlibrary{calc}
\tikzset{
    ncbar angle/.initial=90,
    ncbar/.style={
        to path=(\tikztostart)
        -- ($(\tikztostart)!#1!\pgfkeysvalueof{/tikz/ncbar angle}:(\tikztotarget)$)
        -- ($(\tikztotarget)!($(\tikztostart)!#1!\pgfkeysvalueof{/tikz/ncbar angle}:(\tikztotarget)$)!\pgfkeysvalueof{/tikz/ncbar angle}:(\tikztostart)$)
        -- (\tikztotarget)
    },
    ncbar/.default=0.5cm,
}

\tikzset{square left brace/.style={ncbar=0.5cm}}
\tikzset{square right brace/.style={ncbar=-0.5cm}}

\tikzset{round left paren/.style={ncbar=0.5cm,out=120,in=-120}}
\tikzset{round right paren/.style={ncbar=0.5cm,out=60,in=-60}}
\usetikzlibrary{positioning, fit}   
\tikzset{block/.style={draw, thick, text width=2cm ,minimum height=1.3cm, align=center},   
line/.style={-latex}     
}  
\usepackage{pgfplots}
\pgfplotsset{compat=1.18}
\usepgfplotslibrary{fillbetween}
\definecolor{curveRed}{RGB}{210,20,35}
\definecolor{curveBlue}{RGB}{0,114,178}
\definecolor{bgSkyBlue}{RGB}{86,180,233}
\definecolor{bgGray}{RGB}{153,153,153}
\definecolor{bgOrange}{RGB}{230,159,0}
\definecolor{barRed}{RGB}{210,20,35}
\newcommand{\e}{\mathrm{e}}

\newcommand{\Pcal}{\mathcal{P}}
\newcommand{\Vocab}{\mathcal{V}}
\newcommand{\Normalize}{\mathcal{N}}

\newcommand{\R}{\mathbb{R}}
\newcommand{\Sphe}{\mathbb{S}}
\newcommand{\dSphe}{\Sphe^{d-1}}
\newcommand{\PdSphe}{\Pcal(\dSphe)}
\newcommand{\mubar}{\overline{\mu}}
\newcommand{\rhobar}{\overline{\rho}}
\newcommand{\Z}{\mathbb{Z}}

\newcommand{\Xred}{L^2_0}

\DeclareMathOperator{\KL}{KL}
\DeclareMathOperator{\Ent}{Ent}

\newcommand{\Rayleigh}{R_{\barrho}(g_\ell)}

\newcommand{\la}{\left\langle}
\newcommand{\ra}{\right\rangle}

\makeatletter

\newcommand{\pushright}[1]{\ifmeasuring@#1\else\omit\hfill$\displaystyle#1$\fi\ignorespaces}
\newcommand{\pushleft}[1]{\ifmeasuring@#1\else\omit$\displaystyle#1$\hfill\fi\ignorespaces}

\newcommand{\barmu}{\overline{\mu}}
\newcommand{\barrho}{\overline{\rho}}
\newcommand{\Energy}{{E}_{\varepsilon,A}}
\newcommand{\Energygap}{{\Delta}_{\varepsilon,A}}
\newcommand{\Dissipation}{D}
\newcommand{\regparam}{\lambda_{\mathrm{reg}}}
\newcommand{\frob}{\textup{F}}

\newcommand{\oneGD}{\textup{1-GD}}
\newcommand{\op}{\textup{op}}
\newcommand{\Proj}{P^\perp}
\newcommand{\bareps}{\overline{\varepsilon}}
\newcommand{\Nhd}{\mathcal{V}}
\newcommand{\Hop}{\mathsf{H}}
\newcommand{\Hbar}{\Hop}
\newcommand{\Kbar}{\mathsf{K}}

\newcommand{\Abar}{\mathsf{A}}
\newcommand{\Adbar}{\Abar^\dagger}

\newcommand{\Bbar}{\mathsf{B}}

\newcommand{\cGram}{\mathsf{W}^{\textup{c}}}

\newcommand{\Pbar}{\wLaplacian}

\newcommand{\Minimizers}{\Bab{\mubar^\pm}}
\makeatother
\DeclareMathOperator{\Crit}{Crit}
\DeclareMathOperator{\Div}{{div}}

\DeclareMathOperator{\dist}{{dist}}

\DeclareMathOperator*{\argmax}{{arg~max}}
\DeclareMathOperator*{\argmin}{{arg~min}}

\DeclareMathOperator*{\osc}{{osc}}
\DeclareMathOperator*{\supp}{{supp}}
\DeclareMathOperator*{\minimize}{{minimize}}

\DeclareMathOperator*{\diag}{{diag}}

\DeclareMathOperator{\Hessian}{{Hess}}

\def\Set#1{\Setdef#1\Setdef}
\def\Setdef#1|#2\Setdef{\left\{#1\,\;\mathstrut\vrule\,\;#2\right\}}%
\renewcommand{\th}{%
    \ifmmode
        ^\mathrm{th}%
    \else%
        \textsuperscript{th}\xspace%
    \fi%
}
\makeatletter
\newcommand{\subalign}[1]{%
  \vcenter{%
    \Let@ \restore@math@cr \default@tag
    \baselineskip\fontdimen10 \scriptfont\tw@
    \advance\baselineskip\fontdimen12 \scriptfont\tw@
    \lineskip\thr@@\fontdimen8 \scriptfont\thr@@
    \lineskiplimit\lineskip
    \ialign{\hfil$\m@th\scriptstyle##$&$\m@th\scriptstyle{}##$\hfil\crcr
      #1\crcr
    }%
  }%
}
\makeatother

%% file: contents/abst_cr.tex
\begin{abstract}
Transformers perform inference by iteratively transforming token representations across layers.
This layerwise computation has been studied empirically, and recent mean-field theories of Transformer dynamics explain how attention can drive token distributions toward clustering.
However, existing analyses treat model parameters as prescribed, leaving open how training reshapes this clustering picture.
We study this question in a noisy mean-field Transformer in which only a parameter-linear FFN is trained under $L^2$ regularization.
We identify a training-induced phase of the inference dynamics across layers: after clustering in the early layers, the token distribution can leave the clustered regime near the final layers. A single gradient-descent update already produces this escape, and minimizers of the training objective satisfy an exponential turnpike bound under which clustering persists through the interior layers.
Our analysis is based on an entropy-regularized interaction energy that captures the clustering bias of attention, for which we establish a double-well landscape and a local Polyak--\L{}ojasiewicz inequality.
More broadly, our results point toward a training-aware mean-field theory in which training and inference dynamics of Transformers are treated together.
\end{abstract}

%% file: contents/intro2_cr.tex
\section{Introduction}\label{sec:intro}
Modern language models built on the Transformer architecture~\parencite{vaswani2017attention} perform inference by iteratively transforming token representations across layers. 
This naturally suggests a dynamical viewpoint: along the depth axis, a Transformer defines a discrete-time dynamics of tokens.
Understanding the internal dynamics of Transformers is important both for post-hoc explanation and for mitigating risks arising from their lack of transparency as discussed in \parencite{ExplainablityLLM}.
Broad empirical literature has studied aspects of this layerwise computation and its mechanisms, including attention flow~\parencite{abnar-zuidema-2020-quantifying}, mechanistic interpretability~\parencite{bereska2024mechanistic}, and causal analysis and model editing~\parencite{Geiger21,Meng22locatingfactual}.
These approaches typically inspect particular trained models and their learned representations at a microscopic level.

Complementary to these empirical approaches, recent mathematical works in \parencite{lu2019understanding,sander22a,geshkovski2023clusters} seek macroscopic descriptions of Transformer dynamics by replacing the collection of token representations with their empirical distribution and studying its evolution across layers.
Within this line, model weights are typically prescribed rather than selected by training.
For prototypical self-attention models, fixed scalar or symmetric weights yield a gradient-flow structure for an attention-induced interaction energy and explain token clustering~\parencite{geshkovski2025perspective,burger2025analysis,kuehn2026spectralselectionsymmetricselfattention}.
Variants of this prescribed-parameter picture further analyze localization, synchronization, metastability, and multiscale transitions around clustering by adding feedforward blocks with fixed weights, initialization randomness~\parencite{alvarezlopez2026perceptronslocalizationattentionsmeanfield,polyanskiy2025synchronizationmeanfieldmodelscircle,criscitiello2026synchronizationcirclesspheresnonlinear,agazzi2026stochasticscalinglimitssynchronization,geshkovski2024dynamicmetastabilityselfattentionmodel,bruno2025emergence,bruno2025a}.

The missing piece is therefore a macroscopic account of \emph{training-selected} dynamics:
\begin{center}
\textbf{Q.} \emph{How does training reshape the mean-field dynamics of token distributions in Transformers?}
\end{center}
\begin{figure}[t]
   \hspace*{-5em}%
  \centering
  \resizebox{\linewidth}{!}{\input{contents/theory_figs/bootstrap_like_curve_snippet_v34}}
  \caption{Overview of our results.
  The y-axis is an informal distance from the token distribution to a clustered state; in the analysis, the corresponding proxy is  \(\Energygap\) in \zcref{sec:main}.
    Training can induce a phase of escaping from the clustered regime. 
    The convergence rate from \zphasefromto{phase:cluster}{phase:turnpike} and the escape rate from \zphasefromto{phase:turnpike}{phase:escape} are approximately given in \zcref{rmk:rate_a,prop:one-step-escape}, respectively.
}
  \label{fig:overview}
\end{figure}
As a first step toward this question, we study a simple noisy mean-field Transformer in which each layer \(t\) contains a feedforward network (FFN) $u_{W_t}$ with trainable parameters \(W_t\).
Our central finding is that the dynamics selected by training---after a single gradient-descent update (\zcref{prop:one-step-escape}) and, in the interior-turnpike form, under minimizers of the training objective (\zcref{thm:exp_turnpike})---differentiate the clustering picture into the $3$ phases as shown in \zcref{fig:overview}:
\begin{phases}
    \item[Initial clustering]\label{phase:cluster}
    the attention interaction pulls the token distribution toward a clustered state;

    \item[Turnpike plateau]\label{phase:turnpike}
    the distribution remains near the clustered state through most intermediate layers, forming a plateau;

    \item[Terminal escape]\label{phase:escape}
    in the final layers, the trained FFN can drive the distribution away from the plateau in a direction that decreases a loss function $\ell$.
\end{phases}
Our contribution lies in theoretically elucidating the mechanisms behind these phenomena:
\begin{enumerate}
    \item In \zcref{sec:energy}, we identify the clustering state through an attention-induced interaction energy.
    We prove that this energy is of the double-well type, providing the background for the persistence of  \zcref{phase:turnpike} in the intermediate layer.
    \item \zcref{sec:adjoint} provides a lower bound representing \zcref{phase:escape} that occurs when the parameter $W$ is updated only once by gradient descent (GD).
    The proof illustrates that \zcref{phase:escape} can be caused by a dynamics of gradients between layers through backpropagation.
    \item \zcref{sec:exp_turnpike} provides an upper bound that supports \zcref{phase:cluster,phase:turnpike,phase:escape} globally in layers for the \emph{training-selected inference dynamics}: the depth dynamics generated by a global minimizer of the training objective, as opposed to the trajectory of an optimization algorithm.
    Parameter optimality constrains the energy dissipation, which can trigger phase transitions.
\end{enumerate}
Rather than following the full optimization trajectory, these results analyze the two tractable ends of the training process---a single gradient-descent step and the minimizers of the objective---and carry complementary statements: a lower bound on the terminal energy lift, and an interior-turnpike upper bound. The finite-step regime in between is probed numerically.
Numerical experiments in \zcref{sec:numerical} further show that the strength of the turnpike plateau and terminal escape depends quantitatively on the geometric relation between the terminal loss and the interaction energy, as well as on the spectrum of the matrix governing the attention interaction.

\noindent\textbf{Notation.}
Let $\Z_+$ denote the set of positive integers.
For $N\in\Z_+$, $\bab{N}\coloneqq \Bab{1,2,\dots,N}$.
We denote by $\vab{\bullet}$ and $\aab{\bullet,\bullet}$ the Euclidean norm and inner product, respectively, and by $\Vab{\bullet}_{\frob}$ the Frobenius norm.
Let $\dSphe\coloneqq\Set{x\in\R^d | \vab{x}=1}\subset\R^d$ be the unit sphere with tangent space $T_x\dSphe\coloneqq\Set{v\in\R^d | \aab{v,x}=0}$ at $x\in\dSphe$.
Let $\nabla$, $\Div$, and $\Laplacian$ denote the gradient, divergence, and Laplace--Beltrami operator on $\dSphe$, respectively.
For an integer $k\ge 0$, $C^k(\dSphe)$ and $H^k(\dSphe)$ denote the standard spaces of $k$-times continuously differentiable and Sobolev functions on $\dSphe$, respectively. 
For a compact metric space $X$, $\Pcal(X)$ denotes the space of Borel probability measures on $X$, endowed with the narrow topology, which is metrized by the ($2$-)Wasserstein distance $W_2$.
For a Borel map $T\colon X\to X$ and $\mu\in\Pcal(X)$, $T_\#\mu\in\Pcal(X)$ denotes the pushforward of $\mu$ by $T$,

%% file: contents/theory_figs/bootstrap_like_curve_snippet_v34.tex
\providecolor{curveRed}{RGB}{210,20,35}
\providecolor{curveBlue}{RGB}{0,114,178}
\providecolor{bgSkyBlue}{RGB}{86,180,233}
\providecolor{bgGray}{RGB}{153,153,153}
\providecolor{bgOrange}{RGB}{230,159,0}
\providecolor{barRed}{RGB}{210,20,35}

\def\Tval{8}
\pgfmathsetmacro{\ToneThird}{\Tval/3}
\pgfmathsetmacro{\TtwoThird}{2*\Tval/3}
\pgfmathsetmacro{\TleftCenter}{\ToneThird/2}
\pgfmathsetmacro{\TmidCenter}{(\ToneThird+\TtwoThird)/2}
\pgfmathsetmacro{\TrightCenter}{(\TtwoThird+\Tval)/2}
\def\PhaseLabelY{3.70}
\def\PhaseImageY{0.72}
\def\PhaseImageWidth{2.15cm}
\def\PhaseTransitionWidth{1.25cm}
\def\SequenceArrowY{0.595}
\def\InnerTitleY{1.30}
\pgfmathsetmacro{\InnerLeftX}{2.55}
\pgfmathsetmacro{\InnerRightX}{5.50}
\pgfmathsetmacro{\TheoremCenterX}{\TmidCenter}
\def\TheoremTopY{1.68}
\def\TheoremBottomY{1.40}
\def\TheoremTwoTopY{1.05}
\def\TheoremTwoBottomY{0.82}
\def\ArrowTargetX{2.10}
\def\RedArrowTargetX{6.35}
\def\aval{1.05}
\def\ymin{0.055}
\def\amp{2.46}
\def\redampfactor{0.50}
\def\blueoffset{0.20}
\def\blueampfactor{1.50}
\pgfmathsetmacro{\ArrowTargetY}{\blueoffset + \ymin + \blueampfactor*\amp*(exp(-\aval*\ArrowTargetX) + exp(-\aval*(\Tval-\ArrowTargetX)))/(1 + exp(-\aval*\Tval))}
\pgfmathsetmacro{\RedArrowTargetY}{\ymin + \redampfactor*\amp*(exp(-\aval*\RedArrowTargetX) + exp(-\aval*(\Tval-\RedArrowTargetX)))/(1 + exp(-\aval*\Tval))}

\begin{tikzpicture}
\begin{axis}[
  width=15.5cm,
  height=7.8cm,
  scale only axis,
  xmin=0, xmax=\Tval,
  ymin=0, ymax=4.15,
  axis lines=box,
  clip=false,
  enlargelimits=false,
  domain=0:\Tval,
  samples=360,
  axis line style={black,line width=0.9pt},
  tick style={black,line width=0.9pt},
  major tick length=4.5pt,
  minor tick length=0pt,
  xmajorgrids=false,
  xminorgrids=false,
  ymajorgrids=true,
  yminorgrids=true,
  major grid style={black!22,line width=0.35pt},
  minor grid style={black!40,densely dotted,line width=0.25pt},
  minor x tick num=0,
  minor y tick num=4,
  xlabel={},
  ylabel={},
  xtick={0,\Tval},
  xticklabels={$0$,$T$},
  ytick=\empty,
  tick label style={font=\normalsize},
  xticklabel style={font=\Large,yshift=-1pt},
  yticklabel style={font=\normalsize},
]

\path[fill=bgSkyBlue,fill opacity=0.18,draw=none] (axis cs:0,0) rectangle (axis cs:\ToneThird,4.15);
\path[fill=bgGray,fill opacity=0.16,draw=none] (axis cs:\ToneThird,0) rectangle (axis cs:\TtwoThird,4.15);
\path[fill=bgOrange,fill opacity=0.18,draw=none] (axis cs:\TtwoThird,0) rectangle (axis cs:\Tval,4.15);

\addplot[black!22,line width=0.35pt] coordinates {(\ToneThird,0) (\ToneThird,4.15)};
\addplot[black!22,line width=0.35pt] coordinates {(\TtwoThird,0) (\TtwoThird,4.15)};


\addplot[
  name path=redLower,
  draw=none,
  domain=0:\Tval,
]
  {\ymin + \redampfactor*\amp*(exp(-\aval*x) + exp(-\aval*(\Tval-x)))/(1 + exp(-\aval*\Tval))};

\addplot[
  name path=blueUpper,
  draw=none,
  domain=0:\Tval,
]
  {\blueoffset + \ymin + \blueampfactor*\amp*(exp(-\aval*x) + exp(-\aval*(\Tval-x)))/(1 + exp(-\aval*\Tval))};

\addplot[
  draw=none,
  fill=black,
  fill opacity=0.10,
]
fill between[of=blueUpper and redLower];

\addplot[
  draw=curveRed,
  line width=1.5pt,
  line cap=round,
  line join=round,
  dash pattern=on 4pt off 3pt,
  domain=0:{0.5*\Tval},
]
  {\ymin + \redampfactor*\amp*(exp(-\aval*x) + exp(-\aval*(\Tval-x)))/(1 + exp(-\aval*\Tval))};

\addplot[
  draw=curveRed,
  line width=1.5pt,
  line cap=round,
  line join=round,
  domain={0.5*\Tval}:\Tval,
]
  {\ymin + \redampfactor*\amp*(exp(-\aval*x) + exp(-\aval*(\Tval-x)))/(1 + exp(-\aval*\Tval))};

\addplot[
  draw=curveBlue,
  line width=1.5pt,
  line cap=round,
  line join=round,
]
  {\blueoffset + \ymin + \blueampfactor*\amp*(exp(-\aval*x) + exp(-\aval*(\Tval-x)))/(1 + exp(-\aval*\Tval))};

\addplot[
  draw=black,
  line width=1.0pt,
  line join=round,
]
coordinates {
  (0.00,2.8278)
  (0.35,1.8589)
  (0.80,1.3696)
  (1.15,0.8546)
  (1.55,0.6752)
  (1.95,0.3932)
  (2.35,0.3576)
  (2.75,0.2256)
  (3.15,0.2174)
  (3.55,0.1638)
  (3.95,0.1823)
  (4.35,0.1669)
  (4.75,0.2260)
  (5.15,0.2225)
  (5.55,0.3461)
  (5.95,0.3867)
  (6.35,0.6313)
  (6.75,0.7672)
  (7.15,1.2847)
  (7.55,1.7239)
  (8.00,2.8012)
};

\node[
  anchor=north,
  font=\sffamily\fontsize{18}{21}\selectfont,
  align=center,
  inner sep=0pt
] at (axis cs:\TleftCenter,\PhaseLabelY) {\shortstack[c]{\zcref{phase:cluster}:\\ Clustering}};

\node[
  anchor=north,
  font=\sffamily\fontsize{18}{21}\selectfont,
  align=center,
  inner sep=0pt
] at (axis cs:\TmidCenter,\PhaseLabelY) {\shortstack[c]{\zcref{phase:turnpike}:\\ Turnpiking}};

\node[
  anchor=north,
  font=\sffamily\fontsize{18}{21}\selectfont,
  align=center,
  inner sep=0pt
] at (axis cs:\TrightCenter,\PhaseLabelY) {\shortstack[c]{\zcref{phase:escape}:\\ Escaping}};

\draw[
  draw=black!35,
  line width=6.0pt,
  line cap=rect,
  ->,>=latex
] (axis description cs:0.24,\SequenceArrowY) -- (axis description cs:0.43,\SequenceArrowY);

\draw[
  draw=black!35,
  line width=6.0pt,
  line cap=rect,
  ->,>=latex
] (axis description cs:0.57,\SequenceArrowY) -- (axis description cs:0.76,\SequenceArrowY);

\node[anchor=north,inner sep=0pt] (initialImg)
  at (axis description cs:0.1667,\PhaseImageY)
  {\includegraphics[width=\PhaseImageWidth]{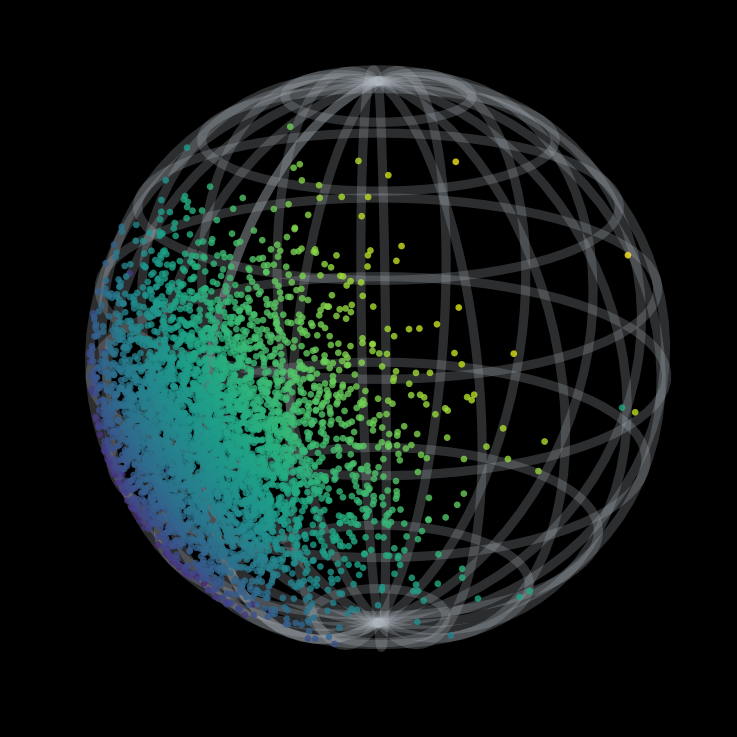}};

\node[anchor=center,inner sep=0pt] (initialToSteadyImg)
  at (axis description cs:0.3333,\SequenceArrowY)
  {\includegraphics[width=\PhaseTransitionWidth]{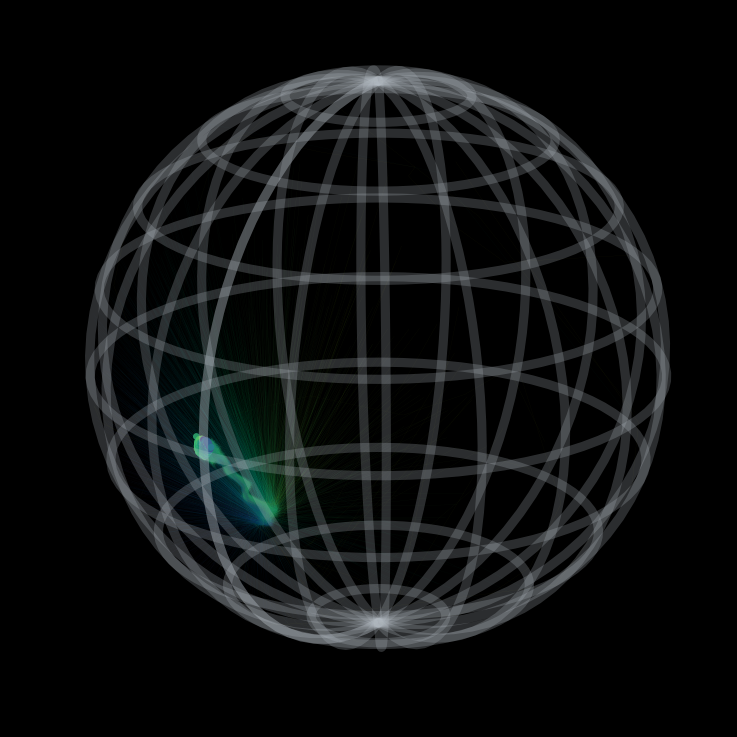}};

\node[anchor=north,inner sep=0pt] (steadyImg)
  at (axis description cs:0.5000,\PhaseImageY)
  {\includegraphics[width=\PhaseImageWidth]{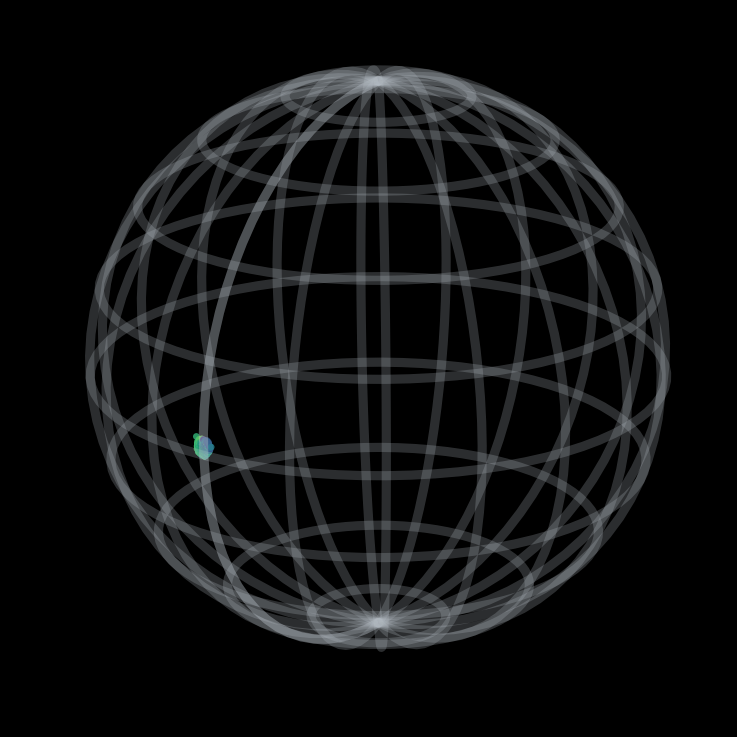}};

\node[anchor=center,inner sep=0pt] (steadyToTerminalImg)
  at (axis description cs:0.6667,\SequenceArrowY)
  {\includegraphics[width=\PhaseTransitionWidth]{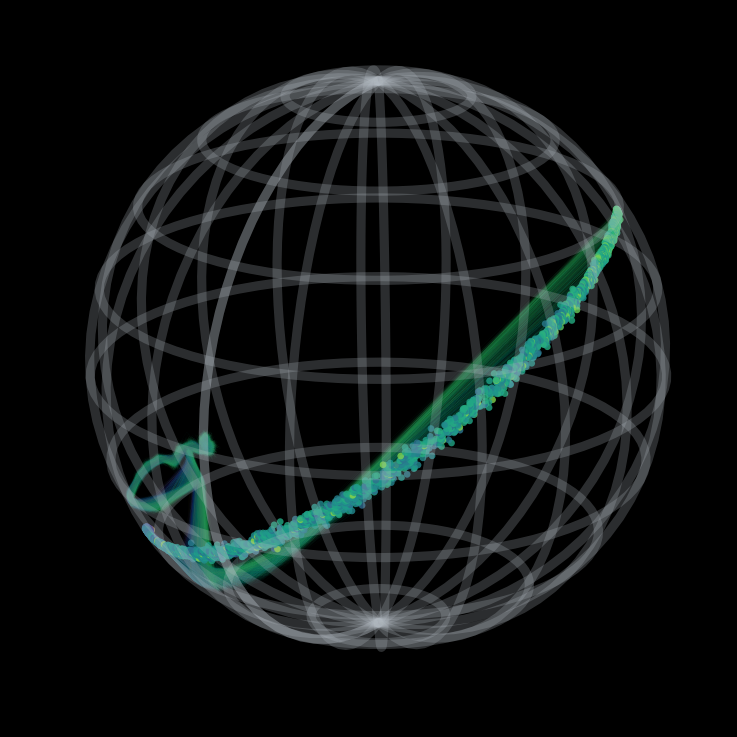}};

\node[anchor=north,inner sep=0pt] (terminalImg)
  at (axis description cs:0.8333,\PhaseImageY)
  {\includegraphics[width=\PhaseImageWidth]{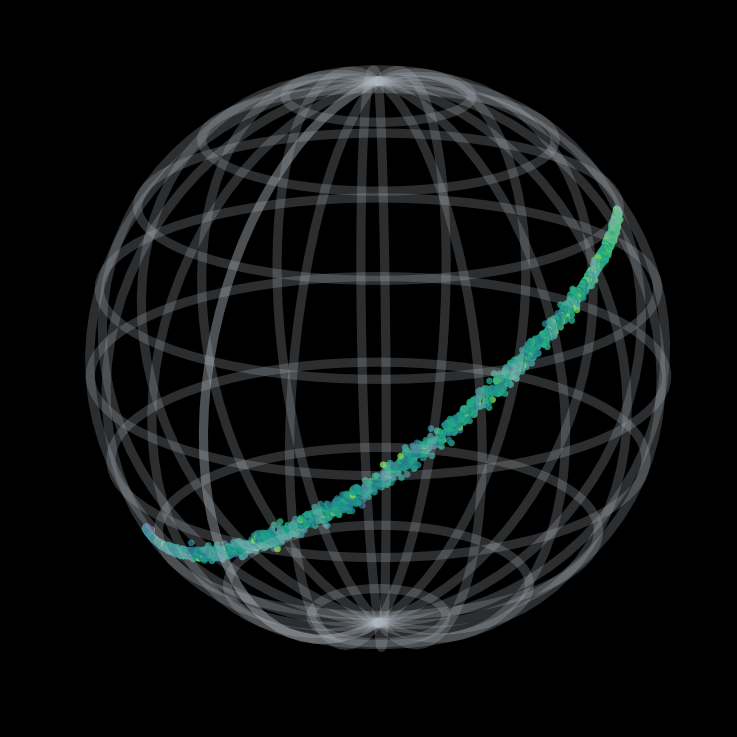}};

\node[
  anchor=center,
  font=\sffamily\large,
  inner sep=0pt,
  align=center
] (theoremInnerTitle) at (axis cs:\InnerLeftX,\InnerTitleY) {\zcref{thm:exp_turnpike}};

\node[
  anchor=north,
  font=\sffamily\scriptsize,
  inner sep=0pt,
  align=center
] (theoremInnerSubtitle) at ([yshift=-1pt] theoremInnerTitle.south) {Exponential turnpike-type bound};

\draw[
  ->,
  >=latex,
  draw=curveBlue,
  line width=1.1pt,
  line cap=round,
  line join=round
] ([yshift=-2pt] theoremInnerSubtitle.south)
  -- (axis cs:\ArrowTargetX,\ArrowTargetY);

\node[
  anchor=center,
  font=\sffamily\large,
  inner sep=0pt,
  align=center
] (propInnerTitle) at (axis cs:\InnerRightX,\InnerTitleY) {\zcref{prop:one-step-escape}};

\node[
  anchor=north,
  font=\sffamily\scriptsize,
  inner sep=0pt,
  align=center
] (propInnerSubtitle) at ([yshift=-1pt] propInnerTitle.south) {Adjoint analysis at 1-GD step};

\draw[
  ->,
  >=latex,
  draw=curveRed,
  line width=1.1pt,
  line cap=round,
  line join=round
] ([yshift=-2pt] propInnerSubtitle.south)
  -- (axis cs:\RedArrowTargetX,\RedArrowTargetY);

\node[
  anchor=north,
  font=\Large,
  inner sep=0pt
] at (axis description cs:0.5,-0.03) {Layer $t$};

\node[
  anchor=south,
  rotate=90,
  font=\Large,
  align=center,
  inner sep=0pt
] at (axis description cs:-0.03,0.5)
  {Distance$(\approx\Energygap)$};

\end{axis}
\end{tikzpicture}

%% file: contents/setting_cr.tex
\section{Framework and problem setting}
\label{sec:framework}

For theoretical analysis, we model the mean-field dynamics of tokens in a Transformer by a non-local Fokker--Planck equation.
Training is then formulated as an optimization problem over the parameters appearing in the drift term, which corresponds to tuning the feedforward network inside each layer.

\subsection{Modeling token dynamics by a non-local Fokker--Planck equation}
\label{subsec:token-dynamics}

We first recall a discrete-time dynamical system representation of a Transformer and then introduce the mean-field model studied in this paper.
According to \parencite{sander22a,geshkovski2025perspective}, we can regard the Transformer as mapping an input sequence of (embedded) tokens $(x_i)_{i=1}^N$ to the terminal values $(x_i(L))_{i=1}^N$ of a discrete-time dynamical sequence $((x_i(l))_{i=1}^N)_{l=1}^L$ satisfying $x_i(1)=x_i$ and
\begin{equation}
\begin{aligned}
\textstyle    x_i(l+1)=\Normalize\pab[Big]{x_i(l)+u_{W(l)}(x_i(l))+\sum\limits_{j=1}^N\e^{\beta\la{Q(l)x_i(l)},{K(l)x_j(l)}\ra}V(l)x_j(l)},&& l\in[L-1],
\end{aligned}
    \label{eq:discrete-transformer-dynamics}
\end{equation}
where $\Normalize\colon\R^d\to\dSphe$ is the norm normalization and $u_W$ is an FFN with weight matrices $W(l)$ for each layer $l$; for example, $u_W(x) = W_1 \tanh(W_2x)$ where $W = (W_1, W_2)$, $W_1, W_2 \in \R^{d \times d}$.
In the theoretical analysis, we specify the parametrization of $u_W$ in \zcref{ass:linear-ffn-parametrization} below.
The coefficient $\beta$ is a positive constant, and the matrices $Q(l)$, $K(l)$, $V(l)$ are weight matrices that define the query, key, and value, respectively.
In the above discrete update the FFN enters additively; in a standard pre-layer-normalization (pre-LN) block it composes with the attention sub-block, and under the residual scaling $h=T/L$ the composed and additive forms coincide to first order---a Lie--Trotter splitting in the sense of \parencite{lu2019understanding}---with formal accumulated error $O(1/L)$ over $L$ layers.
Previous works in, e.g., \parencite{sander22a,Karagodin24causal,geshkovski2025perspective,bruno2025emergence,burger2025analysis} have analyzed the case of \emph{Symmetric unnormalized self-attention (S-USA)}, in which the weight matrices are shared across layer $l$ and satisfy $Q^\top K=V=V^\top$.
The S-USA model is introduced as a useful theoretical assumption for analyzing the token clustering phenomenon. 
Following the literature, we consider a noisy continuous-time counterpart of the model with a (continuous) depth $T>0$ of layer as follows:
\begin{equation}
\begin{aligned}
\textstyle
         \d X_i(t)
    = \Proj_{X_i(t)}\pab[Big]{u_{W(t)}(X_i(t))+\dfrac1N \sum\limits_{j=1}^N \e^{\la  X_i(t), A X_j(t)\ra} AX_j(t)}\d t
    + \sqrt{2\eps}\d B_i(t),
    && t\in(0,T).
\end{aligned}
    \label{eq:noisy-particle-dynamics}
\end{equation}
Here $X_i$ is a stochastic process that models $x_i$, $\Proj_x\coloneqq I-x x^\top\in\R^{d\times d}$    is the projection onto  $T_x\dSphe$ at $x\in\dSphe$, $A\in \R^{d\times d}$ is a symmetric matrix denoting $Q^\top K$ in S-USA, $B_i$ are independent Brownian motions on $\dSphe$, and $\eps>0$ controls the strength of the injected noise.
We set $\beta=1$ for simplicity.

By the standard argument of the propagation of chaos as in \parencite{Chaintron_2022}, the law $\mu_t$ of the tokens $X_t$ at each layer $t\in[0,T]$ follows the following non-local Fokker--Planck equation
\begin{equation}
    \left\{
    \begin{aligned}
    \partial_{t} \mu_t(x)+\Div\pab{\mu_t(x) P_x^\perp(\chi_A[\mu]+u_{W_t})(x)}&=\eps\Laplacian\mu_t(x),&& t\in(0,T),x\in\dSphe, \\ 
    \left.\mu_{t}\right|_{t=0}&=\mu_{0}, 
    \end{aligned}
    \label{eq:ODE-Net}
    \right.
\end{equation}
when the number $N$ of tokens is large and the input tokens $(x_i)$ are sampled i.i.d.~from a token distribution $\mu_0\in\PdSphe$.
Here $\chi_A[\mu]\colon\dSphe\to \R^d$ is the vector field given by
\begin{equation}
    \begin{aligned}
    \chi_A[\mu](x)\coloneqq \int_{\dSphe}\e^{\aab{x,Ay}}Ay\d\mu(y) ,&& x\in \dSphe.
    \end{aligned}
\end{equation}
Given a fixed input token distribution $\mu_0$, we denote by $\mu^W=(\mu_t^W)_{t\in[0,T]}\in C((0,T);\PdSphe)$ the (weak) solution to \eqref{eq:ODE-Net} with initial condition $\mu_0$, corresponding to the parameter $W=(W_t)_{t\in[0,T]}$.
We call the terminal distribution $\mu_T^W\coloneqq\eval{\mu_t^W}_{t=T}$ the output of the mean-field Transformer.
An early measure-level continuity-equation formulation of token dynamics, together with a Wasserstein-gradient-flow interpretation for Sinkhorn-normalized attention, goes back to \textcite{sander22a}.

\begin{remark}[Role of the noise $\eps$]
Prior work on noisy Transformer dynamics has focused on simpler settings, such as scalar $Q,K,V$ without an FFN \parencite{balasubramanian2025structurestationarysolutionsmckeanvlasov,rigollet2026meanfielddynamicstransformers}.
Here, noise is included for two theoretical purposes: it provides parabolic regularization, which facilitates smooth solutions $\mu_t$, and it helps eliminate spurious stationary distributions that can arise in deterministic settings, cf. \parencite{chen2025quantitativeclusteringmeanfieldtransformer,agazzi2026stochasticscalinglimitssynchronization}.
At the modeling level, the i.i.d. Brownian motions in \eqref{eq:noisy-particle-dynamics} can be viewed as a simple model of small errors entering each token's computation independently, such as numerical error.
The Brownian motions in \eqref{eq:noisy-particle-dynamics} are independent across tokens. Models in which the noise is shared across tokens~\parencite{agazzi2026stochasticscalinglimitssynchronization,fedorov2026clusteringdeepstochastictransformers,koubbi2026homogenizedtransformers} lead instead to stochastic PDE limits and are not covered by the analysis below.
Whether the three-phase picture below holds at $\eps=0$ is theoretically open, although it is observed numerically (\zcref{sec:numerical}).
\end{remark}

\subsection{Mean-field optimal control formulation of training}
\label{subsec:optimal-control}
We now formulate training as an optimal control problem for the mean-field dynamics introduced above.
Let $\ell\colon\Pcal(\dSphe)\to\R$ be a (narrowly-)continuous loss function with bounded oscillation $\osc\ell\coloneqq\sup\ell-\inf\ell<+\infty$.
We optimize the layer-dependent FFN parameters $W_t$ so as to minimize the terminal loss with an $L^2$ regularization parameter $\regparam>0$:
\begin{equation}
    \minimize_{W\in L^2(0,T)} J(W)\coloneqq\ell(\mu_T^W)+\frac{\regparam}{2}\Vab{W}^2_{L^2(0,T)},
    \label{eq:mf-optimal-control}
\end{equation}
where $L^2(0,T)$ is the function space of square integrable functions $t\mapsto W_t$ on $[0, T]$ that take values in the parameters of $u_{W_t}$ and $\Vab{W}_{L^2(0,T)}^2\coloneqq{\int_0^T\Vab{W_t}_\frob^2\d t}$.
Under plain gradient descent, the $L^2$ penalty is equivalent to weight decay, a standard regularizer in Transformer training.

\begin{Example}[Bag-of-words loss]\label{eg:next-token-prediction}
Consider predicting Bag-of-Words proposed in \parencite{ma-etal-2018-bag} using the output token distribution $\mu_T$ of a mean-field Transformer.
Let $\Vocab$ be a finite set of vocabulary words, and let $R\colon\R^{d} \to \R^{\Vocab}$ be an affine transformation that maps each output token to word logits; we fix this transformation.
A loss function $\ell = \ell_{\textup{BoW}}$ with the Bag-of-Words distribution $q \in \Pcal(\Vocab)$ as the target can be modeled in the following form
\begin{equation}
    \ell_{\textup{BoW}}(\mu) \coloneqq-\sum_{w\in\Vocab}q(w)\log\pab{\int\softmax(R(x))_w\d\mu(x)}.
    \label{eq:BoW-loss}
\end{equation}
For finite tokens $(x_i)_{i=1}^N$, this becomes $-\sum_w q(w)\log\pab{\frac1N\sum_i\softmax(R(x_i))_w}$.
\end{Example}

The above formulation is closely related to the optimal-control view of ResNets and Neural ODEs~\parencite{E2017,E2018,JiequnHan_Li_2022}.
\begin{remark}[Other possible formulations of training]
    \label{rmk:train_setting}
    More practical learning settings, such as the next-token prediction task, can also be addressed by reformulating \eqref{eq:ODE-Net} using a distribution that includes labels (or token positions), following the approach in \parencite{E2018,BONNET2023113161}.
    We adopt the formulation given in \eqref{eq:mf-optimal-control} to align with the literature on mean-field Transformers.
\end{remark}

%% file: contents/theory_cr.tex
\section{Main mathematical results}
\label{sec:main}
In this section, we investigate what kind of dynamics, along the depth direction, are exhibited by the solution \(\mu^{W^\ast}\) generated by a layer-wise FFN parameter path $W^\ast=(W^\ast_t)_{t\in[0,T]}\in L^2(0,T)$ that minimizes the objective in \eqref{eq:mf-optimal-control}.

In the same spirit as the analyses of \parencite{sander22a,geshkovski2025perspective,burger2025analysis}, one observes that \eqref{eq:ODE-Net} admits the following gradient-flow-type structure associated with the entropy-regularized interaction energy \(\Energy\):
\begin{equation}
    \begin{aligned}
        &\partial_t \mu
        + \Div\pab{
            \mu\Proj\pab{
                -\nabla\fdv{\Energy}{\mu}
                + u_{W_t}
            }
        }=0,
        &&\Energy(\mu)
        \coloneqq
        \eps\Ent(\mu)
        -\frac12\iint
            \e^{\aab{x,Ay}}\d\mu(x)\d\mu(y),
        \label{eq:gradient_flow_form}
    \end{aligned}
\end{equation}
where $\nabla\fdv{}{\mu}$ is the (formal) Wasserstein gradient operator, $\Ent\colon\PdSphe\to\R\cup\Bab{+\infty}$ is the entropy defined by $\Ent(\rho\d\omega)=\int\rho\log\rho\d\omega$ for a probability density $\rho$ with the uniform distribution $\omega$ on $\dSphe$, and by $+\infty$ otherwise.
If \(W=0\), the token distribution $\mu_t$ moves solely in the direction of decreasing the interaction energy \(\Energy\), and is expected to concentrate into a small number of clusters.

On the other hand, in the variational problem \eqref{eq:mf-optimal-control}, the parameter path \(W\) is optimized so that the terminal distribution \(\mu_T^W\) makes the loss \(\ell\) small.
Hence, there are two effects: the Attention, which decreases \(\Energy\) and drives clustering of the token distribution, and the FFN \(u_{W_t}\), which attempts to move \(\mu_T^W\) in a direction that decreases the loss \(\ell\).
From these two effects, we will find the following three phases to emerge in the trained dynamics:
\begin{itemize}
    \item[{\bf\zcref{phase:cluster}}] In the first layers \(t\in(0,\tau)\) with \(\tau\ll T\), the Attention $\chi_A$ is dominant, and the token distribution lowers \(\Energy\) while localizing near a small number of points.
    This phase helps us to save on the regularization cost $\int_0^\tau\Vab{W_t}^2_\frob\d t \approx 0$.
    \item[{\bf\zcref{phase:turnpike}}] In the middle layers \(t\in[\tau,T-\tau]\), the token distribution stays in a plateau state near the set of stationary points along which \(\Energy\) is almost constant.
    These dynamics are induced by the tendency to keep the cost
        \(
            \int_{\tau}^{T-\tau}\Vab{W_t}_\frob^2\d t
        \)
    relatively small.
    \item[{\bf\zcref{phase:escape}}]In terminal layers \(t\in[T-\tau,T]\), the token distribution can leave the plateau to decrease the loss \(\ell\).
    This can be achieved because the FFN \(u_{W_t}\) can dominate the attention \(\chi_A\) over the short layer interval \([T-\tau,T]\).
\end{itemize}
In this section, we mathematically formulate and verify the above dynamical three-phase structure in the regime of large depth, namely \(T\gg1\).
In \zcref{sec:energy}, we identify the clustered distribution $\barmu$ that appears in \zcref{phase:turnpike}. This allows us to formulate the situation in \zcref{phase:cluster,phase:escape} using the energy gap $\Energygap\coloneqq\Energy-\min\Energy$.
In \zcref{sec:adjoint}, we consider parameters obtained by applying gradient descent (GD) for a single step and demonstrate the mechanism by which \zcref{phase:escape} can be triggered.
\zcref{sec:exp_turnpike} provides an exponential turnpike-type bound for a minimizer of the variational problem \eqref{eq:mf-optimal-control}.

\input{contents/theorems/landscape_cr}
\input{contents/theorems/adjoint_cr}
\input{contents/theorems/exp_turnpike_cr}


%% file: contents/theorems/landscape_cr.tex
\subsection{Landscape of the energy \texorpdfstring{$\Energy$}{E} in the mean-field Transformer}\label{sec:energy}

In this subsection, we study the stationary points of the energy $\Energy$ and the landscape around them to formulate \zcref{phase:turnpike}.
In general, not only local minima but also saddle points of \(\Energy\) exist continuously.
For example, considering the case where $\eps\searrow0$ and $A=I$, it is known that the minimizers of $\Energy$ are arbitrary Dirac delta distributions and that other saddle points may also exist continuously \parencite{geshkovski2025perspective, burger2025analysis,bruno2025a}.
To eliminate such spurious situations, we consider the following generic case:

\begin{assumption}[Non-degenerate positive attention matrix]
\label{ass:nondegenerate-A}
    The symmetric matrix \(A\in\R^{d\times d}\) in \eqref{eq:ODE-Net} is positive definite and the maximum eigenvalue is simple, i.e., the eigenvalues $(\lambda_i)_{i=1}^d\subset\R$ of $A$ sorted in descending order satisfy $\lambda_1\gneq\lambda_2\ge \cdots\ge\lambda_d\gneq 0$.
    Without loss of generality, one can assume that the standard orthonormal basis $(e_i)_{i=1}^d\subset\R^d$ diagonalizes $A$. 
\end{assumption}

\zcref{ass:nondegenerate-A} can be checked directly from the spectrum of $A$. Under this assumption, the following statement shows that $\Energy$ has exactly two minimizers and that no non-minimizing stationary point is path-connected to a minimizer within the stationary set:
\begin{proposition}[Landscape analysis of $\Energy$]
\label{prop:energy-landscape}
Under \zcref{ass:nondegenerate-A}, let $e_1\in\dSphe$ be the principal eigenvector corresponding to $\lambda_1$.
Then, there exists a constant $\bareps=\bareps(A,d)$ depending on $A$ and $d$ such that, for any $\eps\in(0,\bareps)$, there exists a neighborhood $\Nhd_+\subset\PdSphe$ of the Dirac delta $\delta_{e_1}\in\PdSphe$ such that the following holds:
\begin{itemize}
    \item  $\Energy$ has a unique minimizer $\mubar^+=\barrho^+\omega$ in $\Nhd_+$, which is a solution to the equation \(\barrho^+(x)\propto\e^{\frac1\eps\int\e^{\aab{x,Ay}}\barrho^+(y)\d\omega(y)}\) with density $\barrho^+$.
    In particular, $\barrho^+$ is positive and $C^\infty$.
    \item Let $S\colon\dSphe\to\dSphe$ be the antipodal map $ x \mapsto -x$, and let $\mubar^- = S_\#\mubar^+$.
    Then, the set of global minimizers of $\Energy$ consists only of $\mubar^+$ and $\mubar^-$, and the Hessian operator $\Hop$ at $\barmu\in\Bab{\mubar^+,\mubar^-}\eqqcolon\Minimizers$ is bounded for $h\in L^2_0(\barmu)\coloneqq\Set{h\in L^2(\mubar)|\int h \d\mubar=0}$ as
    \begin{equation}
    \begin{aligned}
        \eps\Vab{h}^2_{L^2(\barmu)}\ge\aab{h,\Hop h}_{L^2(\barmu)}\ge\frac{\eps}{2}\pab{1-\frac{\lambda_2}{\lambda_1}}\Vab{h}^2_{L^2(\barmu)},
    \end{aligned}
    \label{eq:PD_Hessian}
    \end{equation}
    where
    \(
        \Hop h\coloneqq\eps h-\varPi_0\pab{\int\e^{\aab{x,Ay}}h(y)\d\mubar(y)}
    \)
    with the projection $\varPi_0$ onto $L^2_0(\barmu)$.

    \item Let $\Crit\Energy$ be the set of critical points of $\Energy$, which consists of solutions of the Gibbs-type equation, and let \(\Delta_{\textup{crit}}(\eps)\coloneqq\inf\Set{\Energygap(\mu)|\mu\in\Crit\Energy\setminus\Minimizers}\).
    Then
    \begin{equation}\label{eq:1st_energy_gap}
    \liminf_{\eps\searrow0}\Delta_{\textup{crit}}(\eps)
    \ge
    \Delta_0\coloneqq\frac{\min\Bab{\sinh\lambda_1,\e^{\lambda_1}-\e^{\lambda_2}}}{2}>0.
    \end{equation}
    \end{itemize}
\end{proposition}

Informally speaking, \zcref{prop:energy-landscape} means that unless the energy gap of the input token distribution $\mu_0$ is large enough to exceed $\Delta_0$, $\mu_0$ is in a region that is attracted to the set of minimizers $\Bab{\mubar^\pm}$. 
This observation will be useful in proving the key lemma in \zcref{sec:exp_turnpike}.

%% file: contents/theorems/adjoint_cr.tex
\subsection{Adjoint analysis in the first step in gradient descent}\label{sec:adjoint}

In this section, we investigate the principles underlying \zcref{phase:escape}.
Specifically, we show that even for a model trained for just a single step of GD, an energy gap can grow near the final layer.

For simplicity, consider the following situation: the input token distribution is the minimizer $\mubar\coloneqq\mubar^+$ without loss of generality, and the loss function is sufficiently smooth; specifically, assume that the linear functional derivative \(g_\ell\coloneqq\fdv{\ell}{\mu}[\barmu]\colon\dSphe\to\R\), defined in the sense of \parencite{Carmona2018}, is $C^\infty$.
Suppose the parameters in GD are initialized to $W^{\boldsymbol{0}}\coloneqq0$.
Note that in this case $\mu^{W^{\boldsymbol{0}}}_t\equiv \mubar$.
Furthermore, for the FFN $u_W$, we assume the parameter-linear parameterization used in \parencite{Barboni22,Marion23,Scagliotti2023,IsobeOkumura24}:
\begin{assumption}[Linear-in-parameters FFN]
\label{ass:linear-ffn-parametrization}
    The FFN $u_W\colon\R^d\to\R^d$ has the form
    \(
                u_W(x)=W\sigma(x),
        W\in\R^{d\times p},
    \)
    where \(p\in\Z_+\) is the number of hidden features and \(\sigma\colon\R^d\to\R^p\) is a smooth function with a constant $C_\sigma\coloneqq\nicefrac{\Vab{\sigma}_{L^\infty(\dSphe)}^2}{2}<\infty$.
    Thus, the parameters to be optimized in \eqref{eq:mf-optimal-control} become an $\R^{d \times p}$-valued $L^2$ path $(W_t)_{t\in(0,T)}\in L^2(0,T;\R^{d \times p})$.
\end{assumption}
Let us consider the solution $\mu^{W^{\oneGD}}$ to equation \eqref{eq:ODE-Net} under the initial value $\mu_0=\mubar$ and the parameters $W^{\oneGD}\coloneqq W^{\boldsymbol{0}}-\alpha\nabla J(0)$ updated by one step with a learning rate $\alpha>0$.
Recall that the (Fréchet) gradient $\nabla J(0)\in L^2(0,T)$ of $J(W)$ with respect to $W\in L^2(0,T)$ can be described via the adjoint method, which is the continuous-time version of backpropagation (see \parencite{NeuralODE,Stefano20} and \zcref{sec:proof_adjoint}).
We obtain the following expression:
\begin{lemma}\label{lem:1GDparam}
Under \zcref{ass:linear-ffn-parametrization} and the regularity of $g_\ell$, $W^{\oneGD}\in L^2(0,T;\R^{d\times p})$ is given by 
\begin{equation}\label{eq:1GDparam}
\begin{aligned}
    W^{\oneGD}_t=-\alpha\int\nabla\phi_t(x)\sigma(x)^\top\d\bar{\mu}(x)\in\R^{d\times p}, &&t\in(0,T),
\end{aligned}
\end{equation}
where $(\phi_t)_{t\in[0,T]}\in C^1([0,T];C(\dSphe))\cap C([0,T];C^2(\dSphe))$ is the unique (classical) solution of the {backward weighted diffusion equation} given by
\begin{equation}\label{eq:bkwd_wDiffusion}
\begin{aligned}
     \partial_t\phi_t+\Hop\wLaplacian\phi_t=0, &&\phi_T=g_\ell, 
\end{aligned}
\end{equation}
with $\wLaplacian\phi\coloneq{\rhobar}^{-1}\Div(\rhobar\nabla\phi)$ for $\phi\in H^2(\dSphe)\cap L^2_0(\mubar)$.
\end{lemma}
Roughly speaking, $\phi_t$ and $\nabla\phi_t$ correspond to the linear functional and Wasserstein gradient of the loss $\ell$ at layer $t$, namely $\fdv{\ell}{\mu_t}$ and $\nabla\fdv{\ell}{\mu_t}$, respectively. The fact that this follows from \eqref{eq:bkwd_wDiffusion} implies that, as shown in \zcref{fig:token-dist-adjoint}, the gradient can increase exponentially as we approach the final layer.
However, this does not hold if $g_\ell=0$, i.e., if the loss $\ell$ is stationary at the minimizer $\mubar$ of $\Energy$.
Based on this analysis, the following estimate follows: the energy gap may increase near the final layer:
\begin{CatchyBox}{}

\begin{proposition}[Final-layer escape in one-step gradient descent]
\label{prop:one-step-escape}
Suppose \zcref{ass:nondegenerate-A,ass:linear-ffn-parametrization}.
Let $\mu^{W^{\oneGD}}$ be the solution to equation \eqref{eq:ODE-Net} with initial value $\mu_0=\mubar$ under the parameter $W^{\oneGD}$ in \eqref{eq:1GDparam}.
Then, for every $r>0$ there exist constants $\overline{\alpha}, C_R > 0$, depending on $r$ as in \zcref{prop:small-alpha-nonlinear-response}, such that for any $\alpha \in (0, \overline{\alpha})$ and $t \in (T - r, T)$, it holds that
        \[
            \Energygap(\mu_{t}^{W^\oneGD})
            \ge
            \frac{1}{4}
            \pab{
            \alpha{\varOmega_{t,T}(g_\ell)}
            \Vab{g_\ell}_{\Hop^{-1}}
            \exp\pab{-(T-t)\Rayleigh}
            -\alpha^2{C_R}}^2_+.
        \]
Here, if \(g_\ell=0\), the right-hand side is understood as \(0\), $\Vab{g_\ell}_{\Hop^{-1}} \coloneqq \sqrt{\aab{g_\ell, \Hop^{-1}g_\ell}}$, and $\varOmega_{t,T},R_{\barrho}$ is defined for $g_\ell\neq0$ by
\begin{equation}
\label{eq:Rayleighratio}
\begin{aligned}
    \varOmega_{t,T}(g_\ell) \coloneqq\frac{\Vab{\nabla J(0)}_{L^2(0,t)}^2}{\Vab{\phi_{t}}^2_{\Hop^{-1}}},&&\Rayleigh \coloneqq \frac{\Vab{\nabla g_\ell}^2_{L^2(\mubar)}}{\Vab{g_\ell}_ {\Hop^{-1}}^2}=\left.{\Vab{\nabla\fdv{\ell}{\mu}[\mubar]}^2_{L^2(\mubar)}}\middle/ {\Vab{\fdv{\ell}{\mu}[\mubar]}_{\Hop^{-1}}^2}\right..
\end{aligned}
\end{equation}
\end{proposition}
\end{CatchyBox}
The proof is given in \zcref{sec:proof_adjoint}, where we use the Jensen inequality as a key step.
This result shows that, already in the initial steps of GD, energy can increase exponentially with respect to $t$ near the terminal layer.
The Rayleigh quotient $\Rayleigh$ governs the rate of escape.
\begin{figure}[t]
  \centering

  \begin{minipage}[t]{0.49\linewidth}
    \centering
    \resizebox{\linewidth}{!}{\input{contents/theory_figs/adjint_state}}
    \captionof{figure}{Token distribution $\mu_t$ and adjoint gradient $\phi_t$.}
    \label{fig:token-dist-adjoint}
  \end{minipage}
  \hfill
  \begin{minipage}[t]{0.49\linewidth}
    \centering
    \includegraphics[width=\linewidth]{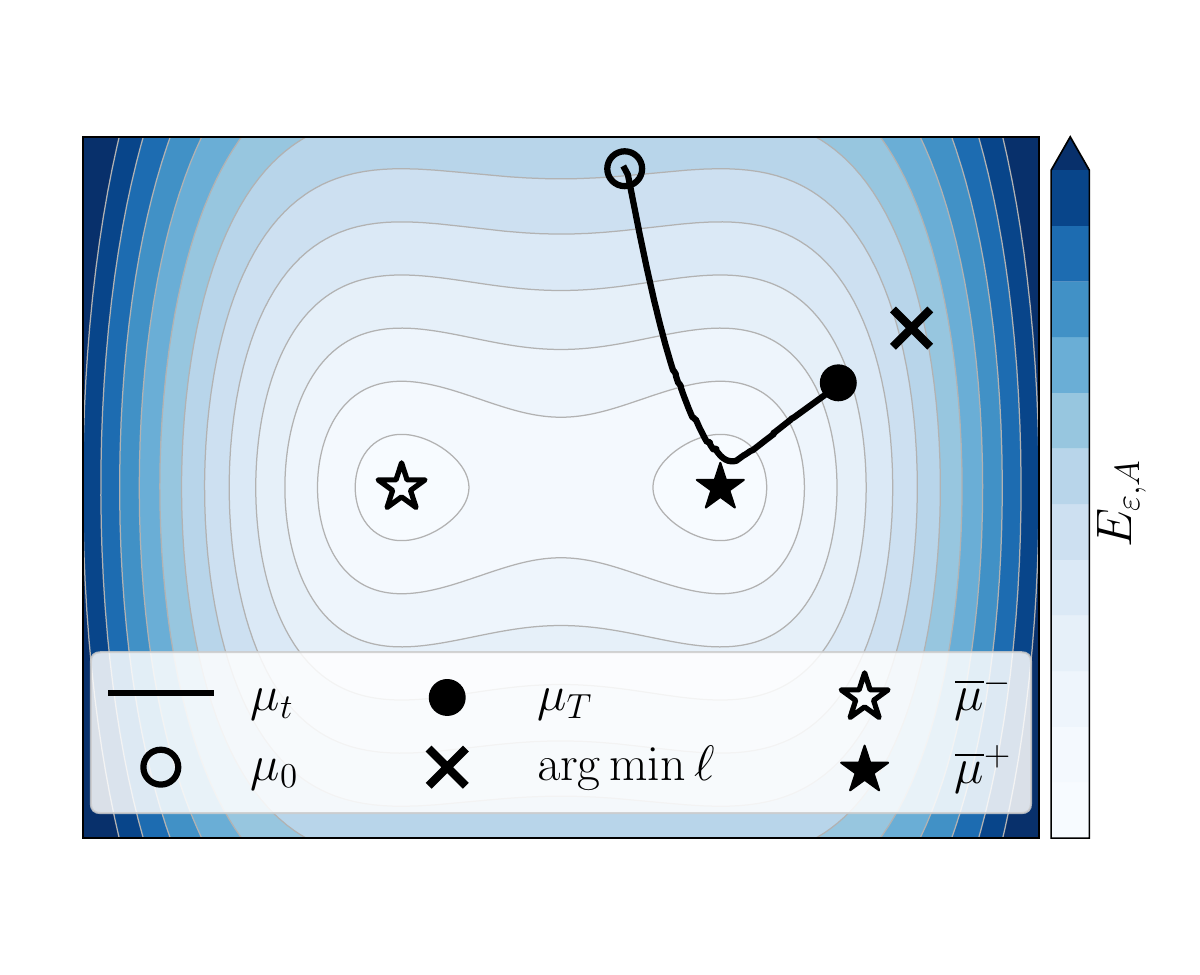}
    \captionof{figure}{Energy-landscape view of terminal escape.
    }
    \label{fig:double-well}
  \end{minipage}

\end{figure}

%% file: contents/theory_figs/adjint_state.tex

\definecolor{stateblue}{RGB}{0,116,188}
\definecolor{adjred}{RGB}{213,94,0}
\definecolor{softred}{RGB}{213,94,0}

\begin{tikzpicture}[
  x=1cm,y=1cm,
  >=Stealth,
  line cap=round,line join=round,
  every node/.style={font=\sffamily},
  statecurve/.style={black,line width=1.15pt},
  adjcurve/.style={adjred,line width=1.20pt,dash pattern=on 5pt off 3pt},
  dot/.style={circle,fill=black,inner sep=0pt,minimum size=7.0pt},
  rdot/.style={circle,fill=adjred,inner sep=0pt,minimum size=7.0pt},
  dashedloss/.style={stateblue,dashed,line width=1.05pt,-{Stealth[length=2.9mm,width=2.2mm]}},
  bluedash/.style={stateblue,dashed,line width=1.05pt},
  gradarrow/.style={stateblue,line width=2.4pt,-{Stealth[length=4.2mm,width=3.2mm]}},
  redarrow/.style={adjred,line width=1.20pt,dash pattern=on 5pt off 3pt,-{Stealth[length=3.6mm,width=2.8mm]}},
  axisline/.style={black,line width=0.70pt},
  tick/.style={black,line width=0.85pt},
  paneltext/.style={black,font=\sffamily\fontsize{19}{21}\selectfont},
  timetext/.style={black,font=\sffamily\fontsize{16}{18}\selectfont},
  mathnote/.style={black!65,font=\sffamily\fontsize{16}{18}\selectfont},
  dlabel/.style={stateblue,font=\sffamily\fontsize{16}{18}\selectfont},
  alabel/.style={softred,font=\sffamily\fontsize{16}{18}\selectfont}
]

\path[use as bounding box] (-0.30,-3.10) rectangle (13.20,4.30);

\coordinate (z0)  at (1.55,2.95);
\coordinate (zi)  at (4.85,1.75);
\coordinate (zip) at (7.45,1.39);
\coordinate (zN)  at (11.40,1.20);

\coordinate (a0)  at (1.55,1.20);
\coordinate (ai)  at (4.85,2.05);
\coordinate (aip) at (7.45,2.67);
\coordinate (aN)  at (11.40,2.95);

\draw[statecurve,-{Stealth[length=3.5mm,width=2.6mm]}] plot[smooth] coordinates {
  (1.55,2.95) (2.34,2.54) (3.13,2.22) (4.11,1.92) (4.85,1.75) (6.08,1.54) (7.45,1.39) (8.64,1.31) (9.82,1.25) (10.71,1.22) (11.40,1.20)
};

\foreach \p in {z0,zN}{\node[dot] at (\p) {};}

\node[mathnote,anchor=south] at ($(z0)+(0.00,0.28)$) {$\mu_0$};
\node[mathnote,anchor=south] at ($(zN)+(0.00,0.28)$) {$\mu_T$};

\draw[rounded corners=2pt,fill=white,draw=black!25] (1.25,-1.40) rectangle (11.95,-0.52);
\draw[statecurve] (1.70,-0.96) -- (2.70,-0.96);
\node[anchor=west,font=\sffamily\fontsize{16}{18}\selectfont] at (2.95,-0.96) {token dist. $\mu$};
\draw[adjcurve] (6.55,-0.96) -- (7.55,-0.96);
\node[anchor=west,text=adjred,font=\sffamily\fontsize{16}{18}\selectfont] at (7.80,-0.96) {gradient $\phi_t$};

\draw[redarrow] plot[smooth] coordinates {
  (11.40,2.95) (10.61,2.54) (9.82,2.22) (8.84,1.92) (7.66,1.67) (6.28,1.47) (4.90,1.35) (3.72,1.28) (2.73,1.23) (2.04,1.21) (1.55,1.20)
};

\foreach \p in {a0,aN}{\node[rdot] at (\p) {};}

\node[alabel,anchor=south] at ($(a0)+(0.00,0.28)$) {$\phi_0$};
\node[alabel,anchor=south] at ($(aN)+(0.00,0.30)$) {$\textstyle\phi_T=\frac{\delta\ell}{\delta\mu}$};


\draw[axisline,-{Stealth[length=3.4mm,width=2.6mm]}] (-0.25,0.35) -- (13.12,0.35);
\foreach \x in {1.55,4.85,7.45,11.40}{\draw[tick] (\x,0.04) -- (\x,0.65);}

\node[timetext,anchor=north] at (1.55,0.12) {$0$};
\node[timetext,anchor=north] at (11.40,0.12) {$T$};
\node[black,font=\sffamily\fontsize{22}{24}\selectfont] at (12.85,-0.22) {$t$};

\end{tikzpicture}

%% file: contents/theorems/exp_turnpike_cr.tex
\subsection{Exponential turnpike bound for optimized dynamics}\label{sec:exp_turnpike}

Finally, we study the token distribution \(\mu_t^{\ast}\coloneqq\mu_t^{W^{\ast}}\) induced by a global minimizer \(W^{\ast}\) of the variational problem \eqref{eq:mf-optimal-control}.
We prove an exponential turnpike-type bound showing that \(\Energygap(\mu_t^{\ast})\) becomes exponentially small away from the two boundary layers. 
\zcref{fig:double-well} schematically shows an optimized trajectory \((\mu_t^\ast)_t\): after staying near \(\Minimizers\), it escapes toward lower \(\ell\) while retaining that clustering bias, and reaches the output \(\mu_T^\ast\).
The theorem below formalizes the corresponding plateau--boundary-layer structure.
As in \parencite{IsobeOkumura24}, the variational problem \eqref{eq:mf-optimal-control} admits a minimizer under \zcref{ass:linear-ffn-parametrization}.
To state the theorem, we impose that the initial energy gap and the oscillation of the terminal loss are not too large:
\begin{assumption}[energy-threshold conditions]\label{ass:nolargegap}
    The initial energy gap $\Energygap(\mu_0)$ and  $\osc\ell$ satisfy
    \(
       \mathfrak E_{\textup{max}}\coloneqq \Energygap (\mu_0)+\frac{2C_\sigma}{\regparam}\osc\ell<\Delta_0,
       \osc\ell<\frac{\regparam}{8C_\sigma}\Delta_0
    \)
    where $\Delta_0$ is given by \zcref{prop:energy-landscape}.
\end{assumption}
\zcref{ass:nolargegap} is explicit in the initial energy gap, $\osc\ell$, $\regparam$, $C_\sigma$, and the margin $\Delta_0$.
It excludes initial distributions with little bias toward a clustered configuration (e.g., near-uniform ones)---consistent with the biased token representations the model is meant to describe---and it is the only place where closeness to the minimizers enters: convergence toward $\Minimizers$ is not assumed.

The next lemma formulates the main consequence of \zcref{ass:nolargegap}: for sufficiently long depth $T$, $\mu^\ast_t$ enters the sublevel set \(\Bab{\Energygap(\mu)\le\eta}\) at some time.
For a small $\eta>0$, the sublevel set has exactly two connected components.
We denote by \(\gC_+(\eta),\gC_-(\eta)\) the components containing \(\mubar^+\) and \(\mubar^-\), respectively.
The next lemma makes this precise.
\begin{CatchyBox3}{}
\begin{lemma}[entry into one low-energy component]\label{lem:same-component-entry}
Assume \zcref{ass:nolargegap}.
Then, there exists $\bareps>0$ such that for every \(\eps\in(0,\bareps)\), there exist \(\eta, r_\ast>0\) and \(L_\eps<\infty\) such that the following holds: For every sufficiently large \(T\) and every optimal pair \((\mu^{\ast},W^{\ast})\), there exist times \(s_-\in[0,L_\eps]\), \(s_+\in[T-L_\eps,T]\), and a sign \(\varsigma\in\{+,-\}\) such that \( \mu_{s_-}^{\ast},\mu_{s_+}^{\ast}\in\gU_{\varsigma}\cap\gC_{\varsigma}(\eta)\) where $\gU_\varsigma\subset\PdSphe$ is defined for each $\varsigma$ by
\begin{equation}\label{eq:local-chart-basins-def}
\gU_+\coloneqq\Set{(1+h)\mubar^+\in\PdSphe|h\in C^\infty(\dSphe), \int h\d\mubar^+=0, \Vab{h}_{L^\infty}<r_\ast},
\end{equation}
and set \(\gU_-\coloneqq S_\#\gU_+\).
\end{lemma}
\end{CatchyBox3}

\begin{SketchOfProof}[\zcref{lem:same-component-entry}]
By \zcref{cor:slope-gap-from-explicit-barrier}, the first inequality in
\zcref{ass:nolargegap} gives a positive constant \(d_\eta>0\) such that
\[\textstyle
\Energygap(\mu)\le\mathfrak E_{\textup{max}},
\qquad
\mu\notin\gC_+(\eta)\sqcup\gC_-(\eta)
\quad\Longrightarrow\quad
\Dissipation(\mu)
\coloneqq
\int\vab{\nabla\fdv{\Energy}{\mu}[\mu]}^2\d\mu
\ge d_\eta .
\]
Moreover, \(J(W^{\ast})\le J(0)\) gives
\(
\frac{\regparam}{2}\Vab{W^{\ast}}_{L^2(0,T)}^2\le\osc\ell ,
\)
and the gradient-flow structure yields
\(
\int_0^T\Dissipation(\mu_t^{\ast})\d t
\le
2\Energygap(\mu_0)+\frac{4C_\sigma}{\regparam}\osc\ell .
\)
Hence \zcref{lem:regularized-set-entry} gives times
\(s_-\in[0,L_\eps]\) and
\(s_+\in[T-L_\eps,T]\) at which the trajectory enters \(\gC_+(\eta)\sqcup\gC_-(\eta)\).
The two entries cannot be in different components of $\gC_\pm(\eta)$.
Indeed,
\zcref{cor:finite-temp-sign-changing-barrier} gives
\(
\inf_{\int x_1\d\mu=0}\Energygap(\mu)\ge\frac34\Delta_0,
\)
while \(\eta\le\Delta_0/4\).
Therefore, switching components requires an energy increase of at least \(\Delta_0/2\), hence an \(L^2\)-cost of \(W^{\ast}\) larger than the terminal-loss improvement allowed by the second inequality in \zcref{ass:nolargegap}.  
Thus \zcref{lem:same-branch-from-sign-barrier} places the two entries in the same component.
Finally, \zcref{lem:parabolic-chart-entry} upgrades these entries to the neighborhoods of densities: \(\mu_{s_-}^{\ast},\mu_{s_+}^{\ast}\in\gU_{\varsigma}\cap\gC_{\varsigma}(\eta)\).
On \(\gU_{\varsigma}\), the Hessian coercivity in \zcref{prop:energy-landscape} gives the local Polyak--\L{}ojasiewicz (PL) inequality.
\end{SketchOfProof}

We impose regularity on \(\ell\) as follows to use the optimality condition of \eqref{eq:mf-optimal-control} and its linearization:

\begin{assumption}[smooth loss]\label{ass:ctrl-samewell}
The loss \(\ell\) is of class $C^2$ on an \(L^2(\omega)\)-open set containing all H\"older-continuous probability densities when regarded as a functional of the density via \(\ell(\rho\omega)\).
\end{assumption}

Under \zcref{ass:ctrl-samewell}, every \((\mu^\ast, W^\ast)\) satisfies the optimality condition \(\regparam W_t^{\ast}=-\int\nabla \phi_t\sigma^\top\d\mu^\ast_t\)---the fixed-point form of the one-step update \eqref{eq:1GDparam}---where the adjoint state \(\phi\) solves the backward equation \eqref{eq:bkwd_wDiffusion} linearized along \(\mu^\ast\) rather than at \(\mubar\).
The Hessian coercivity of \zcref{prop:energy-landscape} and the weighted Poincar\'e inequality for \(\mubar\) make this backward dynamics exponentially stable, with rate and constant independent of \(T\); hence \(\phi\) and \(W^{\ast}\) are exponentially small away from the terminal layer.
Substituting this bound on \(W^{\ast}\) (\zcref{cor:control-tail}) into the local PL inequality on \(\gU_\pm\) yields the following exponential turnpike bound, whose full proof is given in \zcref{sec:proof_expturnpike}.
Informally, for every sufficiently deep network, every optimal token flow stays exponentially close in energy to one clustered state, except in two boundary layers of depth-independent width at the input and at the output.
\begin{CatchyBox}{}
\begin{theorem}\label{thm:exp_turnpike}
    Fix \(\eps\in(0,\bareps)\).
    Suppose \zcref{ass:nondegenerate-A,ass:linear-ffn-parametrization,ass:nolargegap,ass:ctrl-samewell} hold.
    Then there exist \(T^\ast,C,a>0\) such that for every \(T>T^\ast\) and every optimal pair
    \((\mu^{\ast},W^{\ast})\),
    \[
    \begin{aligned}
            \Energygap(\mu_t^{\ast})
    \le
    C
    \pab{\e^{-a t}+\e^{-a(T-t)}},
    && t\in[0,T].
    \end{aligned}
    \]
    \end{theorem}
\end{CatchyBox}

\begin{remark}[Dependence of $a$ on $A$]\label{rmk:rate_a}
The rate \(a\) is the smaller of the local PL rate and the backward decay rate of \eqref{eq:bkwd_wDiffusion} along \(\mu^\ast\).
Both depend on \(A\) through \(1-\nicefrac{\lambda_2}{\lambda_1}\) and \(\lambda_1\e^{-\lambda_1}\), as made explicit in \zcref{rem:explicit-turnpike-constants}.
\end{remark}

\begin{remark}[A linearized weak-attention model]\label{rmk:weak-attention}
To exhibit the three phases in a setting where \eqref{eq:mf-optimal-control} can be solved explicitly, we consider the weak-attention regime.
Replacing $A$ by $\beta A$ and expanding in $\beta$, the energy becomes $\Energy(\mu)=\eps\Ent(\mu)-\frac{\beta}{2}\aab{m_\mu,Am_\mu}+O(\beta^2)$ up to a constant, where $m_\mu\coloneqq\int x\d\mu$, so only $m_\mu$ interacts at leading order.
For $\sigma\equiv1$ in \zcref{ass:linear-ffn-parametrization} and $\ell(\mu)=\frac12\vab{m_\mu-b}^2$, linearizing around $\omega$ gives $\dot m_j=-r_jm_j+\frac{d-1}{d}w_j$ with $r_j\coloneqq(d-1)(\eps-\beta\lambda_j/d)$, and \eqref{eq:mf-optimal-control} becomes a finite-dimensional linear--quadratic problem.
For $\beta\lambda_1<\eps d$, its solution is \(m_j^\ast(t)=m_j(0)\e^{-r_jt}+\theta_j\e^{-r_j(T-t)}\frac{1-\e^{-2r_jt}}{1-\e^{-2r_jT}}\) for some $\theta_j\in\R$ bounded uniformly in $T$.
\end{remark}

%% file: contents/numerical_cr.tex
\section{Numerical experiments}\label{sec:numerical}

We validate the turnpike behavior through particle simulations of
\eqref{eq:noisy-particle-dynamics}.
We simulate \(N=64\) particles up to horizon \(T=80\), and work in a projected
\(d=2\) setting with \(A=\diag(1,\lambda_2)\).  The particles
are initialized in a cap around \(e_1=(1,0)\), so the experiments track how the
learned control moves an initially clustered token distribution along depth.
Implementation details are given in
\zcref{app:numerical-details}.
We compare two terminal losses: an align-target loss and a Bag-of-Words (BoW)
loss as in \zcref{eg:next-token-prediction}.  The align-target loss is
\(\ell_{\textup{align}}(\mu)=1-\aab{\int x\d\mu(x),u_{\textup{tar}}}\), with
\(u_{\textup{tar}}=e_2=(0,1)\); it encourages the first moment to point toward a
target direction away from the initial cap
center.  For the BoW loss, we use the finite-vocabulary version of
\eqref{eq:BoW-loss} with logits \(R(x)_w=\aab{x,c_w}\), uniform target \(q\),
and candidate vectors $c_w\in\Sphe^{1}$ sampled from a von Mises--Fisher distribution with
concentration \(\kappa_{\textup{cand}}\); this tests how the geometry of the readout candidates affects the terminal-layer escape.
Before the parameter sweeps below, \zcref{app:numerical-experiment0} shows
that the turnpike profile appears only after optimizing the controls.
We report three diagnostics.  First, we plot the signed angle
\(\theta\) relative to the reference direction, so
\(\theta\approx0\) indicates alignment with the interior plateau.  Second, we
compute the energy \(\Energy\) from \eqref{eq:gradient_flow_form} along the
simulated trajectories.  Third, we fit the terminal energy lift to
\(C\exp(-a(T-t))\) and compare the fitted rate \(a\) with the theoretical rate
computed numerically as the Rayleigh quotient \(\Rayleigh\) in
\zcref{prop:one-step-escape}.

\noindent\textbf{Experiment 1: Attention matrix structure induces the turnpike.}
We fix the align-target loss and vary the eigenvalue \(\lambda_2\) of
\(A=\diag(1,\lambda_2)\).  In
\zcref{fig:numerical-lambda2}(a), the angular trajectories show that smaller
\(\lambda_2\) keeps the particles aligned near \(\theta=0\) through most of
the depth and then produces a sharper terminal displacement.  The
energy curves in \zcref{fig:numerical-lambda2}(b) show the same effect at the
level of \(\Energy\): smaller \(\lambda_2\) creates a deeper and more
persistent interior plateau before the terminal lift.  Finally,
\zcref{fig:numerical-lambda2}(c) compares the fitted terminal exponent \(a\)
with the numerically computed theoretical rate \(\Rayleigh\).  Their shared
dependence on \(\lambda_2\) indicates that the spectral gap of the
attention matrix controls the strength of the turnpike and the sharpness of
the final-layer escape.

\begin{figure}[t]
    \centering
    \begin{subfigure}[t]{0.33\linewidth}
        \centering
        \includegraphics[height=0.23\textheight,width=0.85\linewidth,keepaspectratio]{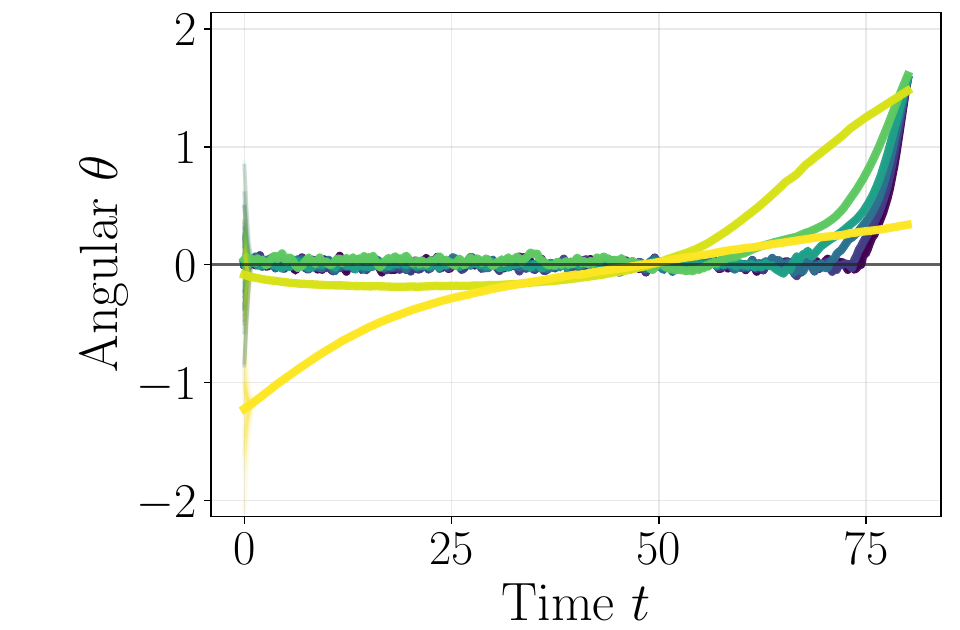}
        \caption{Angular trajectories.}
        \label{fig:numerical-lambda2-angular}
    \end{subfigure}\hfill
    \begin{subfigure}[t]{0.33\linewidth}
        \centering
        \includegraphics[height=0.23\textheight,width=0.85\linewidth,keepaspectratio]{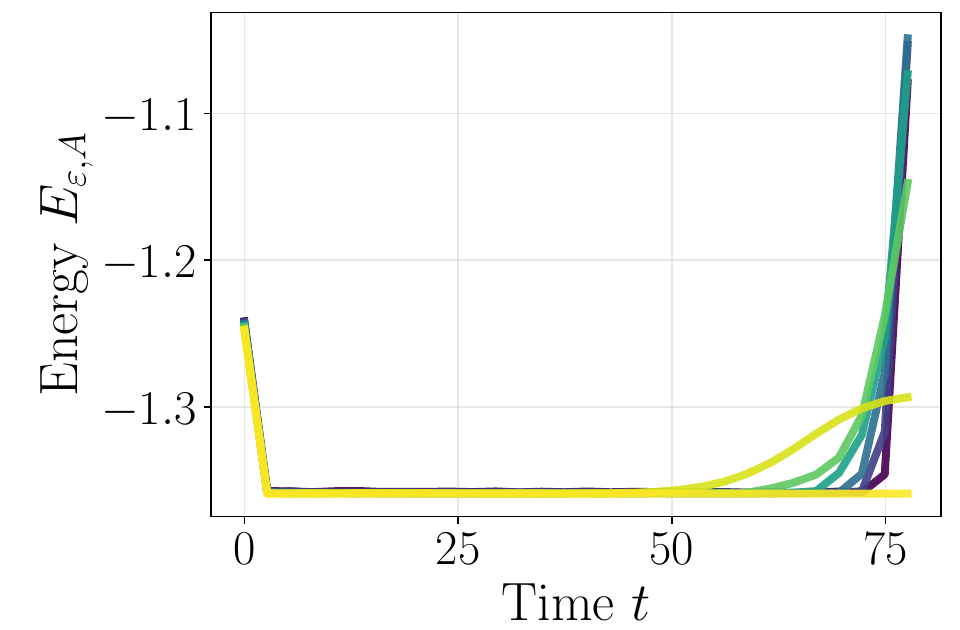}
        \caption{Energy trajectories.}
        \label{fig:numerical-lambda2-energy}
    \end{subfigure}\hfill
    \begin{subfigure}[t]{0.33\linewidth}
        \centering
        \includegraphics[height=0.23\textheight,width=0.85\linewidth,keepaspectratio]{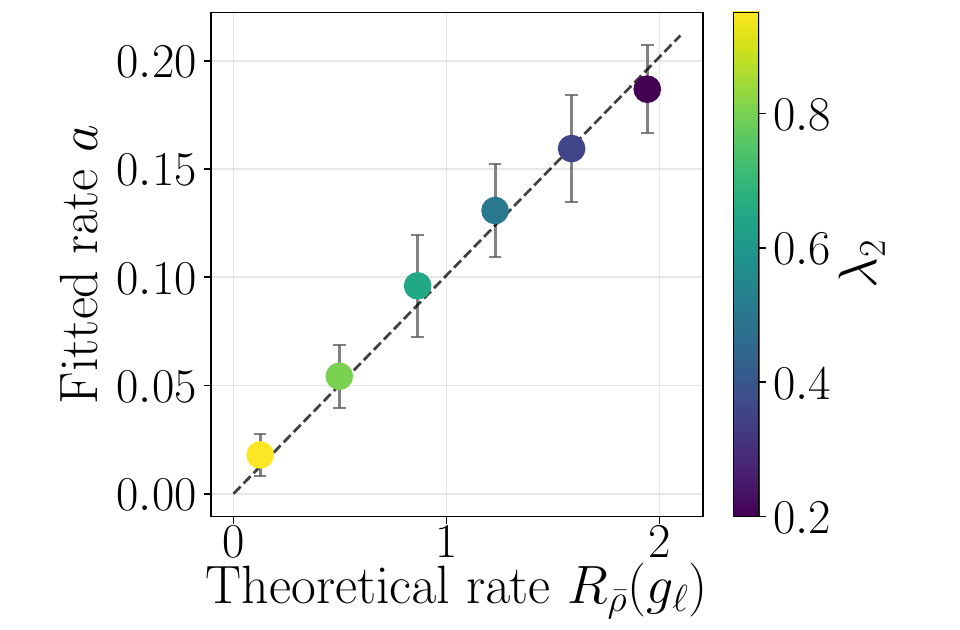}
        \caption{Fitted/theoretical rate.}
        \label{fig:numerical-lambda2-rate}
    \end{subfigure}
    \caption{Spectral dependence of the turnpike in the align-target experiment for various $\lambda_2$.}
    \label{fig:numerical-lambda2}
\end{figure}

\noindent\textbf{Experiment 2: Target-distribution geometry modulates the turnpike.}
We fix \(\lambda_2=0.65\) and vary the concentration
\(\kappa_{\textup{cand}}\) of the von Mises--Fisher distribution used to sample
the BoW candidate vectors.  In \zcref{fig:numerical-token}(a), smaller
\(\kappa_{\textup{cand}}\) produces a sharper terminal angular
dispersion, indicating that a less concentrated readout geometry can destabilize the plateau more sharply near the terminal layer.
The energy trajectories in
\zcref{fig:numerical-token}(b) also show that smaller \(\kappa_{\textup{cand}}\)
produces a stronger terminal escape.
Across the different values of \(\kappa_{\textup{cand}}\), the fitted terminal
exponent \(a\) remains consistent with the numerically computed theoretical
bound based on \(\Rayleigh\), which indicates that the
target-distribution geometry modulates the strength of the final-layer escape
in line with the theoretical bound.

\begin{figure}[t]
    \centering
    \begin{subfigure}[t]{0.33\linewidth}
        \centering
        \includegraphics[height=0.23\textheight,width=0.85\linewidth,keepaspectratio]{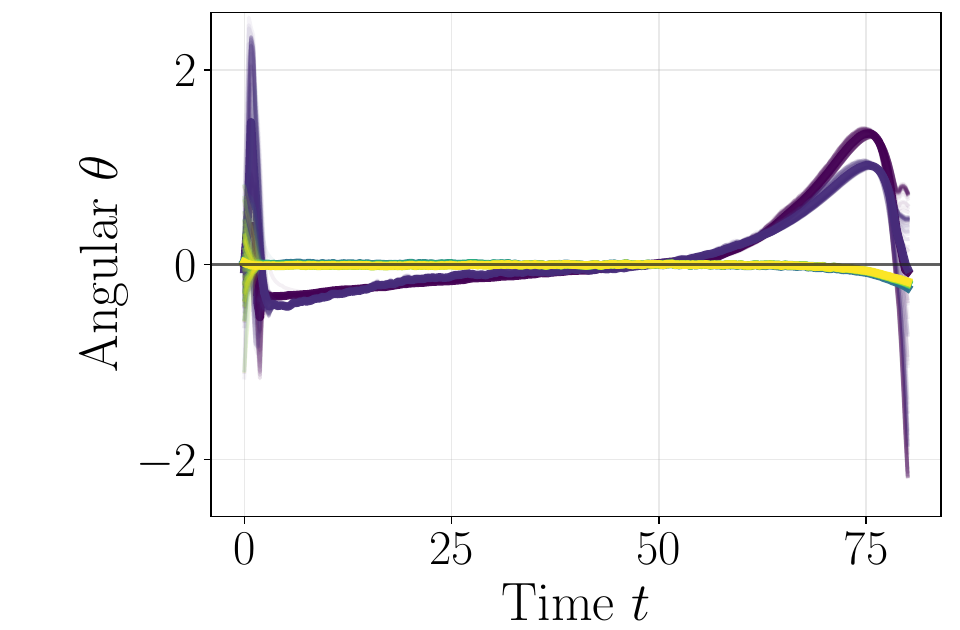}
        \caption{Angular trajectories.}
        \label{fig:numerical-token-angular}
    \end{subfigure}\hfill
    \begin{subfigure}[t]{0.33\linewidth}
        \centering
        \includegraphics[height=0.23\textheight,width=0.85\linewidth,keepaspectratio]{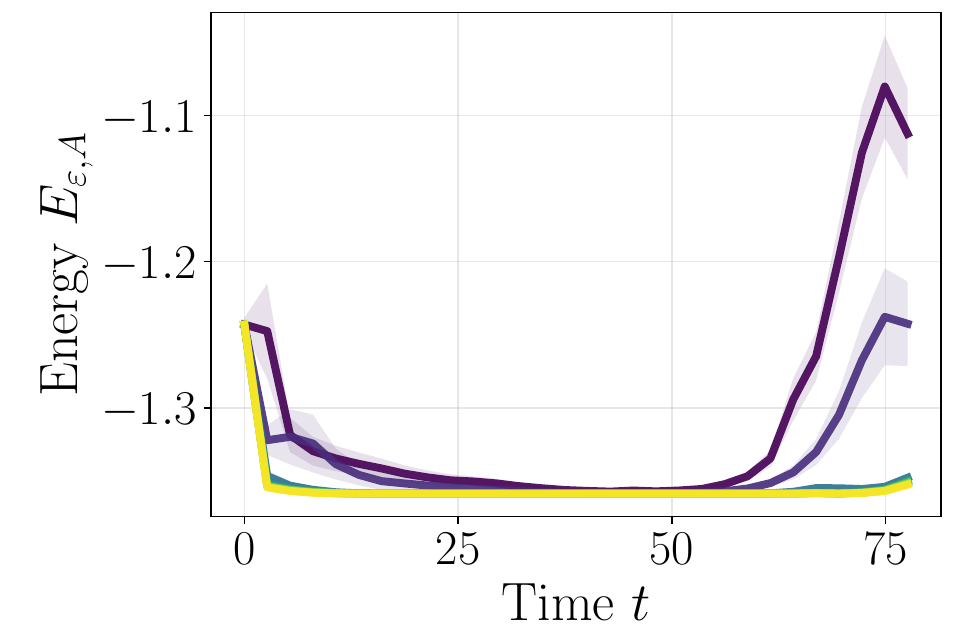}
        \caption{Energy trajectories.}
        \label{fig:numerical-token-energy}
    \end{subfigure}\hfill
    \begin{subfigure}[t]{0.33\linewidth}
        \centering
        \includegraphics[height=0.23\textheight,width=0.85\linewidth,keepaspectratio]{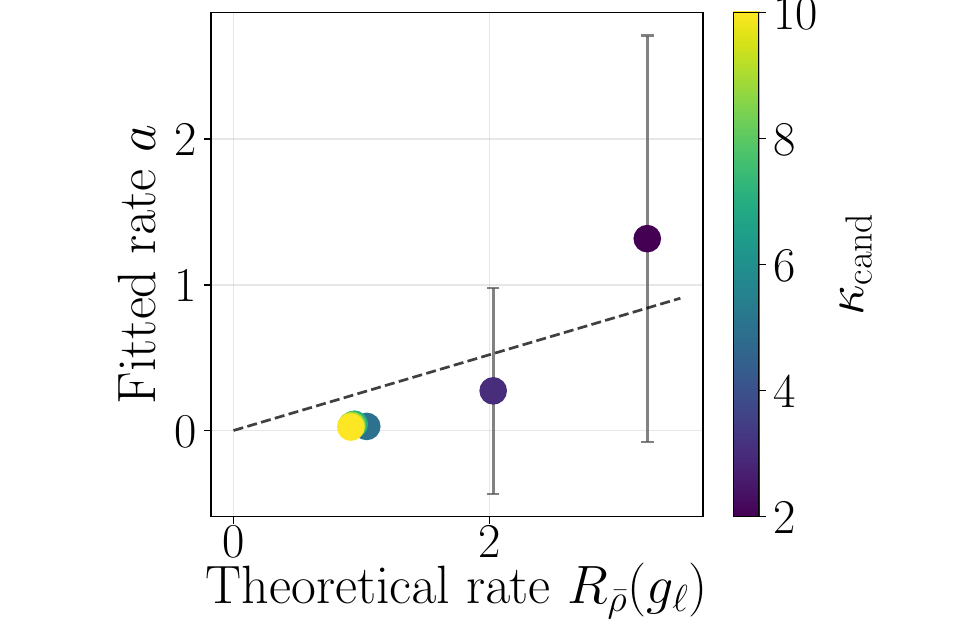}
        \caption{Fitted/theoretical rate.}
        \label{fig:numerical-token-rate}
    \end{subfigure}
    \caption{Target-geometry dependence of the turnpike in the BoW experiment for various $\kappa_{\textup{cand}}$.}
    \label{fig:numerical-token}
\end{figure}


\noindent\textbf{Robustness checks.}
We complement the sweeps above with a robustness study across seeds, diffusion strengths, horizons, particle numbers, regularization, and targets, using 20 independent runs per configuration and one fixed three-phase criterion applied without retuning.
The results are summarized in \zcref{tab:robustness}, and the protocol and full tables are given in \zcref{app:numerical-robustness}.
The full profile appears in every run at the reference setting, from the deterministic dynamics ($\eps=0$) through moderate diffusion, and for $T\in\Bab{40,80,120}$, where the plateau grows with $T$ while the boundary layers keep a $T$-independent width.
Across the candidate-bank sweep the interior plateau is always present, whereas the strength and timing of the terminal escape depend on the target.
At $\eps=10^{-1}$ the clustering and the plateau degrade while the terminal response persists, and under a global random initialization the profile requires early entry into a single well.
These experiments are stress tests of the qualitative picture in a controlled particle model.
They neither verify the quantitative hypotheses of \zcref{thm:exp_turnpike} nor constitute measurements on pretrained Transformers.

%% file: contents/related_works_cr.tex
\section{Related Work, Limitations, and Discussion}\label{sec:related}

\noindent\textbf{Relation to prescribed-weight mean-field Transformer dynamics.}
Existing mean-field Transformer analyses reviewed in \zcref{sec:intro} enriched the prescribed-weight mean-field picture, but they do not identify the training-induced terminal escape from a clustered regime studied here.
The closest work in motivation to ours is that of \textcite{koubbi2024impactloraemergenceclusters}, which analyzes the effect of LoRA-type attention perturbations on token clustering. 
However, their theory treats the perturbed parameters as given, whereas our analysis studies how training-selected parameters reshape the dynamics.
On the noise side, stationary solutions of the associated McKean--Vlasov equations on spheres are studied in \parencite{ShalovaSchlichting+2026}, and \textcite{mun2026phasetransitionsnoisytransformer} prove a sharp phase-transition dichotomy for the noisy USA free energy in every dimension $d\ge2$; both are equilibrium analyses of the same idiosyncratic-noise family, on top of which our results add the training layer.

\noindent\textbf{Relation to turnpike theory.}
Turnpike phenomena are classical in optimal control, with related estimates for Neural ODEs~\parencite{Geshkovski_Zuazua_2022,Gugat25}, mean-field optimal control~\parencite{Herty_Zhou_2025,bonnetweill2026exponentialturnpiketheoremsnonlinear}, and broader optimal control problems~\parencite{trelat2025turnpikeoptimalcontrolbeyond}. Our estimate is related in spirit, but the plateau regime is generated by the attention-induced interaction rather than imposed as a target or representation regularizer; terminal escape can then be driven by the trained FFN parameters and terminal loss.

\noindent\textbf{Practical implications for Transformer models.}
The predicted depth profile has qualitative counterparts in pretrained models: large blocks of interior layers can be removed with little damage and interventions harm most at the early and final layers~\parencite{gromov2025unreasonable,lad2025remarkable}, and the terminal escape in our model echoes reported sharp final-layer representation shifts in pretrained multilingual models~\parencite{wendler-etal-2024-llamas,shibata2026suppressingfinallayerhidden}.
We do not claim that our FFN-only noisy mean-field model directly explains these phenomena in full-scale Transformers.
The practical implications of our model are instead task-dependent: for distributed-information objectives such as \(\ell_{\textup{BoW}}\), a final-layer jump may correct attention-induced clustering, whereas for clustering-compatible objectives \(\ell\), it may signal avoidable over-reliance on the final layers.

%% file: contents/concluding_cr.tex
\section{Conclusion}
We studied how training reshapes token dynamics in Transformers by optimizing a parameter-linear FFN in a noisy mean-field model.
Training can qualitatively modify the clustering picture of the prescribed-parameter setting: the trained dynamics may exhibit early clustering, a long turnpike plateau, and terminal escape, where the adjoint analysis captures terminal-layer amplification after a gradient-descent update and the exponential turnpike estimate captures the persistence of the clustered plateau for optimized parameters.
Our theorems concern the two-well regime; landscapes with additional metastable states, verification of the mechanism in pretrained Transformers, and extensions beyond the noisy S-USA setting with FFN-only training are left for future work.

%% file: contents/proofs_cr.tex
\section{Additional proofs}\label{sec:proofs}
The constants that appear in the main text and in the proofs below are collected in \zcref{tab:constants}.
\begin{table}[h]
\centering
\caption{Constants used in the main results.}
\label{tab:constants}
\small
\begin{tabular}{lll}
\toprule
Symbol & Meaning & Defined in \\
\midrule
$\lambda_1>\lambda_2$ & two largest eigenvalues of $A$, spectral gap $\lambda_1-\lambda_2$ & \zcref{ass:nondegenerate-A} \\
$\bareps$ & small-noise threshold & \zcref{prop:energy-landscape} \\
$\Delta_0$ & critical-level gap $\frac12\min\Bab{\sinh\lambda_1,\e^{\lambda_1}-\e^{\lambda_2}}$ & \zcref{prop:energy-landscape} \\
$C_\sigma$ & bound on the feature map $\sigma$ & \zcref{ass:linear-ffn-parametrization} \\
$\regparam$ & regularization coefficient & \eqref{eq:mf-optimal-control} \\
$\mathfrak E_{\textup{max}}$ & initial energy gap plus loss oscillation & \zcref{ass:nolargegap} \\
$\eta$ & sublevel threshold defining the low-energy components $\gC_\pm(\eta)$ & \zcref{lem:same-component-entry} \\
$r_\ast$ & radius of the density charts $\gU_\pm$ & \eqref{eq:local-chart-basins-def} \\
$L_\eps$ & entry time scale from both ends of $[0,T]$ & \zcref{lem:same-component-entry} \\
$T^\ast, C, a$ & minimal depth, constant, and exponential rate of the turnpike bound & \zcref{thm:exp_turnpike} \\
\bottomrule
\end{tabular}
\end{table}
    \subsection{Proof for \texorpdfstring{\zcref{sec:energy}}{Section 3.1}}\label{sec:proof_landscape}
This appendix proves \zcref{prop:energy-landscape}. The static small-noise theorem stated first constructs the two entropic wells \(\mubar^\pm\) and establishes the gap \(\Delta_0\) between the minimal energy and every other critical level. \zcref{sec:app:lan-gibbs} proves the gap by passing to the small-temperature limit of the critical points, and \zcref{app:static} completes the proof of the theorem by a Laplace-type expansion around each well that is uniform in \(\eps\). \zcref{sec:app:lan-hessian,sec:app:lan-hessian-asymptotics} compute the Hessian of \(\Energy\) at \(\mubar\) in the density chart and identify its leading scale, which gives the coercivity constant of the local PL inequality. \zcref{sec:app:lan-transport} relates this Hessian to the second variation along transport perturbations.
    \input{contents/proofs/landscape_reorg_sphere_cr}
    \subsection{Proof for \texorpdfstring{\zcref{sec:adjoint}}{Section 3.2}}\label{sec:proof_adjoint}
This appendix proves \zcref{lem:1GDparam,prop:one-step-escape}. \zcref{sec:app:adj-notation} fixes the operator notation and the sign convention of the adjoint equation, and \zcref{sec:app:adj-1gd} derives the one-step gradient-descent update of \zcref{lem:1GDparam}. \zcref{sec:app:adj-gramian} introduces the finite-time Gramian and the visibility factor, which measure how much of the terminal gradient reaches layer \(t\) through the backward dynamics. \zcref{sec:app:adj-taylor,sec:app:adj-response} prove a local Taylor lower bound for the energy gap in the density chart and a small-learning-rate estimate for the nonlinear response of the state. \zcref{sec:app:adj-escape} combines these ingredients with a Rayleigh--Jensen lower bound for the costate to prove \zcref{prop:one-step-escape}.
    \input{contents/proofs/adjoint_reorg_cr}
    \subsection{Proof for \texorpdfstring{\zcref{sec:exp_turnpike}}{Section 3.3}}\label{sec:proof_expturnpike}
This appendix proves \zcref{thm:exp_turnpike} in three stages. \zcref{sec:app:poincare,sec:app:taylor,sec:app:local-pl} collect the analytic inputs around a single well, namely the weighted Poincar\'e estimates, the local Taylor estimates in the density chart, and the local PL inequality with a controlled strict dissipativity estimate, which localize the state near a well without yet giving the turnpike. \zcref{sec:app:two-well,sec:app:energy-budget,sec:app:sign-barrier,sec:app:chart-entry,sec:app:regularization} treat the double-well structure under \zcref{ass:nolargegap}. The energy budget forces entry into a low-energy component, the sign barrier excludes switching between the two branches, and the parabolic regularization upgrades this entry to the density chart, which proves \zcref{lem:same-component-entry}. \zcref{sec:app:adjoint-stability,sec:app:optimality-system,sec:app:adjoint-decay} control the optimal parameters under \zcref{ass:ctrl-samewell}. A \(T\)-uniform stability estimate for the linearized evolution along the optimal trajectory, the optimality system with the stationarity condition \(\regparam W^\ast_t=-\int\nabla\phi_t\sigma^\top\d\mu^\ast_t\), and the backward decay of the adjoint state make \(W^\ast\) exponentially small away from the terminal layer, which completes the proof.
    \input{contents/proofs/exp_turnpike_reorg_sphere_cr}

%% file: contents/proofs/landscape_reorg_sphere_cr.tex
The first static input is the following small-noise theorem.  It constructs the two entropic wells and gives the small-temperature critical-level gap used in \zcref{prop:energy-landscape}.  The chord estimate is kept inside the package because it is the finite-\(\eps\) mechanism behind one-well uniqueness.

\begin{theorem}[static small-noise package]
\label{thm:package}
Let
\begin{equation}\label{eq:static-min-value}
m_\eps\coloneqq \operatorname*{inf}_{\nu\in\mathcal P(\dSphe)}
\Energy(\nu).
\end{equation}
Set
\begin{equation}\label{eq:Delta0-static-theorem}
Q_A\coloneqq \max\{\e^{\lambda_2},\cosh\lambda_1\},
\qquad
\Delta_0\coloneqq \frac12(\e^{\lambda_1}-Q_A)
=\frac12\min\{\e^{\lambda_1}-\e^{\lambda_2},\sinh\lambda_1\}>0.
\end{equation}
There exist \(\eps_0>0\), disjoint narrow neighborhoods \(V_+\) of
\(\delta_{e_1}\) and \(V_-\) of \(\delta_{-e_1}\), with \(V_-=S_\#V_+\) for
\(S(x)=-x\), such that the following holds for \(0<\eps<\eps_0\).

\begin{enumerate}[label=(\roman*)]
\item Every global minimizer of \(\Energy\) lies in \(V_+\cup V_-\).
There is a global minimizer \( \mubar^+\in V_+\), and
\(\mubar^-\coloneqq S_\#\mubar^+\) is a global minimizer in \(V_-\).  We write
\(\mubar^+=\rhobar^+\omega\), \(\mubar^-=\rhobar^-\omega\), and, on the
selected positive branch, \(\mubar=\mubar^+\), \(\rhobar=\rhobar^+\).

\item Let \(\eta_\eps\downarrow0\).  Suppose
\(\mu_\eps=\rho_\eps\omega\) is any family of critical points satisfying
\begin{equation}\label{eq:low-energy-critical-main}
\Energy(\mu_\eps)
\le m_\eps+\eta_\eps,
\qquad
\mu_\eps\in V_+,
\end{equation}
and hence the Euler--Lagrange equation
\begin{equation}\label{eq:EL-critical-main}
\eps(\log\rho_\eps+1)-K_A*\mu_\eps
=\Lambda_\eps.
\end{equation}
Then there are constants \(r,c,C>0\), independent of the particular family,
and centers \(x_{\mu,\eps}\in\dSphe\) such that
\begin{equation}\label{eq:center-Oeps-main}
x_{\mu,\eps}=e_1+O(\eps),
\end{equation}
\begin{equation}\label{eq:tail-uniform-main}
\mu_\eps\bigl(\dSphe\setminus B_r(x_{\mu,\eps})\bigr)
\le C\e^{-c/\eps},
\end{equation}
and, with \(u_\mu(x)=\exp_{x_{\mu,\eps}}^{-1}(x)\) inside the cap and
\(u_\mu=0\) outside,
\begin{equation}\label{eq:sigmaeps-main}
\sigma_\eps^2=\frac{\eps}{\e^{\lambda_1}\lambda_1},
\end{equation}
the rescaled density converges to the standard Gaussian in weighted
\(L^1(\mathbb R^{d-1})\), uniformly over such families.  In particular, for
every unit tangent vector \(\tau\in T_{x_{\mu,\eps}}\dSphe\),
\begin{equation}\label{eq:uniform-first-moment-main}
\int_\dSphe  \aab{u_\mu(x),\tau}\d\mu_\eps(x)=O(\eps),
\end{equation}
\begin{equation}\label{eq:uniform-second-moment-main}
\int_\dSphe  \aab{u_\mu(x),\tau}^2\d\mu_\eps(x)
=\sigma_\eps^2+o(\eps),
\end{equation}
and, for every integer \(m\ge0\),
\begin{equation}\label{eq:uniform-momentm-main}
\int_\dSphe  |u_\mu(x)|^m\d\mu_\eps(x)=O(\eps^{m/2}).
\end{equation}
The same statement holds in \(V_-\) after applying the antipodal map.

For the selected branch \(\mubar=\mubar^+\) we write
\(x_\eps=x_{\mubar,\eps}\) and \(u(x)=\exp_{x_\eps}^{-1}(x)\) on the cap,
extended by zero outside.  Thus
\begin{equation}\label{eq:tail-main}
\mubar\bigl(\dSphe\setminus B_r(x_\eps)\bigr)\le C\e^{-c/\eps},
\end{equation}
\begin{equation}\label{eq:moment1-main}
\int_\dSphe  \aab{u(x),\tau}\d\mubar(x)=O(\eps),
\end{equation}
\begin{equation}\label{eq:moment2-main}
\int_\dSphe  \aab{u(x),\tau}^2\d\mubar(x)=\sigma_\eps^2+o(\eps),
\end{equation}
and
\begin{equation}\label{eq:momentm-main}
\int_\dSphe  |u(x)|^m\d\mubar(x)=O(\eps^{m/2}).
\end{equation}

\item Let \(\eta_\eps\downarrow0\), and let
\(\mu_\eps^0,\mu_\eps^1\in V_+\) be two low-energy critical point families
satisfying \eqref{eq:low-energy-critical-main}, with \(\mu_\eps\) replaced by each \(\mu_\eps^i\). Set
\[
\mu_\eps^s=(1-s)\mu_\eps^0+s\mu_\eps^1,
\qquad s\in[0,1].
\]
For \(\mu\in\mathcal P(\dSphe)\) define
\(\Pi_\mu f\coloneqq f-\int_\dSphe  f\d\mu\) and
\begin{equation}\label{eq:K-mu-def-main}
\mathsf K_\mu h
\coloneqq \Pi_\mu\left[\int_\dSphe  K_A(\cdot,y)h(y)\d\mu(y)\right],
\qquad h\in L^2_0(\mu).
\end{equation}
Then, uniformly in \(s\in[0,1]\) and in the pair
\(\mu_\eps^0,\mu_\eps^1\),
\begin{equation}\label{eq:chord-K-op-main}
\Vab{\mathsf K_{\mu_\eps^s}}_{\op,L^2_0(\mu_\eps^s)}
\le
\left(\frac{\lambda_2}{\lambda_1}+o(1)\right)\eps.
\end{equation}
The same estimate holds for chords of low-energy critical points in \(V_-\).

\item For every sufficiently small fixed \(\eps>0\), there exists
\(\eta_{\textup{crit}}=\eta_{\textup{crit}}(\eps)>0\) such that each well
contains at most one critical point \(\mu\) satisfying
\[
\Energy(\mu)\le m_\eps+\eta_{\textup{crit}}.
\]  In particular, the global minimizer set consists of
exactly two antipodal branches:
\begin{equation}\label{eq:package-two-minimizers}
\Minimizers
\coloneqq
\argmin_{\mu\in\mathcal P(\dSphe)}
\Energy(\mu)
=\left\{\mubar^+,\mubar^-\right\},
\qquad
\mubar^-=S_\#\mubar^+.
\end{equation}

\end{enumerate}
\end{theorem}

For the local density-chart analysis we work on the positive branch supplied by
\zcref{thm:package} and set \(\mubar\coloneqq \mubar^+\),
\(\rhobar\coloneqq \rhobar^+\).  For fixed \(\eps\), we use the notation
\(\Minimizers=\{\mubar^+,\mubar^-\}\) from \zcref{prop:energy-landscape}.

\subsubsection{Gibbs-limit proof of the critical-level gap}\label{sec:app:lan-gibbs}

We record the Gibbs-limit facts needed for the non-minimal critical-level gap in \zcref{prop:energy-landscape}.  The point is that limits of positive-temperature Gibbs critical points satisfy a global max-support condition coming from the Laplace principle.

Throughout this appendix, we write $K_A(x,y)\coloneqq\e^{\aab{x,Ay}}$ and
\begin{equation}\label{eq:appendix-IA-W-def}
I_A(\mu)\coloneqq \iint_{\dSphe\times\dSphe}K_A(x,y)\d\mu(x)\d \mu(y),
\qquad
W_\mu(x)\coloneqq (K_A*\mu)(x)=\int_\dSphe  K_A(x,y)\d\mu(y).
\end{equation}

\begin{definition}[positive-temperature critical set and its zero-temperature limits]
\label{def:gibbs-limit-sets}
For the fixed noise level \(\eps\), define
\begin{equation}\label{eq:Crit-eps-A-def}
\Crit\Energy
\coloneqq 
\Set{\mu=\rho\omega\in\mathcal P(\dSphe)|
\begin{array}{l}
\rho>0,
\displaystyle
\rho(x)=
\frac{\exp(W_\mu(x)/\eps)}
{\int_\dSphe \exp(W_\mu(z)/\eps)\d\omega(z)}
\end{array}
}.
\end{equation}
For a varying noise level \(\eps_n\), the corresponding critical set is denoted by
\(\Crit E_{\eps_n,A}\).  Equivalently,
\[
\eps(\log\rho+1)-K_A*\mu
\quad\text{is constant on }\dSphe.
\]
We define the zero-temperature Gibbs-limit set by
\begin{equation}\label{eq:Glim-def}
\mathfrak G_A^{\mathrm{lim}}
\coloneqq 
\Set{
\nu\in\mathcal P(\dSphe)|
\begin{array}{l}
\exists\,\eps_n\downarrow0,\ \exists\,\mu_n\in\Crit E_{\eps_n,A}
\text{ such that }\mu_n\rightharpoonup\nu
\end{array}
}.
\end{equation}
We also define the explicit max-support outer class
\begin{equation}\label{eq:Gmax-def}
\mathfrak G_A^{\max}
\coloneqq 
\Set{
\nu\in\mathcal P(\dSphe)|
\supp\nu
\subset
\argmax_{x\in\dSphe} W_\nu(x)
}.
\end{equation}
Thus \(\mathfrak G_A^{\mathrm{lim}}\) is the intrinsic object, while
\(\mathfrak G_A^{\max}\) is a tractable necessary outer condition.
\end{definition}

If \(\nu\in\mathfrak G_A^{\max}\), then
\[
W_\nu(x)=\max_{\dSphe} W_\nu
\qquad \nu\text{-a.e.}
\]
and hence
\begin{equation}\label{eq:Gmax-height-equals-interaction}
I_A(\nu)=\int_\dSphe  W_\nu\d\nu=\max_{\dSphe} W_\nu.
\end{equation}
Equivalently,
\[
W_\nu\le I_A(\nu)\quad\text{on }\dSphe,
\qquad
W_\nu=I_A(\nu)\quad\nu\text{-a.e.}
\]
The converse also holds by continuity of \(W_\nu\).

\begin{lemma}[Gibbs limits are max-support states]
\label{lem:gibbs-limits-are-max-support}
One has
\begin{equation}\label{eq:Glim-subset-Gmax}
\mathfrak G_A^{\mathrm{lim}}
\subset
\mathfrak G_A^{\max}.
\end{equation}
Equivalently, if \(\eps_n\downarrow0\),
\(\mu_n\in\Crit E_{\eps_n,A}\), and
\(\mu_n\rightharpoonup\nu\), then
\begin{equation}\label{eq:gibbs-limit-max-support-conclusion}
\supp\nu
\subset
\argmax_{x\in\dSphe} W_\nu(x).
\end{equation}
\end{lemma}

\begin{proof}
Set
\[
W_n\coloneqq W_{\mu_n},
\qquad
W\coloneqq W_\nu.
\]
Since \(K_A\in C(\dSphe\times\dSphe)\) and \(\dSphe\) is compact, narrow convergence gives
\begin{equation}\label{eq:Wn-uniform-to-W-app}
\Vab{W_n-W}_{L^\infty(\dSphe)}\to0.
\end{equation}
Indeed, uniform continuity of \(K_A\) in the first variable reduces the claim to
finitely many scalar narrow-convergence limits.

Let \(W_{\max}\coloneqq \max_{\dSphe} W\).  Fix a point \(x_0\notin\argmax_{\dSphe} W\).
Choose an open neighborhood \(O\) of \(x_0\) and a number \(\delta>0\) such that
\[
\sup_{\overline O} W\le W_{\max}-3\delta.
\]
Choose also an open set \(U\subset\dSphe\) with \(\omega(U)>0\) and
\[
\inf_U W\ge W_{\max}-\delta.
\]
For all sufficiently large \(n\), \eqref{eq:Wn-uniform-to-W-app} gives
\[
\sup_{\overline O} W_n\le W_{\max}-2\delta,
\qquad
\inf_U W_n\ge W_{\max}-\frac32\delta.
\]
Using the Gibbs representation of \(\mu_n\),
\[
\mu_n(O)
\le
\mu_n(\overline O)
\le
\frac{\omega(\dSphe)\exp((W_{\max}-2\delta)/\eps_n)}
{\omega(U)\exp((W_{\max}-\frac32\delta)/\eps_n)}
\le
C_U\exp\left(-\frac{\delta}{2\eps_n}\right)
\to0.
\]
By the Portmanteau theorem for the open set \(O\),
\[
\nu(O)\le \liminf_{n\to\infty}\mu_n(O)=0.
\]
Thus every point outside \(\argmax_{\dSphe} W\) has a neighborhood of
zero \(\nu\)-mass.  This proves
\(\supp\nu\subset\argmax_{\dSphe} W\).
\end{proof}

\begin{lemma}[support-minimality inside the max-support class]
\label{lem:support-minimality-Gmax}
Let \(\nu\in\mathfrak G_A^{\max}\).  Then, for every
\(\eta\in\mathcal P(\supp\nu)\),
\begin{equation}\label{eq:support-minimality-Gmax}
I_A(\nu)\le I_A(\eta).
\end{equation}
\end{lemma}

\begin{proof}
Since \(\nu\in\mathfrak G_A^{\max}\), we have
\(W_\nu=I_A(\nu)\) on \(\supp\nu\).  Set
\(\sigma\coloneqq \eta-\nu\).  Then \(\sigma(\dSphe)=0\) and
\[
\int_\dSphe  W_\nu\d\sigma=0.
\]
Therefore
\[
I_A(\eta)=I_A(\nu+\sigma)
=I_A(\nu)+2\int_\dSphe  W_\nu\d\sigma+I_A(\sigma)
=I_A(\nu)+I_A(\sigma).
\]
It remains only to recall that \(K_A\) is positive definite.  Indeed, for every
finite signed measure \(\sigma\), the uniformly absolutely convergent expansion
on the compact sphere gives
\[
I_A(\sigma)
=
\sum_{n=0}^{\infty}\frac1{n!}
\iint (x^\top Ay)^n\d\sigma(x)\d \sigma(y),
\]
and
\[
\iint (x^\top Ay)^n\d\sigma(x)\d \sigma(y)
=
\left|
\int_\dSphe  (A^{1/2}x)^{\otimes n}\d\sigma(x)
\right|^2
\ge0.
\]
Thus \(I_A(\sigma)\ge0\), proving \eqref{eq:support-minimality-Gmax}.
\end{proof}

\begin{lemma}[two-point bound across the equator]
\label{lem:two-point-bound-across-equator}
Let \(x,y\in\dSphe\) satisfy \(x_1y_1\le0\), where \(x_1=x\cdot e_1\) and
\(y_1=y\cdot e_1\).  Then
\begin{equation}\label{eq:two-point-bound-across-equator}
\min_{t\in[0,1]} I_A\bigl(t\delta_x+(1-t)\delta_y\bigr)
\le
\max\{\e^{\lambda_2},\cosh\lambda_1\}.
\end{equation}
\end{lemma}

\begin{proof}
After interchanging \(x\) and \(y\), we may write
\[
x_1=\alpha\ge0,
\qquad
y_1=-\beta\le0,
\qquad
\alpha,\beta\in[0,1].
\]
Set
\[
s\coloneqq \sqrt{1-\alpha^2},
\qquad
r\coloneqq \sqrt{1-\beta^2}.
\]
Then
\[
x^\top Ax\le \lambda_2+(\lambda_1-\lambda_2)\alpha^2=:a,
\qquad
y^\top Ay\le \lambda_2+(\lambda_1-\lambda_2)\beta^2=:c,
\]
and, by Cauchy--Schwarz in the orthogonal complement of \(e_1\),
\[
x^\top Ay\le -\lambda_1\alpha\beta+\lambda_2sr=:b.
\]
Hence
\[
I_A\bigl(t\delta_x+(1-t)\delta_y\bigr)
\le
 t^2\e^a+2t(1-t)\e^b+(1-t)^2\e^c.
\]
If \(\alpha=\beta=0\), then \(a,c,b\le\lambda_2\), so the claim follows.  Assume
\(\alpha+\beta>0\), and choose the trial weight
\[
t_0\coloneqq \frac{\beta}{\alpha+\beta},
\qquad
1-t_0=\frac{\alpha}{\alpha+\beta}.
\]
This weight cancels the top coordinate.  Thus it is enough to show
\begin{equation}\label{eq:two-point-trial-estimate}
\frac{\beta^2\e^a+\alpha^2\e^c+2\alpha\beta\e^b}{(\alpha+\beta)^2}
\le
(1-\alpha\beta)\e^{\lambda_2}+\alpha\beta\cosh\lambda_1.
\end{equation}
By convexity of the exponential,
\[
\e^a\le (1-\alpha^2)\e^{\lambda_2}+\alpha^2\e^{\lambda_1},
\qquad
\e^c\le (1-\beta^2)\e^{\lambda_2}+\beta^2\e^{\lambda_1}.
\]
Moreover \(sr\le(s^2+r^2)/2\), so
\[
b\le \bar b\coloneqq -\lambda_1\alpha\beta+\frac{\lambda_2}{2}(s^2+r^2).
\]
With
\[
\omega_0\coloneqq \frac{s^2+r^2}{2},
\qquad
\omega_+\coloneqq \frac{(\alpha-\beta)^2}{4},
\qquad
\omega_-\coloneqq \frac{(\alpha+\beta)^2}{4},
\]
one has \(\omega_0+\omega_++\omega_-=1\) and
\[
\bar b=\omega_0\lambda_2+\omega_+\lambda_1+\omega_-(-\lambda_1).
\]
Hence
\[
\e^b\le \omega_0\e^{\lambda_2}+\omega_+\e^{\lambda_1}
+\omega_-\e^{-\lambda_1}.
\]
Substituting these three estimates into the numerator in
\eqref{eq:two-point-trial-estimate}, the coefficient of \(\e^{\lambda_2}\) is
\[
\beta^2(1-\alpha^2)+\alpha^2(1-\beta^2)+\alpha\beta(s^2+r^2)
=(\alpha+\beta)^2(1-\alpha\beta),
\]
the coefficient of \(\e^{\lambda_1}\) is
\[
2\alpha^2\beta^2+2\alpha\beta\omega_+
=\frac{\alpha\beta}{2}(\alpha+\beta)^2,
\]
and the coefficient of \(\e^{-\lambda_1}\) is
\[
2\alpha\beta\omega_-=\frac{\alpha\beta}{2}(\alpha+\beta)^2.
\]
This proves \eqref{eq:two-point-trial-estimate}.  Since the right-hand side of
\eqref{eq:two-point-trial-estimate} is a convex combination of
\(\e^{\lambda_2}\) and \(\cosh\lambda_1\), it is bounded by
\(\max\{\e^{\lambda_2},\cosh\lambda_1\}\).  This proves the lemma.
\end{proof}

\begin{proposition}[sharp first-excited max-support bound]
\label{prop:sharp-first-excited-max-support-bound}
One has
\begin{equation}\label{eq:sharp-first-excited-max-support-bound}
\sup_{\nu\in\mathfrak G_A^{\max}\setminus\{\delta_{e_1},\delta_{-e_1}\}}
I_A(\nu)
=
Q_A\coloneqq \max\{\e^{\lambda_2},\cosh\lambda_1\}.
\end{equation}
In particular \(Q_A<\e^{\lambda_1}\).
\end{proposition}

\begin{proof}
We first prove the upper bound.  Let
\(\nu\in\mathfrak G_A^{\max}\setminus\{\delta_{e_1},\delta_{-e_1}\}\).

Suppose first that \(\supp\nu\subset\{x_1\ge0\}\).  Let
\(x\in\supp\nu\) minimize \(x_1\) on the support and set
\(c\coloneqq x_1\).  If \(0<c<1\), define
\[
v\coloneqq \frac{e_1-cx}{\sqrt{1-c^2}}\in T_x\dSphe.
\]
For every \(y\in\supp\nu\), one has \(y_1\ge c\).  Also
\[
x^\top Ay\le \lambda_1cy_1+
\lambda_2\sqrt{1-c^2}\sqrt{1-y_1^2}.
\]
Consequently
\[
(e_1-cx)^\top Ay
\ge
\lambda_1(1-c^2)y_1
-c\lambda_2\sqrt{1-c^2}\sqrt{1-y_1^2}>0,
\]
where the strict inequality follows from \(y_1\ge c>0\) and
\(\lambda_1>\lambda_2\).  Hence
\[
\partial_v W_\nu(x)
=
\int_\dSphe  \e^{x^\top Ay}v^\top Ay\d\nu(y)>0,
\]
contradicting \(x\in\argmax_{\dSphe} W_\nu\).  Therefore either
\(c=1\), in which case \(\nu=\delta_{e_1}\), or \(c=0\).  Since
\(\nu\ne\delta_{e_1}\), we must have \(c=0\).  Thus there is
\(x\in\supp\nu\) with \(x_1=0\).  By
\zcref{lem:support-minimality-Gmax},
\[
I_A(\nu)\le I_A(\delta_x)=\e^{x^\top Ax}\le \e^{\lambda_2}.
\]
The case \(\supp\nu\subset\{x_1\le0\}\) is identical and gives
\(I_A(\nu)\le\e^{\lambda_2}\), unless \(\nu=\delta_{-e_1}\), which is excluded.

It remains to consider the case where the support meets both closed hemispheres.
Then there are \(x,y\in\supp\nu\) such that \(x_1y_1\le0\).  By
\zcref{lem:support-minimality-Gmax}, for every \(t\in[0,1]\),
\[
I_A(\nu)
\le
I_A\bigl(t\delta_x+(1-t)\delta_y\bigr).
\]
Taking the minimum over \(t\) and applying
\zcref{lem:two-point-bound-across-equator} yields
\[
I_A(\nu)\le Q_A.
\]
This proves the upper bound.

For sharpness, \(\delta_{e_2}\) and \(\delta_{-e_2}\) belong to
\(\mathfrak G_A^{\max}\), and
\[
I_A(\delta_{\pm e_2})=\e^{\lambda_2}.
\]
Also
\[
\sigma_1\coloneqq \frac12(\delta_{e_1}+\delta_{-e_1})
\in\mathfrak G_A^{\max},
\qquad
I_A(\sigma_1)=\cosh\lambda_1.
\]
Thus the supremum is exactly \(Q_A\).  Finally,
\(\e^{\lambda_2}<\e^{\lambda_1}\) and \(\cosh\lambda_1<\e^{\lambda_1}\), so
\(Q_A<\e^{\lambda_1}\).
\end{proof}

\begin{proposition}[explicit small-temperature excited critical gap]
\label{prop:explicit-small-temperature-excited-critical-gap}
Let
\[
\Delta_{\textup{crit}}(\eps)
\coloneqq 
\inf\left\{
\Energygap(\mu):
\mu\in\Crit\Energy\setminus\Minimizers
\right\}.
\]
Then
\begin{equation}\label{eq:explicit-small-temperature-excited-critical-gap}
\liminf_{\eps\downarrow0}\Delta_{\textup{crit}}(\eps)
\ge
\Delta_0\coloneqq \frac12(\e^{\lambda_1}-Q_A).
\end{equation}
Equivalently, for every \(\gamma>0\), there exists \(\eps_\gamma>0\) such that,
for \(0<\eps<\eps_\gamma\),
\begin{equation}\label{eq:explicit-small-temperature-excited-critical-gap-finite}
\mu\in\Crit\Energy\setminus\Minimizers
\quad\Longrightarrow\quad
\Energygap(\mu)\ge \Delta_0-\gamma.
\end{equation}
\end{proposition}

\begin{proof}
Let \(\eps_n\downarrow0\) and let
\(\mu_n\in\Crit E_{\eps_n,A}\setminus\mathcal{M}_{E_{\eps_n}}\).  Passing to a subsequence, assume \(\mu_n\rightharpoonup\nu\).  By
\zcref{lem:gibbs-limits-are-max-support},
\(\nu\in\mathfrak G_A^{\max}\).

If \(\nu=\delta_{e_1}\) or \(\nu=\delta_{-e_1}\), then the Gibbs representation
and standard Laplace estimates imply
\(E_{\eps_n,A}(\mu_n)\to-\frac12\e^{\lambda_1}\).  Since
\(m_{\eps_n}\to-\frac12\e^{\lambda_1}\), we get \(\Delta_{\eps_n,A}(\mu_n)\to0\).  The
one-well low-energy critical uniqueness in \zcref{thm:package} then
forces \(\mu_n\) to be the corresponding minimizer for all large \(n\), contrary
to \(\mu_n\notin\mathcal{M}_{E_{\eps_n}}\).  Thus a non-minimal critical sequence
cannot converge to either ground Dirac.

Therefore \(\nu\notin\{\delta_{e_1},\delta_{-e_1}\}\).  By
\zcref{prop:sharp-first-excited-max-support-bound},
\(I_A(\nu)\le Q_A\).  The test-cap construction in the proof of
\zcref{thm:package} gives
\[
m_{\eps_n}\le -\frac12\e^{\lambda_1}+C\eps_n|\log\eps_n|.
\]
Using \(\eps_n\KL(\mu_n\mid\omega)\ge0\), we obtain
\[
\Delta_{\eps_n,A}(\mu_n)
=
E_{\eps_n,A}(\mu_n)-m_{\eps_n}
\ge
\frac12(\e^{\lambda_1}-I_A(\mu_n))-C\eps_n|\log\eps_n|.
\]
Since \(I_A(\mu_n)\to I_A(\nu)\),
\[
\liminf_{n\to\infty}\Delta_{\eps_n,A}(\mu_n)
\ge
\frac12(\e^{\lambda_1}-I_A(\nu))
\ge
\frac12(\e^{\lambda_1}-Q_A)=\Delta_0.
\]
Taking the infimum over all such sequences proves \eqref{eq:explicit-small-temperature-excited-critical-gap}.
\end{proof}

\subsubsection{Proof of the static small-noise package}\label{app:static}

We give the finite-\(\eps\) argument behind
\zcref{thm:package}.
The zero-temperature interaction only supplies a separation estimate; the actual one-well uniqueness comes from an Euler--Lagrange difference identity and the same-well chord estimate
\eqref{eq:chord-K-op-main}.

\begin{proof}[Proof of \zcref{thm:package}]
\smallskip
\noindent\textbf{1. Direct method and Euler--Lagrange equation.}
Since \(\dSphe\) is compact, \(\mathcal P(\dSphe)\) is narrowly compact.  The entropy is
narrowly lower semicontinuous and \(I_A\) is narrowly continuous because
\(K_A\) is continuous and bounded.  Hence \(\Energy\) admits a global
minimizer.  Every finite-energy minimizer, and more generally every interior
critical point, has a smooth positive density \( \rho\) and satisfies
\begin{equation}\label{eq:EL-static-proof}
\eps(\log\rho(x)+1)-W_\mu(x)=\Lambda_\mu.
\end{equation}
Equivalently,
\begin{equation}\label{eq:Gibbs-static-proof}
\rho(x)=\frac{\exp(W_\mu(x)/\eps)}
{\int_\dSphe  \exp(W_\mu(z)/\eps)\d\omega(z)}.
\end{equation}
Smoothness and positivity follow immediately from the smooth positive kernel in
this fixed point formula.

\smallskip
\noindent\textbf{2. Zero-temperature separation.}
For \(x,y\in\dSphe\),
\[
x^\top Ay
\le \lambda_1.
\]
Because the top eigenvalue is simple, equality holds only for
\((x,y)=(e_1,e_1)\) or \((x,y)=(-e_1,-e_1)\).  Thus
\begin{equation}\label{eq:IA-sharp-static-proof}
I_A(\mu)\le \e^{\lambda_1},
\end{equation}
and equality holds only for \(\mu=\delta_{e_1}\) or
\(\mu=\delta_{-e_1}\).

Let \(r_\eps=\sqrt{\eps|\log\eps|}\) and let \(\nu_\eps\) be the normalized
surface measure on \(B_{r_\eps}(e_1)\).  Then
\[
\KL(\nu_\eps\mid\omega)=O(|\log\eps|),
\qquad
I_A(\nu_\eps)=\e^{\lambda_1}+O(\eps|\log\eps|),
\]
and hence
\begin{equation}\label{eq:min-upper-static-proof}
m_{\eps_n}\le -\frac12\e^{\lambda_1}+C\eps_n|\log\eps_n|.
\end{equation}
If a family \(\mu_\eps\) satisfies
\(\Energy(\mu_\eps)\le m_\eps+\eta_\eps\) with
\(\eta_\eps\to0\), then \(\KL\ge0\) and \eqref{eq:min-upper-static-proof} give
\begin{equation}\label{eq:interaction-near-max-static-proof}
I_A(\mu_\eps)\ge \e^{\lambda_1}-C\eps|\log\eps|-2\eta_\eps.
\end{equation}
Every narrow limit point of \(\mu_\eps\) must therefore be one of
\(\delta_{e_1}\), \(\delta_{-e_1}\).  Fix disjoint narrow neighborhoods
\(V_+\) and \(V_-\) of these two measures, with \(V_-=S_\#V_+\).  Then, for
small \(\eps\), every global minimizer lies in \(V_+\cup V_-\).  By the direct
method there is at least one minimizer; reflecting by \(S\) if necessary gives
a minimizer \(\mubar^+\in V_+\), and \(\mubar^-\coloneqq S_\#\mubar^+\in V_-\) is also
a minimizer.

The same argument applies to any low-energy critical family satisfying
\eqref{eq:low-energy-critical-main}: after restricting to \(V_+\), it converges
narrowly to \(\delta_{e_1}\).  This convergence is uniform over all such
families, in the usual contradiction sense.

\smallskip
\noindent\textbf{3. Uniform one-well Laplace estimates for critical points.}
The last paragraph of Step~2 gives uniform narrow convergence to
\(\delta_{e_1}\) for all low-energy critical families in \(V_+\).  Fix such a
family \(\mu_\eps=\rho_\eps\omega\) satisfying \eqref{eq:low-energy-critical-main}.
Since all \(x\)-derivatives of \(K_A(x,y)\) are continuous and bounded on
\(\dSphe\times\dSphe\), the uniform narrow convergence to \(\delta_{e_1}\) gives
\begin{equation}\label{eq:W-C4-static-proof}
W_{\mu_\eps}\to W_0,
\qquad
W_0(x)=\e^{\lambda_1 x_1},
\end{equation}
in \(C^4(\dSphe)\), uniformly over the family.  The function \(W_0\) has a unique
strict global maximum at \(e_1\), and
\[
-\Hessian_{\dSphe}W_0(e_1)=a_0 I_{T_{e_1}\dSphe},
\qquad
 a_0\coloneqq \e^{\lambda_1}\lambda_1.
\]
Therefore, for a fixed small cap and all small \(\eps\), \(W_{\mu_\eps}\) has a
unique maximizer \(x_{\mu,\eps}\) in that cap, this maximizer is global, and
\(x_{\mu,\eps}\to e_1\).  Moreover, with
\(u\in T_{x_{\mu,\eps}}\dSphe\),
\begin{equation}\label{eq:potential-expansion-static-proof}
W_{\mu_\eps}(x_{\mu,\eps})
-W_{\mu_\eps}(\exp_{x_{\mu,\eps}}u)
=\frac12\aab{A_{\mu,\eps}u,u}+R_{\mu,\eps}(u),
\end{equation}
where
\[
A_{\mu,\eps}=a_0 I+o(1),
\qquad
|R_{\mu,\eps}(u)|\le C|u|^3
\]
uniformly.  In particular, for some fixed \(c_1,c_2>0\),
\begin{equation}\label{eq:quadratic-gap-static-proof}
c_1|u|^2
\le
W_{\mu_\eps}(x_{\mu,\eps})
-W_{\mu_\eps}(\exp_{x_{\mu,\eps}}u)
\le c_2|u|^2
\end{equation}
inside the cap, and there is a fixed gap outside the cap.

Subtracting the Euler--Lagrange equation at the maximizer and at
\(\exp_{x_{\mu,\eps}}u\) gives the centered Gibbs representation
\begin{equation}\label{eq:centered-Gibbs-static-proof}
\rho_\eps(\exp_{x_{\mu,\eps}}u)
=\rho_\eps(x_{\mu,\eps})
\exp\left(-\frac{
W_{\mu_\eps}(x_{\mu,\eps})
-W_{\mu_\eps}(\exp_{x_{\mu,\eps}}u)}{\eps}\right).
\end{equation}
The quadratic bounds imply the peak-height estimate
\[
C^{-1}\eps^{-(d-1)/2}
\le \rho_\eps(x_{\mu,\eps})
\le C\eps^{-(d-1)/2},
\]
and the fixed outside gap gives
\[
\mu_\eps(\dSphe\setminus B_r(x_{\mu,\eps}))\le C\e^{-c/\eps}.
\]
After the scaling \(u=\sigma_\eps\xi\),
\(\sigma_\eps^2=\eps/a_0\), \eqref{eq:potential-expansion-static-proof} and
\eqref{eq:centered-Gibbs-static-proof} give convergence of the scaled density
to the standard Gaussian in weighted \(L^1(\mathbb R^{d-1})\).  The usual
Laplace expansion also yields the first-moment refinement
\begin{equation}\label{eq:first-moment-Oeps-static-proof}
\int u\d\mu_\eps=O(\eps),
\end{equation}
because the leading Gaussian is even and the first non-even contribution is of
order \(\sqrt\eps\) in the scaled variables.  This proves the moment estimates
\eqref{eq:uniform-first-moment-main}--\eqref{eq:uniform-momentm-main}.

It remains to prove the location \eqref{eq:center-Oeps-main} of the center.  Since
\(x_{\mu,\eps}\) maximizes \(W_{\mu_\eps}\),
\begin{equation}\label{eq:mode-equation-static-proof}
0=P_{x_{\mu,\eps}}^\perp
\int_\dSphe  \e^{x_{\mu,\eps}^\top Ay}Ay\d\mu_\eps(y).
\end{equation}
Using the just-proved moments in coordinates centered at \(x_{\mu,\eps}\), the
integral in \eqref{eq:mode-equation-static-proof} equals
\[
\e^{x_{\mu,\eps}^\top A x_{\mu,\eps}}
A x_{\mu,\eps}+O(\eps),
\]
after tangent projection.  Hence
\begin{equation}\label{eq:PAx-Oeps-static-proof}
P_{x_{\mu,\eps}}^\perp A x_{\mu,\eps}=O(\eps).
\end{equation}
Write \(x_{\mu,\eps}=(\sqrt{1-|z|^2},z)\) in the eigenbasis of \(A\), with
\(z\in\mathbb R^{d-1}\).  The tangent components of \(P_x^\perp Ax\) near
\(e_1\) are
\[
((\lambda_2-\lambda_1)z_2,\dots,(\lambda_d-\lambda_1)z_d)+O(|z|^2).
\]
Since \(\lambda_1>\lambda_2\), \eqref{eq:PAx-Oeps-static-proof} gives
\(z=O(\eps)\).  Thus \(x_{\mu,\eps}=e_1+O(\eps)\).
This proves \zcref{thm:package}(ii) in \(V_+\), and the statement in \(V_-\)
follows by applying the antipodal map.

\smallskip
\noindent\textbf{4. Chord operator estimate.}
Let \(\mu_\eps^0,\mu_\eps^1\in V_+\) be two low-energy critical point families
satisfying \eqref{eq:low-energy-critical-main}, with \(\mu_\eps\) replaced by each \(\mu_\eps^i\),
and set \(\mu_\eps^s=(1-s)\mu_\eps^0+s\mu_\eps^1\).  By the previous step,
the centers of \(\mu_\eps^0\) and \(\mu_\eps^1\) are both \(e_1+O(\eps)\).
Consequently, if \(z(x)=\exp_{e_1}^{-1}(x)\) on a fixed cap and \(z=0\)
outside, then uniformly in \(s\in[0,1]\),
\begin{equation}\label{eq:chord-moments-static-proof}
\mu_\eps^s(\dSphe\setminus B_r(e_1))\le C\e^{-c/\eps},
\qquad
\int_\dSphe  z_i z_j\d\mu_\eps^s
=\sigma_\eps^2\delta_{ij}+o(\eps),
\end{equation}
\[
\int_\dSphe  z_i\d\mu_\eps^s=O(\eps),
\qquad
\int_\dSphe  |z|^m\d\mu_\eps^s=O(\eps^{m/2}).
\]
Here \(i,j\) range over the tangent coordinates corresponding to
\(e_2,\dots,e_d\).

Let \(h\in L^2_0(\mu_\eps^s)\).  Since \(\Pi_{\mu_\eps^s}\) only removes
constants,
\[
\aab{\mathsf K_{\mu_\eps^s}h, h}_{L^2(\mu_\eps^s)}
=\iint K_A(x,y)h(x)h(y)\d\mu_\eps^s(x)\d\mu_\eps^s(y).
\]
On the fixed cap, with \(x=\exp_{e_1}z\), \(y=\exp_{e_1}w\), Taylor expansion
at \((e_1,e_1)\) gives
\begin{equation}\label{eq:kernel-expansion-chord-static-proof}
K_A(x,y)
=\e^{\lambda_1}
+\e^{\lambda_1}\sum_{j=2}^d \lambda_j z_j w_j
+q_1(z)+q_2(w)+R(z,w),
\end{equation}
where \(q_1\) and \(q_2\) depend on only one variable, and
\[
|R(z,w)|\le C\bigl((|z|+|w|)^3+|z|^2|w|^2\bigr)
\]
for small \(z,w\).  The constant term and the one-variable terms vanish in the
quadratic form because \(\int h\d \mu_\eps^s=0\).  The tail contribution and
the remainder in \eqref{eq:kernel-expansion-chord-static-proof} are
\(o(\eps)\Vab{h}_{L^2(\mu_\eps^s)}^2\) by \eqref{eq:chord-moments-static-proof}
and Cauchy--Schwarz.  Therefore
\begin{equation}\label{eq:chord-reduction-static-proof}
\aab{\mathsf K_{\mu_\eps^s}h, h}_{L^2(\mu_\eps^s)}
=\e^{\lambda_1}\sum_{j=2}^d \lambda_j
\left(\int_\dSphe  z_j h\d\mu_\eps^s\right)^2
+o(\eps)\Vab{h}_{L^2(\mu_\eps^s)}^2.
\end{equation}
By \eqref{eq:chord-moments-static-proof}, the Gram matrix of
\(\sigma_\eps^{-1}z_j\) in \(L^2(\mu_\eps^s)\) is \(I+o(1)\).  Hence
\begin{equation}\label{eq:chord-bessel-static-proof}
\sum_{j=2}^d
\left(\sigma_\eps^{-1}\int_\dSphe  z_jh\d\mu_\eps^s\right)^2
\le (1+o(1))\Vab{h}_{L^2(\mu_\eps^s)}^2.
\end{equation}
Combining \eqref{eq:chord-reduction-static-proof} and
\eqref{eq:chord-bessel-static-proof}, and using
\(\sigma_\eps^2=\eps/(\e^{\lambda_1}\lambda_1)\), gives
\[
\aab{\mathsf K_{\mu_\eps^s}h, h}
\le
\left(\frac{\lambda_2}{\lambda_1}+o(1)\right)
\eps\Vab{h}_{L^2(\mu_\eps^s)}^2.
\]
The kernel \(e^{x^\top Ay}\) is positive definite because it is
\(e^{\aab{A^{1/2}x,A^{1/2}y}}\).  Thus \(\mathsf K_{\mu_\eps^s}\) is
nonnegative self-adjoint on \(L^2_0(\mu_\eps^s)\), and the last Rayleigh
bound is exactly \eqref{eq:chord-K-op-main}.  The negative well is obtained by
applying the antipodal map.

\smallskip
\noindent\textbf{5. One-well uniqueness of critical points.}
Let \(\mu^0=\rho^0\omega\) and \(\mu^1=\rho^1\omega\) be two low-energy
critical points in the same positive well.  Their Euler--Lagrange equations
imply
\[
\eps(\log\rho^1-\log\rho^0)-K_A*(\mu^1-\mu^0)=c
\]
for a constant \(c\).  Put \(\nu=\mu^1-\mu^0\),
\(\mu^s=(1-s)\mu^0+s\mu^1\), and
\(h_s=d\nu/d\mu^s\).  Since \(\nu(\dSphe)=0\), multiplying the last equation by
\(\nu\) gives
\begin{equation}\label{eq:EL-difference-static-proof}
0=\eps\int_\dSphe (\log\rho^1-\log\rho^0)\d\nu
-\iint K_A(x,y)\d\nu(x)\d\nu(y).
\end{equation}
Moreover,
\[
\int_\dSphe (\log\rho^1-\log\rho^0)\d\nu
=\int_0^1\int_\dSphe  h_s^2\d \mu^s\d s,
\]
and
\[
\iint K_A\d\nu\d\nu
=\aab{\mathsf K_{\mu^s}h_s,h_s}_{L^2(\mu^s)}
\qquad (0\le s\le1).
\]
Integrating the latter identity in \(s\) and using the chord estimate yields,
for small \(\eps\),
\[
0
\ge
\left(1-\frac{\lambda_2}{\lambda_1}+o(1)\right)
\eps\int_0^1\Vab{h_s}_{L^2(\mu^s)}^2\d s.
\]
Since \(\lambda_2<\lambda_1\), this forces \(h_s=0\) for a.e. \(s\), hence
\(\nu=0\) and \(\mu^0=\mu^1\).  The same proof applies in \(V_-\).

The preceding argument is an asymptotic uniqueness statement for critical
families with \(\Energy\le m_\eps+o(1)\).  Equivalently, after fixing
\(\eps\) sufficiently small, there is a positive threshold
\(\eta_{\textup{crit}}(\eps)\) such that each static well contains at most one
critical point satisfying
\[
\Energy(\mu)\le m_\eps+\eta_{\textup{crit}}(\eps).
\]
Indeed, otherwise one could choose a sequence \(\eps_n\downarrow0\) and two
distinct same-well critical points with excess energy tending to zero, and the
argument above would force them to coincide for large \(n\).

Finally, Step 2 gives at least one minimizer in one of the wells and symmetry
gives its antipodal partner.  Since minimizers have zero excess energy, the
one-well low-energy critical uniqueness gives at most one minimizer in each
well.  Therefore the global minimizer set is exactly
\(\{\mubar^+,\mubar^-\}\), as claimed.

\end{proof}

\subsubsection{Hessian operator}\label{sec:app:lan-hessian}
For the selected branch define the mean-zero projection
\(\Pi_0^\mubar f\coloneqq f-\int_\dSphe f\d\mubar\) and
\begin{equation}\label{eq:linear-objects}
(\Kbar h)(x)\coloneqq \Pi_0^\mubar\!\left[\int_\dSphe  \e^{x^\top Ay}h(y)\d\mubar(y)\right],
\qquad
\Hbar=\eps I-\Kbar.
\end{equation}

\subsubsection{Hessian operator-norm asymptotics}\label{sec:app:lan-hessian-asymptotics}

The next lemma is the exact reduction behind the top scale.

\begin{lemma}[first reduction of the kernel form]\label{lem:reduction}
Assume \zcref{thm:package}. Define the tangent compression of $A$ at $x_\eps$ by
\[
B_\eps\coloneqq P_{x_\eps}AP_{x_\eps}\big|_{T_{x_\eps}\dSphe},
\]
and write its eigenpairs as
\[
\nu_{\eps,1}\ge\cdots\ge\nu_{\eps,d-1}>0,
\qquad
B_\eps\tau_{\eps,m}=\nu_{\eps,m}\tau_{\eps,m}.
\]
Then $\nu_{\eps,m}\to\lambda_{m+1}$ for $m=1,\dots,d-1$. Moreover, if
\[
\mathfrak m_{\eps,m}(h)\coloneqq \sigma_\eps^{-1}\int_\dSphe  \aab{u(x),\tau_{\eps,m}}h(x)\d\mubar(x),
\qquad h\in\Xred,
\]
then, uniformly for $\Vab{h}_{L^2(\mubar)}\le1$,
\begin{equation}\label{eq:reduction}
\frac1\eps\aab{\Kbar h ,h}_{L^2(\mubar)}
=
\frac1{\lambda_1}\sum_{m=1}^{d-1}\nu_{\eps,m}\,\mathfrak m_{\eps,m}(h)^2+o(1).
\end{equation}
\end{lemma}

\begin{proof}
Since $x_\eps\to e_1$, the compressed operators $B_\eps$ converge in operator norm to $A_\perp\coloneqq P_{e_1}AP_{e_1}|_{T_{e_1}\dSphe}$, whose eigenvalues are $\lambda_2,\dots,\lambda_d$. Hence $\nu_{\eps,m}\to\lambda_{m+1}$.

For $h\in\Xred$, the projection disappears inside the quadratic form because $\int h\d\mubar=0$, so
\begin{equation}\label{eq:qfK}
\aab{\Kbar h,h}_{L^2(\mubar)}=\iint_{\dSphe\times\dSphe}\e^{x^\top Ay}h(x)h(y)\d\mubar(x)\d \mubar(y).
\end{equation}
By \eqref{eq:tail-main}, the part of the integral where at least one of $(x,y)$ lies outside $B_r(x_\eps)$ is $o(\eps)\Vab{h}_{L^2(\mubar)}^2$. On the cap, write
\[
x=\exp_{x_\eps}(u),\qquad y=\exp_{x_\eps}(v),\qquad u,v\in T_{x_\eps}\dSphe.
\]
A Taylor expansion at $(u,v)=(0,0)$ gives
\[
x^\top Ay=x_\eps^\top Ax_\eps+u^\top B_\eps v+R_\eps^{(1)}(u,v),
\]
where $R_\eps^{(1)}$ is a sum of terms depending only on $u$ or only on $v$, a term of size $|P_{x_\eps}Ax_\eps|\,|u|\,|v|$, and terms of total degree at least $3$ in $(u,v)$. Since $x_\eps\to e_1$, one has $P_{x_\eps}Ax_\eps\to0$.
Exponentiating once more yields
\[
\e^{x^\top Ay}=\e^{x_\eps^\top Ax_\eps}\Bigl(1+u^\top B_\eps v\Bigr)+R_\eps(u,v),
\]
with
\[
|R_\eps(u,v)|\le C\Bigl(|P_{x_\eps}Ax_\eps|\,|u|\,|v|+(|u|+|v|)^3+|u|^2|v|^2\Bigr).
\]
The constant term disappears in \eqref{eq:qfK} because $h$ has mean zero; the terms depending only on one variable disappear for the same reason. Therefore
\[
\aab{\Kbar h,h}_{L^2(\mubar)}
=
\e^{x_\eps^\top Ax_\eps}
\sum_{m=1}^{d-1}\nu_{\eps,m}
\left(\int_\dSphe  \aab{u(x),\tau_{\eps,m}}h(x)\d\mubar(x)\right)^2
+o(\eps)\Vab{h}_{L^2(\mubar)}^2.
\]
Finally, $\e^{x_\eps^\top Ax_\eps}=\e^{\lambda_1}+o(1)$ and $\sigma_\eps^2=\eps/(\e^{\lambda_1}\lambda_1)$, so dividing by $\eps$ gives \eqref{eq:reduction}.
\end{proof}

\begin{theorem}[exact Hessian top scale]\label{thm:K-op}
Assume \zcref{thm:package}. Then
\begin{equation}\label{eq:K-op}
\Vab{\Kbar}_{\op}=\frac{\lambda_2}{\lambda_1}\,\eps+o(\eps),
\qquad \eps\downarrow0,
\end{equation}
as an operator on $\Xred=L^2_0(\mubar)$. Consequently,
\begin{equation}\label{eq:H-coercive}
\aab{\Hbar h,{h}}_{L^2(\mubar)}
\ge
\left(\frac{\lambda_1-\lambda_2}{\lambda_1}\eps+o(\eps)\right)
\Vab{h}_{L^2(\mubar)}^2,
\qquad h\in\Xred.
\end{equation}
\end{theorem}

\begin{proof}
Define the normalized transverse linear modes
\[
\psi_{\eps,m}(x)\coloneqq \sigma_\eps^{-1}\aab{u(x),\tau_{\eps,m}},
\qquad m=1,\dots,d-1.
\]
By \eqref{eq:moment1-main}--\eqref{eq:momentm-main},
\[
\aab{\psi_{\eps,m},\psi_{\eps,n}}_{L^2(\mubar)}=\delta_{mn}+o(1).
\]
Hence the Gram matrix of $\{\psi_{\eps,1},\dots,\psi_{\eps,d-1}\}$ is $I+o(1)$, and therefore
\begin{equation}\label{eq:bessel}
\sum_{m=1}^{d-1}\mathfrak m_{\eps,m}(h)^2\le (1+o(1))\Vab{h}_{L^2(\mubar)}^2
\end{equation}
uniformly for all $h\in\Xred$.
Applying \zcref{lem:reduction} gives
\[
\frac{1}{\eps}\aab{\Kbar h,h}_{L^2(\mubar)}
\le
\frac{\nu_{\eps,1}}{\lambda_1}(1+o(1))\Vab{h}_{L^2(\mubar)}^2,
\]
so
\[
\Vab{\Kbar}_{\op}
\le
\frac{\nu_{\eps,1}}{\lambda_1}\eps+o(\eps)
=
\frac{\lambda_2}{\lambda_1}\eps+o(\eps).
\]
For the matching lower bound, choose
\[
h_{\eps,1}\coloneqq \psi_{\eps,1}-\int_\dSphe  \psi_{\eps,1}\d\mubar\in\Xred.
\]
Then $\Vab{h_{\eps,1}}_{L^2(\mubar)}=1+o(1)$, $\mathfrak m_{\eps,1}(h_{\eps,1})=1+o(1)$, and $\mathfrak m_{\eps,m}(h_{\eps,1})=o(1)$ for $m\neq1$. Plugging this test function into \eqref{eq:reduction} yields the reverse bound and proves \eqref{eq:K-op}. Since $\Hbar=\eps I-\Kbar$ and $\Kbar\ge0$, \eqref{eq:H-coercive} follows immediately.
\end{proof}

\subsubsection{Relation to transport second variations}\label{sec:app:lan-transport}
\zcref{thm:K-op} perturbs the minimizer through the density ratio, $\mu_s=(1+sh)\mubar$, which is adapted to the $L^2(\mubar)$ inner product: the second variation of the entropy is exactly the flat term $\eps\Vab{h}_{L^2(\mubar)}^2$, without curvature corrections.
A transport perturbation $\mu_t=(\operatorname{id}+t\xi)_\#\mubar$ corresponds to $h=-\Div(\rhobar\xi)/\rhobar$; the associated displacement second variation carries additional curvature and shape terms, while its coercivity content coincides with that of \zcref{thm:K-op}.

\begin{proof}[Proof of \zcref{prop:energy-landscape}]
The Gibbs fixed-point equation and smooth positivity of the one-well minimizer follow from the Euler--Lagrange equation in the direct-method part of the proof of \zcref{thm:package}.  The one-well uniqueness and antipodal symmetry give
\[
\Minimizers
=\argmin_{\mu\in\PdSphe}\Energy(\mu)
=\{\mubar^+,\mubar^-\},
\qquad \mubar^-=S_\#\mubar^+.
\]
The Hessian at either branch is \(\Hbar=\eps I-\Kbar\) on the mean-zero space.  By \zcref{thm:K-op},
\[
\Vab{\Kbar}_{\op}=\frac{\lambda_2}{\lambda_1}\eps+o(\eps).
\]
After decreasing the small-noise threshold, this gives
\[
\eps\Vab{h}_{L^2(\mubar)}^2
\ge
\aab{h,\Hbar h}_{L^2(\mubar)}
\ge
\frac{\eps}{2}\left(1-\frac{\lambda_2}{\lambda_1}\right)
\Vab{h}_{L^2(\mubar)}^2,
\qquad h\in L^2_0(\mubar).
\]
Finally, \zcref{prop:explicit-small-temperature-excited-critical-gap} gives
\[
\liminf_{\eps\downarrow0}\Delta_{\textup{crit}}(\eps)
\ge
\frac12(\e^{\lambda_1}-Q_A)
=
\frac12\min\{\e^{\lambda_1}-\e^{\lambda_2},\sinh\lambda_1\}.
\]
This is exactly the asserted landscape statement.
\end{proof}

%% file: contents/proofs/adjoint_reorg_cr.tex
\subsubsection{Operator notation and sign convention}\label{sec:app:adj-notation}

Throughout this appendix we work on the positive minimizer branch
\(\mubar=\rhobar\,\omega\) supplied by \zcref{prop:energy-landscape}.  We write
\[
\varPi_0^\mubar f\coloneqq f-\int_{\dSphe}f\d\mubar
\]
for the orthogonal projection of \(L^2(\mubar)\) onto the mean-zero subspace \(L^2_0(\mubar)\).  When no confusion is possible we write
\(\varPi_0\) for \(\varPi_0^\mubar\).

The Hessian kernel and the reduced Hessian at \(\mubar\) are
\begin{equation}\label{eq:adjoint-KH-def}
\begin{aligned}
    (\Kbar h)(x)
    &\coloneqq
    \varPi_0\pab{\int_{\dSphe}\e^{\aab{x,Ay}}h(y)\d\mubar(y)},
    \\
    \Hop h&\coloneqq \eps h-\Kbar h,
    \qquad h\in L^2_0(\mubar).
\end{aligned}
\end{equation}
By \eqref{eq:PD_Hessian} in \zcref{prop:energy-landscape}, after decreasing the small-noise threshold if necessary,
\begin{equation}\label{eq:adjoint-H-coercive}
    \frac{\eps}{2}\pab{1-\frac{\lambda_2}{\lambda_1}}
    \Vab{h}_{L^2(\mubar)}^2
    \le
    \aab{h,\Hop h}_{L^2(\mubar)}
    \le
    \eps\Vab{h}_{L^2(\mubar)}^2,
    \qquad h\in L^2_0(\mubar).
\end{equation}
We use the equivalent Hilbert norm
\begin{equation}\label{eq:Hop-norms-adjoint}
    \Vab{h}_{\Hop}^2\coloneqq\aab{h,\Hop h}_{L^2(\mubar)},
    \qquad
    \Vab{z}_{\Hop^{-1}}^2\coloneqq\aab{\Hop^{-1}z,z}_{L^2(\mubar)}.
\end{equation}

We use the sign convention fixed in the main text:
\begin{equation}\label{eq:wlap-sign-adjoint}
    \wLaplacian\phi
    \coloneqq
    \rhobar^{-1}\Div\pab{\rhobar\nabla\phi}.
\end{equation}
Thus \(\wLaplacian\) agrees with the Laplace--Beltrami operator when
\(\rhobar\equiv1\), and it is non-positive on \(L^2_0(\mubar)\):
\begin{equation}\label{eq:wlap-form-adjoint}
    \aab{\wLaplacian f,g}_{L^2(\mubar)}
    =-
    \int_{\dSphe}\nabla f\cdot\nabla g\d\mubar.
\end{equation}
The forward linearized generator and its \(L^2(\mubar)\)-adjoint are
\begin{equation}\label{eq:A-Ad-adjoint}
    \Abar\coloneqq\wLaplacian\Hop,
    \qquad
    \Adbar\coloneqq\Hop\wLaplacian.
\end{equation}
The positive backward generator is therefore
\begin{equation}\label{eq:L-positive-adjoint}
    L_\eps\coloneqq-\Adbar=-\Hop\wLaplacian.
\end{equation}
Viewed on the Hilbert space \((L^2_0(\mubar),\aab{\cdot,\cdot}_{\Hop^{-1}})\),
\(L_\eps\) is self-adjoint and non-negative, with closed quadratic form
\begin{equation}\label{eq:q-adjoint-def}
    q_{L_\eps}[z,w]
    \coloneqq
    \int_{\dSphe}\nabla z\cdot\nabla w\d\mubar,
    \qquad
    D(q_{L_\eps})=H^1(\dSphe)\cap L^2_0(\mubar).
\end{equation}
Indeed, for smooth \(z,w\),
\begin{equation}\label{eq:q-operator-identity-adjoint}
\begin{aligned}
    \aab{L_\eps z,w}_{\Hop^{-1}}
    &=\aab{\Hop^{-1}\pab{-\Hop\wLaplacian z},w}_{L^2(\mubar)}  \\
    &=-\aab{\wLaplacian z,w}_{L^2(\mubar)}
    =\int_{\dSphe}\nabla z\cdot\nabla w\d\mubar.
\end{aligned}
\end{equation}
For nonzero \(g_\ell\in H^1(\dSphe)\cap L^2_0(\mubar)\), this gives the Rayleigh quotient
\begin{equation}\label{eq:Rayleigh-adjoint-proof}
    \Rayleigh
    \coloneqq
    \frac{q_{L_\eps}[g_\ell,g_\ell]}{\Vab{g_\ell}_{\Hop^{-1}}^2}
    =
    \frac{\displaystyle\int_{\dSphe}\Vab{\nabla g_\ell}^2\d\mubar}
    {\displaystyle\Vab{g_\ell}_{\Hop^{-1}}^2}.
\end{equation}

The control operator generated by the parameter-linear FFN is
\begin{equation}\label{eq:Bbar-adjoint-proof}
    (\Bbar W)(x)
    \coloneqq
    -\rhobar(x)^{-1}\Div\pab{\rhobar(x)P_x^\perp(W\sigma(x))},
    \qquad W\in\R^{d\times p}.
\end{equation}
Its adjoint is taken with respect to the \(L^2(\mubar)\)-inner product on
\(L^2_0(\mubar)\) and the Frobenius inner product on \(\R^{d\times p}\).  For smooth
\(z\), integration by parts gives
\begin{equation}\label{eq:Bstar-adjoint-proof}
    \Bbar^\ast z
    =
    \int_{\dSphe}\nabla z(x)\sigma(x)^\top\d\mubar(x).
\end{equation}

\subsubsection{Proof of \texorpdfstring{\zcref{lem:1GDparam}}{Lemma 3.4}}\label{sec:app:adj-1gd}

The controlled gradient-flow form of \eqref{eq:ODE-Net} is, in density variables,
\begin{equation}\label{eq:controlled-density-adjoint-proof}
    \partial_t\rho_t
    =
    \Div\pab{\rho_t\nabla\frac{\delta\Energy}{\delta\mu}(x;\rho_t\omega)}
    -\Div\pab{\rho_tP_x^\perp(W_t\sigma(x))}.
\end{equation}
At \(W\equiv0\), the initial datum \(\mu_0=\mubar\) is stationary because
\(\delta\Energy/\delta\mu(\mubar)\) is constant.  If
\(\rho_t=\rhobar(1+h_t)\), then the first variation of \eqref{eq:controlled-density-adjoint-proof}
at \((h,W)=(0,0)\) is
\begin{equation}\label{eq:linearized-state-adjoint-proof}
    \partial_t h_t=\Abar h_t+\Bbar W_t,
    \qquad h_0=0,
\end{equation}
with \(\Abar=\wLaplacian\Hop\).

The functional derivative \(g_\ell=\fdv{\ell}{\mu}[\mubar]\) is defined only up to an
additive constant.  We take its \(\mubar\)-mean-zero representative, still denoted
\(g_\ell\).  Let \(\phi_t\) solve the adjoint equation
\begin{equation}\label{eq:adjoint-equation-proof}
    \partial_t\phi_t+\Adbar\phi_t=0,
    \qquad
    \phi_T=g_\ell.
\end{equation}
Since \(\Adbar=\Hop\wLaplacian\), this is exactly
\eqref{eq:bkwd_wDiffusion}.  Equivalently,
\begin{equation}\label{eq:costate-semigroup-adjoint-proof}
    \phi_t=\e^{(T-t)\Adbar}g_\ell
    =\e^{-(T-t)L_\eps}g_\ell.
\end{equation}
The coefficients are smooth and \(g_\ell\in C^\infty(\dSphe)\), so the solution is classical.

Let \(V\in L^2(0,T;\R^{d\times p})\), and let \(h^V\) be the linearized state response
from \eqref{eq:linearized-state-adjoint-proof}.  Then
\begin{equation}\label{eq:loss-differential-adjoint-proof}
\begin{aligned}
    D\ell(\mubar)[h_T^V\mubar]
    &=\aab{g_\ell,h_T^V}_{L^2(\mubar)}                                      \\
    &=\int_0^T\aab{\Bbar^\ast\phi_t,V_t}_{\frob}\d t.
\end{aligned}
\end{equation}
Thus the unregularized terminal-loss gradient at the zero initialization is
\begin{equation}\label{eq:gradient-zero-adjoint-proof}
    G_t\coloneqq\Bbar^\ast\phi_t
    =
    \int_{\dSphe}\nabla\phi_t(x)\sigma(x)^\top\d\mubar(x).
\end{equation}
The quadratic regularizer in \eqref{eq:mf-optimal-control} has zero derivative at
\(W\equiv0\), so the one-step update with learning rate \(\alpha\) is
\begin{equation}\label{eq:oneGD-adjoint-proof}
    W_t^{\oneGD}
    =-\alpha G_t
    =-\alpha
    \int_{\dSphe}\nabla\phi_t(x)\sigma(x)^\top\d\mubar(x),
\end{equation}
which is \eqref{eq:1GDparam}.

\subsubsection{Finite-time Gramian and the visibility factor}\label{sec:app:adj-gramian}

For \(0\le t\le T\), define the finite-time control Gramian
\begin{equation}\label{eq:Gramian-adjoint-proof}
    \cGram_t
    \coloneqq
    \int_0^t \e^{u\Abar}\Bbar\Bbar^\ast\e^{u\Adbar}\d u .
\end{equation}
The linearized state generated by the one-step parameter \(W^{\oneGD}\) satisfies
\begin{equation}\label{eq:linearized-oneGD-state-adjoint-proof}
    \partial_t h_t^{\mathrm{lin},\alpha}
    =\Abar h_t^{\mathrm{lin},\alpha}-\alpha\Bbar\Bbar^\ast\phi_t,
    \qquad
    h_0^{\mathrm{lin},\alpha}=0.
\end{equation}
Using \(\phi_s=\e^{(t-s)\Adbar}\phi_t\) for \(0\le s\le t\), Duhamel's formula gives
\begin{equation}\label{eq:linearized-oneGD-gramian-adjoint-proof}
    h_t^{\mathrm{lin},\alpha}=-\alpha\cGram_t\phi_t.
\end{equation}

For \(z\in L^2_0(\mubar)\), define
\begin{equation}\label{eq:Obs-adjoint-proof}
    \mathcal O_t(z)
    \coloneqq
    \int_0^t\Vab{\Bbar^\ast\e^{u\Adbar}z}_{\frob}^2\d u.
\end{equation}
Then
\begin{equation}\label{eq:Obs-Gramian-adjoint-proof}
    \mathcal O_t(z)=\aab{z,\cGram_tz}_{L^2(\mubar)}.
\end{equation}
For \(z=\phi_t\), this is precisely the squared norm of the unregularized loss-gradient
signal on \([0,t]\):
\begin{equation}\label{eq:Omega-main-appendix-match}
    \mathcal O_t(\phi_t)
    =
    \int_0^t\Vab{G_s}_{\frob}^2\d s
    =
    \Vab{\nabla J(0)}_{L^2(0,t)}^2,
\end{equation}
where the last equality uses the notation of \zcref{sec:adjoint}.  Hence the visibility factor
in \eqref{eq:Rayleighratio} is
\begin{equation}\label{eq:Omega-adjoint-proof}
    \varOmega_{t,T}(g_\ell)
    =
    \frac{\mathcal O_t(\phi_t)}{\Vab{\phi_t}_{\Hop^{-1}}^2}.
\end{equation}
By Cauchy--Schwarz in the \(\Hop^{-1}\)-geometry,
\begin{equation}\label{eq:Omega-Gramian-comparison-adjoint-proof}
\begin{aligned}
    \mathcal O_t(z)
    &=\aab{\Hop^{-1/2}z,\Hop^{1/2}\cGram_tz}_{L^2(\mubar)}             \\
    &\le
    \Vab{z}_{\Hop^{-1}}\Vab{\Hop^{1/2}\cGram_tz}_{L^2(\mubar)}.
\end{aligned}
\end{equation}
Consequently,
\begin{equation}\label{eq:Omega-lower-Gramian-adjoint-proof}
    \Vab{\Hop^{1/2}\cGram_tz}_{L^2(\mubar)}
    \ge
    \frac{\mathcal O_t(z)}{\Vab{z}_{\Hop^{-1}}}
    \quad(z\ne0).
\end{equation}

\subsubsection{Local Taylor lower bound for the energy gap}\label{sec:app:adj-taylor}

\begin{lemma}[local energy lower bound in the density chart]\label{lem:local-energy-lower-adjoint}
There exists a chart radius \(r_{\textup{ch}}>0\) such that, whenever
\(\mu=(1+h)\mubar\), \(h\in L^2_0(\mubar)\), and
\(\Vab{h}_{L^\infty(\dSphe)}\le r_{\textup{ch}}\), one has
\begin{equation}\label{eq:local-energy-lower-adjoint}
    \Energygap(\mu)
    \ge
    \frac14\Vab{h}_{\Hop}^2.
\end{equation}
\end{lemma}
\begin{proof}
The Euler--Lagrange equation at \(\mubar\) and \(\int h\d\mubar=0\) give the exact expansion
\begin{equation}\label{eq:energy-taylor-adjoint}
\begin{aligned}
    \Energy((1+h)\mubar)-\Energy(\mubar)
    &=\eps\int_{\dSphe}\pab{(1+h)\log(1+h)-h}\d\mubar       \\
    &\quad
    -\frac12\iint_{\dSphe\times\dSphe}
    \e^{\aab{x,Ay}}h(x)h(y)\d\mubar(x)\d\mubar(y).
\end{aligned}
\end{equation}
Since \((1+z)\log(1+z)-z=\frac12z^2+O(|z|^3)\) for \(|z|\le r_{\textup{ch}}\),
\begin{equation}\label{eq:energy-taylor-remainder-adjoint}
    \Energy((1+h)\mubar)-\Energy(\mubar)
    =\frac12\Vab{h}_{\Hop}^2+R_E(h),
    \qquad
    |R_E(h)|\le C\eps\Vab{h}_{L^\infty}\Vab{h}_{L^2(\mubar)}^2.
\end{equation}
Using the lower bound in \eqref{eq:adjoint-H-coercive}, we choose
\(r_{\textup{ch}}>0\) so small that the remainder is bounded by
\(\frac14\Vab{h}_{\Hop}^2\).  Since \(\mubar\) is a global minimizer,
\(\Energygap((1+h)\mubar)=\Energy((1+h)\mubar)-\Energy(\mubar)\), and
\eqref{eq:local-energy-lower-adjoint} follows.
\end{proof}

We also record the corresponding first-variation expansion, used below in the nonlinear response.
For \(\mu=(1+h)\mubar\), set
\begin{equation}\label{eq:Phi-def-adjoint-proof}
    \Phi(h)
    \coloneqq
    \varPi_0\pab{\fdv{\Energy}{\mu}((1+h)\mubar)-\fdv{\Energy}{\mu}(\mubar)}.
\end{equation}
Then
\begin{equation}\label{eq:Phi-H-remainder-adjoint-proof}
    \Phi(h)=\Hop h+R_F(h),
    \qquad
    R_F(h)=\eps\varPi_0\pab{\log(1+h)-h},
\end{equation}
and for every integer \(s>\dim(\dSphe)/2+2\),
\begin{equation}\label{eq:firstvar-remainder-adjoint-proof}
    \Vab{R_F(h)}_{H^s(\dSphe)}
    \le
    C_s\eps\Vab{h}_{H^s(\dSphe)}^2
\end{equation}
whenever \(\Vab{h}_{L^\infty}\le r_{\textup{ch}}\).

\subsubsection{Small-learning-rate nonlinear response}\label{sec:app:adj-response}

\begin{proposition}[small-learning-rate nonlinear response]\label{prop:small-alpha-nonlinear-response}
Fix a terminal-layer thickness \(r>0\) and an integer
\(s>\dim(\dSphe)/2+2\).  Assume \(g_\ell\in H^{s+1}(\dSphe)\cap L^2_0(\mubar)\).
Then there exist constants \(\overline\alpha>0\) and \(C_R<\infty\) such that, for every
\(0<\alpha\le\overline\alpha\), the one-step-GD trajectory
\(\mu_t^{W^{\oneGD}}\) remains in the selected local density chart on \([0,T]\).  Moreover,
for all \(t\in[T-r,T]\), if
\[
    \mu_t^{W^{\oneGD}}=(1+h_t^\alpha)\mubar,
\]
then
\begin{equation}\label{eq:nonlinear-response-estimate-main}
    h_t^\alpha=-\alpha\cGram_t\phi_t+R_t^\alpha,
    \qquad
    \Vab{R_t^\alpha}_{\Hop}\le C_R\alpha^2.
\end{equation}
\end{proposition}
\begin{proof}
Let \(r_{\textup{ch}}\) be a chart radius for which
\zcref{lem:local-energy-lower-adjoint} and \eqref{eq:firstvar-remainder-adjoint-proof} hold.
Inside the chart write
\[
    \mu_t^{W^{\oneGD}}=(1+h_t^\alpha)\mubar,
    \qquad
    h_t^\alpha\in L^2_0(\mubar).
\]
Using \eqref{eq:controlled-density-adjoint-proof}, \eqref{eq:Phi-def-adjoint-proof}, and
\eqref{eq:Phi-H-remainder-adjoint-proof}, the chart equation is
\begin{equation}\label{eq:chart-equation-adjoint-proof}
    \partial_t h_t
    =
    \Abar h_t+\Bbar W_t+\mathcal R(h_t,W_t),
\end{equation}
where
\begin{equation}\label{eq:nonlinear-remainder-adjoint-proof}
\begin{aligned}
    \mathcal R(h,W)
    &\coloneqq
    \wLaplacian R_F(h)
    +\rhobar^{-1}\Div\pab{h\rhobar\nabla\Phi(h)}
    -\rhobar^{-1}\Div\pab{h\rhobar P_x^\perp(W\sigma(x))}.
\end{aligned}
\end{equation}
The Sobolev algebra property and smoothness of \(\rhobar\) and \(\sigma\) imply
\begin{equation}\label{eq:nonlinear-remainder-tame-adjoint-proof}
    \Vab{\mathcal R(h,W)}_{H^{s-2}(\dSphe)}
    \le
    C\pab{
    \Vab{h}_{H^s(\dSphe)}^2
    +\Vab{h}_{H^s(\dSphe)}\Vab{W}_{\frob}}
\end{equation}
for \(\Vab{h}_{L^\infty}<r_{\textup{ch}}\).

Let
\begin{equation}\label{eq:zeta-adjoint-proof}
    \zeta_t\coloneqq-\cGram_t\phi_t.
\end{equation}
A direct differentiation using \eqref{eq:Gramian-adjoint-proof} and
\(\partial_t\phi_t+\Adbar\phi_t=0\) gives
\begin{equation}\label{eq:zeta-equation-adjoint-proof}
    \partial_t\zeta_t=\Abar\zeta_t-\Bbar\Bbar^\ast\phi_t,
    \qquad
    \zeta_0=0.
\end{equation}
Thus \(\alpha\zeta_t\) is the first-order response forced by
\(W_t^{\oneGD}=-\alpha\Bbar^\ast\phi_t\).  By standard parabolic estimates for smooth
coefficients and \(g_\ell\in H^{s+1}\),
\begin{equation}\label{eq:zeta-Hs-bound-adjoint-proof}
    \sup_{0\le t\le T}\Vab{\zeta_t}_{H^s(\dSphe)}\le C_\zeta.
\end{equation}
Moreover,
\begin{equation}\label{eq:oneGD-W-bound-adjoint-proof}
    \sup_{0\le t\le T}\Vab{W_t^{\oneGD}}_{\frob}
    \le C_W\alpha.
\end{equation}

On the maximal chart interval set
\[
    Z_t^\alpha\coloneqq h_t^\alpha-\alpha\zeta_t
    =h_t^\alpha+\alpha\cGram_t\phi_t.
\]
Subtracting \(\alpha\) times \eqref{eq:zeta-equation-adjoint-proof} from
\eqref{eq:chart-equation-adjoint-proof} gives
\begin{equation}\label{eq:Z-equation-adjoint-proof}
    \partial_t Z_t^\alpha
    =\Abar Z_t^\alpha+\mathcal R(h_t^\alpha,W_t^{\oneGD}),
    \qquad
    Z_0^\alpha=0.
\end{equation}
Local parabolic stability gives, for every \(T'\) contained in the chart interval,
\begin{equation}\label{eq:Z-stability-adjoint-proof}
    \sup_{0\le t\le T'}\Vab{Z_t^\alpha}_{H^s(\dSphe)}
    \le
    C_{\textup{par}}\int_0^{T'}
    \Vab{\mathcal R(h_\tau^\alpha,W_\tau^{\oneGD})}_{H^{s-2}(\dSphe)}\d\tau.
\end{equation}
Bootstrap
\(\sup_{0\le t\le T'}\Vab{h_t^\alpha}_{H^s}\le M_{\textup{boot}}\alpha\), with
\(M_{\textup{boot}}>2C_\zeta\).  Then \eqref{eq:nonlinear-remainder-tame-adjoint-proof} and
\eqref{eq:oneGD-W-bound-adjoint-proof} give
\[
    \sup_{0\le t\le T'}
    \Vab{\mathcal R(h_t^\alpha,W_t^{\oneGD})}_{H^{s-2}}
    \le C\alpha^2.
\]
Therefore \eqref{eq:Z-stability-adjoint-proof} yields
\[
    \sup_{0\le t\le T'}\Vab{Z_t^\alpha}_{H^s(\dSphe)}\le C_Z\alpha^2.
\]
For \(\alpha\le\overline\alpha\) small, this improves the bootstrap bound and, by Sobolev
embedding, keeps \(\Vab{h_t^\alpha}_{L^\infty}<r_{\textup{ch}}\).  The standard continuation
criterion therefore gives chart retention on the whole interval \([0,T]\).
Finally, \eqref{eq:adjoint-H-coercive} gives
\(\Vab{f}_{\Hop}\lesssim\Vab{f}_{L^2(\mubar)}\lesssim\Vab{f}_{H^s(\dSphe)}\), so
\[
    \Vab{Z_t^\alpha}_{\Hop}\le C_R\alpha^2.
\]
Since \(Z_t^\alpha=h_t^\alpha+\alpha\cGram_t\phi_t\), this is
\eqref{eq:nonlinear-response-estimate-main}.
\end{proof}

\subsubsection{Proof of \texorpdfstring{\zcref{prop:one-step-escape}}{Proposition 3.5}}\label{sec:app:adj-escape}

\begin{lemma}[Rayleigh--Jensen lower bound for the costate]\label{lem:jensen-costate-adjoint-proof}
Assume \(g_\ell\ne0\) and \(g_\ell\in H^1(\dSphe)\cap L^2_0(\mubar)\).  Then, for every
\(r\ge0\),
\begin{equation}\label{eq:Jensen-used-adjoint-proof}
    \Vab{\e^{r\Adbar}g_\ell}_{\Hop^{-1}}
    \ge
    \Vab{g_\ell}_{\Hop^{-1}}
    \exp\pab{-r\Rayleigh}.
\end{equation}
\end{lemma}
\begin{proof}
Since \(L_\eps=-\Adbar\) is self-adjoint and non-negative on
\((L^2_0(\mubar),\aab{\cdot,\cdot}_{\Hop^{-1}})\), the spectral theorem gives a probability
measure \(\nu_{g_\ell}\), normalized by \(\Vab{g_\ell}_{\Hop^{-1}}^2\), such that
\[
    \Vab{\e^{r\Adbar}g_\ell}_{\Hop^{-1}}^2
    =
    \Vab{g_\ell}_{\Hop^{-1}}^2
    \int_0^\infty \e^{-2r\lambda}\d\nu_{g_\ell}(\lambda).
\]
Moreover,
\[
    \int_0^\infty\lambda\d\nu_{g_\ell}(\lambda)=\Rayleigh.
\]
The map \(\lambda\mapsto\e^{-2r\lambda}\) is convex, so Jensen's inequality gives
\[
    \int_0^\infty \e^{-2r\lambda}\d\nu_{g_\ell}(\lambda)
    \ge
    \e^{-2r\Rayleigh}.
\]
Taking square roots proves the claim.
\end{proof}

\begin{proof}[Proof of \zcref{prop:one-step-escape}]
If \(g_\ell=0\), then the one-step parameter vanishes and the statement is understood with
zero right-hand side.  Assume \(g_\ell\ne0\).  Let \(t\in[T-r,T]\) and set
\(z_t\coloneqq\phi_t=\e^{(T-t)\Adbar}g_\ell\).  By
\zcref{prop:small-alpha-nonlinear-response},
\[
    h_t^\alpha=-\alpha\cGram_tz_t+R_t^\alpha,
    \qquad
    \Vab{R_t^\alpha}_{\Hop}\le C_R\alpha^2.
\]
Therefore
\begin{equation}\label{eq:Hnorm-lower-adjoint-proof}
    \Vab{h_t^\alpha}_{\Hop}
    \ge
    \alpha\Vab{\Hop^{1/2}\cGram_tz_t}_{L^2(\mubar)}-C_R\alpha^2.
\end{equation}
Using \eqref{eq:Omega-lower-Gramian-adjoint-proof} and the definition
\eqref{eq:Omega-adjoint-proof},
\begin{equation}\label{eq:Omega-Hnorm-lower-adjoint-proof}
    \Vab{\Hop^{1/2}\cGram_tz_t}_{L^2(\mubar)}
    \ge
    \varOmega_{t,T}(g_\ell)\Vab{z_t}_{\Hop^{-1}}.
\end{equation}
Finally, \zcref{lem:jensen-costate-adjoint-proof} gives
\begin{equation}\label{eq:costate-lower-final-adjoint-proof}
    \Vab{z_t}_{\Hop^{-1}}
    \ge
    \Vab{g_\ell}_{\Hop^{-1}}\exp\pab{-(T-t)\Rayleigh}.
\end{equation}
Combining \eqref{eq:Hnorm-lower-adjoint-proof}--\eqref{eq:costate-lower-final-adjoint-proof} and
then taking positive parts yields
\[
    \Vab{h_t^\alpha}_{\Hop}
    \ge
    \pab{
    \alpha\varOmega_{t,T}(g_\ell)\Vab{g_\ell}_{\Hop^{-1}}
    \exp\pab{-(T-t)\Rayleigh}
    -C_R\alpha^2}_+.
\]
The trajectory remains in the density chart by \zcref{prop:small-alpha-nonlinear-response}; hence
\zcref{lem:local-energy-lower-adjoint} applies and gives
\[
    \Energygap(\mu_t^{W^{\oneGD}})
    \ge
    \frac14\Vab{h_t^\alpha}_{\Hop}^2.
\]
This proves the asserted lower bound.
\end{proof}

%% file: contents/proofs/exp_turnpike_reorg_sphere_cr.tex
\subsubsection{Weighted Laplacian and one-well Poincaré estimates}\label{sec:app:poincare}
\begin{equation}\label{eq:linear-objects-2}
\Pbar\phi\coloneqq \rhobar^{-1}\Div\bigl(\rhobar\nabla\phi\bigr).
\end{equation}

Both \(\Kbar\) and \(\Pbar\) preserve \(\Xred=L^2_0(\mubar)\), since the projection is built into \(\Kbar\) and \(\int_\dSphe\Pbar\phi\d\mubar=\int_\dSphe\Div(\rhobar\nabla\phi)\d\omega=0\) on the closed sphere.

\begin{proposition}[weighted Poincare gap]\label{prop:poincare}
For every fixed $\eps>0$, there exists $\kappa_P(\eps)>0$ such that
\begin{equation}\label{eq:poincare}
-\aab{\Pbar z,z}_{L^2(\mubar)}\ge \kappa_P(\eps)\Vab{z}_{L^2(\mubar)}^2,
\qquad z\in D(\Pbar)\cap\Xred.
\end{equation}
\end{proposition}

\begin{proof}
Since $\rhobar\in C^\infty(\dSphe)$ and $\inf_{\dSphe}\rhobar>0$, the operator $\Pbar$ is uniformly elliptic and self-adjoint on the compact sphere $\dSphe$. Its kernel consists of the constants, so the first positive eigenvalue of $-\Pbar$ gives \eqref{eq:poincare} on the mean-zero subspace.
\end{proof}

For the selected branch define
\begin{equation}\label{eq:branch-potential-def}
W_\eps^*(x)\coloneqq \int_\dSphe  K_A(x,y)\d\mubar(y),
\qquad
U_\eps(x)\coloneqq -W_\eps^*(x),
\qquad
U_0(x)\coloneqq -\e^{\lambda_1 x_1}.
\end{equation}
The Euler--Lagrange equation gives
\begin{equation}\label{eq:gibbs-selected-branch}
\rhobar(x)=Z_\eps^{-1}\exp\left(-\frac{U_\eps(x)}{\eps}\right).
\end{equation}

\begin{lemma}[\(C^3\)-one-well convergence of the branch potential]
\label{lem:C3-one-well-potential}
Assume \zcref{thm:package}.  Then there exists
\(\delta_3(\eps;A,d)\downarrow0\) such that
\begin{equation}\label{eq:C3-one-well-potential}
\Vab{W_\eps^*-K_A(\cdot,e_1)}_{C^3(\dSphe)}
+
\Vab{U_\eps-U_0}_{C^3(\dSphe)}
\le \delta_3(\eps;A,d).
\end{equation}
More concretely, after decreasing the small-noise threshold if
necessary,
\begin{equation}\label{eq:C3-one-well-rate}
\delta_3(\eps;A,d)\le C_{3,A,d}\eps^{1/2}.
\end{equation}
\end{lemma}

\begin{proof}
For \(0\le j\le3\), the map
\((x,y)\mapsto \nabla_x^jK_A(x,y)\) is smooth on the compact set
\(\dSphe\times\dSphe\).  Hence it is uniformly Lipschitz in the second variable.  By
\zcref{thm:package}, the selected branch satisfies
\(x_\eps=e_1+O(\eps)\), the tail estimate \eqref{eq:tail-main}, and the moment
bounds \eqref{eq:momentm-main}.  Since \(u=\exp_{x_\eps}^{-1}(y)\) on the good
cap and is set to zero outside,
\[
\int_\dSphe  \dist_\dSphe(y,e_1)\d\mubar(y)
\le
C\left(\int_\dSphe  |u(y)|^2\d\mubar(y)\right)^{1/2}
+C\mubar(\dSphe\setminus B_r(x_\eps))+C \dist_\dSphe(x_\eps,e_1)
\le C_{A,d}\eps^{1/2}.
\]
Therefore, for \(0\le j\le3\),
\[
\sup_{x\in\dSphe}
\left|\nabla_x^j W_\eps^*(x)-\nabla_x^jK_A(x,e_1)\right|
\le
C_{j,A,d}\int_\dSphe \dist_\dSphe(y,e_1)\d\mubar(y)
\le C_{j,A,d}\eps^{1/2}.
\]
Summing over \(j\) proves the claim.  Since
\(K_A(x,e_1)=\e^{\lambda_1x_1}\), the statement for \(U_\eps\) follows by
changing sign.
\end{proof}

\begin{proposition}[one-well scale of the weighted Poincare gap]
\label{prop:one-well-poincare-scale}
Assume \zcref{thm:package} and \(d\ge2\).  There exist
\(\bar\eps_{\mathrm{1w}}(A,d)>0\) and
\(c_{\mathrm{1w}}(A,d)>0\) such that, for every sufficiently small \(0<\eps<\bar\eps_{\mathrm{1w}}(A,d)\),
the weighted Poincare constant in \zcref{prop:poincare} satisfies
\begin{equation}\label{eq:one-well-poincare-scale}
\kappa_P(\eps)\ge \frac{c_{\mathrm{1w}}(A,d)}{\eps}.
\end{equation}
Moreover, one may choose the constant so that, with
\begin{equation}\label{eq:a-plus-minus-def}
a_+\coloneqq \lambda_1\e^{\lambda_1},
\qquad
a_-\coloneqq \lambda_1\e^{-\lambda_1},
\qquad
m\coloneqq d-1,
\end{equation}
there is a dimensional constant \(c_d>0\) for which
\begin{equation}\label{eq:c1w-explicit-rough}
c_{\mathrm{1w}}(A,d)
\ge
c_d\min\{a_+,ma_-\}
\ge
c_d\lambda_1\e^{-\lambda_1}.
\end{equation}
\end{proposition}

\begin{proof}
We use a standard Lyapunov--Poincare argument combined with the
Bakry--Emery local convexity criterion.  The estimates required for the present
one-well Gibbs measure are recorded explicitly below.

Set
\[
G_\eps\coloneqq U_\eps-\min_{\dSphe}U_\eps,
\qquad
L_\eps\coloneqq \eps\Laplacian-\nabla G_\eps\cdot\nabla.
\]
Then \(\mubar=Z_\eps^{-1}\exp(-G_\eps/\eps)\omega\), and
\(\Pbar=\eps^{-1}L_\eps\).  Thus \(\kappa_P(\eps)\) is the usual Poincare constant for the Gibbs measure
\(\exp(-G_\eps/\eps)\omega\), with Dirichlet form
\(\int_\dSphe |\nabla f|^2\d\mubar\).

For the limiting potential \(U_0=-\e^{\lambda_1x_1}\), the only critical points
on \(\dSphe\) are \(e_1\) and \(-e_1\).  The point \(e_1\) is a nondegenerate
minimum and \(-e_1\) is a nondegenerate maximum, with
\begin{equation}\label{eq:U0-hessians-wells}
\Hessian_{\dSphe}U_0(e_1)=a_+I_{T_{e_1}\dSphe},
\qquad
\Hessian_{\dSphe}U_0(-e_1)=-a_-I_{T_{-e_1}\dSphe}.
\end{equation}
By \zcref{lem:C3-one-well-potential} and Morse stability, after decreasing
\(\bar\eps_{\mathrm{1w}}\), the potential \(U_\eps\) has exactly two critical
points: a minimum \(x_{\eps,+}\) near \(e_1\) and a maximum \(x_{\eps,-}\) near
\(-e_1\).  Moreover,
\begin{equation}\label{eq:eps-well-hessian-bounds}
\Hessian_{\dSphe}G_\eps(x_{\eps,+})\ge \frac12a_+I,
\qquad
\Laplacian G_\eps(x_{\eps,-})\le -\frac12ma_-,
\end{equation}
and, outside fixed small neighborhoods of these two points,
\begin{equation}\label{eq:eps-gradient-away-critical}
|\nabla G_\eps|\ge g_A>0.
\end{equation}

Choose \(R=R(A,d)\) large and define
\[
\mathcal U_\eps\coloneqq B_{R\sqrt\eps}(x_{\eps,+}).
\]
For \(\eps\) small, this ball is geodesically convex and contained in a region
where \(\Hessian_{\dSphe}G_\eps\ge a_+I/4\).  The Bakry--Emery criterion applied to the
restricted Gibbs measure on \(\mathcal U_\eps\), with Neumann boundary condition,
gives the local Poincare estimate
\begin{equation}\label{eq:local-pi-Ueps}
\Var_{\mubar^{\mathcal U_\eps}}(f)
\le
\frac{4\eps}{a_+}
\int_{\mathcal U_\eps}|\nabla f|^2\d\mubar^{\mathcal U_\eps},
\end{equation}
where \(\mubar^{\mathcal U_\eps}\) is \(\mubar\) conditioned on
\(\mathcal U_\eps\).  Hence the restricted Poincare constant satisfies
\begin{equation}\label{eq:rho-U-lower}
\rho_{\mathcal U_\eps}\ge \frac{a_+}{4\eps}.
\end{equation}

Now take the Lyapunov function
\[
V_\eps\coloneqq \exp\left(\frac{G_\eps}{2\eps}\right)\ge1.
\]
A direct computation gives
\begin{equation}\label{eq:Lyap-computation-onewell}
\frac{L_\eps V_\eps}{\eps V_\eps}
=
\frac{\Laplacian G_\eps}{2\eps}
-
\frac{|\nabla G_\eps|^2}{4\eps^2}.
\end{equation}
In the annulus around \(x_{\eps,+}\) outside \(\mathcal U_\eps\), the quadratic
lower bound on \(|\nabla G_\eps|\) and the choice of \(R\) make the negative term
in \eqref{eq:Lyap-computation-onewell} dominate the Laplacian term.  Near
\(x_{\eps,-}\), the second estimate in \eqref{eq:eps-well-hessian-bounds} gives a
negative contribution of order \(ma_-/\eps\).  Away from the two critical
neighborhoods, \eqref{eq:eps-gradient-away-critical} gives a stronger negative
bound for small \(\eps\).  Consequently there exist constants
\(\Lambda_A,B_A>0\), independent of \(\eps\), such that
\begin{equation}\label{eq:Lyap-condition-onewell}
\frac{L_\eps V_\eps}{\eps V_\eps}
\le
-\frac{\Lambda_A}{\eps}
+
\frac{B_A}{\eps}\mathbf 1_{\mathcal U_\eps},
\end{equation}
with
\begin{equation}\label{eq:Lambda-B-rough}
\Lambda_A\ge c_d\min\{a_+,ma_-\},
\qquad
B_A\le C_d a_+.
\end{equation}

The Lyapunov--Poincare criterion now yields
\begin{equation}\label{eq:Lyap-PI-combine-onewell}
\kappa_P(\eps)
\ge
\frac{(\Lambda_A/\eps)\rho_{\mathcal U_\eps}}
{B_A/\eps+\rho_{\mathcal U_\eps}}.
\end{equation}
Using \eqref{eq:rho-U-lower} and \eqref{eq:Lambda-B-rough}, we obtain
\[
\kappa_P(\eps)
\ge
\frac{\Lambda_Aa_+}{4B_A+a_+}\frac1\eps
\ge
\frac{c_d}{\eps}\min\{a_+,ma_-\}.
\]
This proves \eqref{eq:one-well-poincare-scale}--\eqref{eq:c1w-explicit-rough}.
\end{proof}

\subsubsection{Local Taylor estimates in the density chart}\label{sec:app:taylor}

We use the mean-zero projection \(\Pi_0^\mubar f=f-\int_\dSphe f\d\mubar\) from
\eqref{eq:linear-objects}.  Fix $r_0\in(0,1/2)$ and consider measures of the form
\[
\mu=\rho\,\omega,
\qquad
\rho=(1+h)\rhobar,
\qquad
h\in\Xred,
\qquad
\Vab{h}_{L^\infty(\dSphe)}\le r_0.
\]
The bound on $\Vab{h}_{L^\infty}$ is only the definition of the local chart; it is not a model assumption. Set
\[
\Energygap(\mu)\coloneqq \Energy(\mu)-\Energy(\mubar).
\]
Also define the \emph{centered first-variation increment}
\begin{equation}\label{eq:Phi-mu}
\Phi_\mu
\coloneqq 
\Pi_0^\mubar\left(
\frac{\delta\Energy}{\delta\mu}(\mu)
-
\frac{\delta\Energy}{\delta\mu}(\mubar)
\right)
\in \Xred.
\end{equation}
Because $\fdv{\Energy}{\mu}(\mubar)$ is constant on $\dSphe$, this is the same as $\Pi_0^\mubar(\fdv{\Energy}{\mu}(\mu))$.

\begin{lemma}[free-energy Taylor expansion]\label{lem:energy-taylor}
For $\mu=(1+h)\mubar$ with $\Vab{h}_{L^\infty}\le r_0$,
\begin{equation}\label{eq:E-taylor}
\Energygap(\mu)=\frac12\aab{\Hbar h,{h}}_{L^2(\mubar)}+R_E(h),
\end{equation}
where
\begin{equation}\label{eq:RE-bound}
|R_E(h)|\le C_E\eps\Vab{h}_{L^\infty(\dSphe)}\Vab{h}_{L^2(\mubar)}^2.
\end{equation}
\end{lemma}

\begin{proof}
Using the Euler--Lagrange equation at $\mubar$ and the identity $\int h\d\mubar=0$, one finds
\[
\Energygap(\mu)
=
\eps\int_\dSphe  \bigl[(1+h)\log(1+h)-h\bigr] \d\mubar
-
\frac12\iint_{\dSphe\times\dSphe}K_A(x,y)h(x)h(y)\d\mubar(x)\d \mubar(y).
\]
Write
\[
(1+z)\log(1+z)-z=\frac12 z^2+r(z),
\qquad |r(z)|\le C|z|^3\quad (|z|\le r_0).
\]
The interaction term is exactly quadratic in $h$. Thus
\[
\Energygap(\mu)=\frac12\Bigl(\eps\int_\dSphe  h^2\d\mubar-\iint K_A(x,y)h(x)h(y)\d\mubar(x)\d \mubar(y)\Bigr)+\eps\int_\dSphe  r(h)\d\mubar,
\]
which is exactly \eqref{eq:E-taylor} with $R_E(h)=\eps\int_\dSphe  r(h)\d\mubar$. The bound \eqref{eq:RE-bound} follows from $|r(h)|\le C|h|^3$ and the prefactor \(\eps\).
\end{proof}

\begin{lemma}[centered first-variation Taylor expansion]\label{lem:firstvar-taylor}
For $\mu=(1+h)\mubar$ with $\Vab{h}_{L^\infty}\le r_0$,
\begin{equation}\label{eq:Phi-taylor}
\Phi_\mu=\Hbar h+R_F(h),
\end{equation}
where
\begin{equation}\label{eq:RF-bound}
\Vab{R_F(h)}_{L^2(\mubar)}\le C_F\eps\Vab{h}_{L^\infty(\dSphe)}\Vab{h}_{L^2(\mubar)}.
\end{equation}
Moreover, for every integer $s>\frac{d-1}{2}+1$ there exists $C_{F,s}$ such that
\begin{equation}\label{eq:RF-Hs}
\Vab{R_F(h)}_{H^s(\dSphe)}\le C_{F,s}\eps\Vab{h}_{H^s(\dSphe)}^2
\end{equation}
for $\Vab{h}_{L^\infty(\dSphe)}\le r_0$.
\end{lemma}

\begin{proof}
From the explicit formula for $\fdv{\Energy}{\mu}$,
\[
\fdv{\Energy}{\mu}(\mu)-\fdv{\Energy}{\mu}(\mubar)
=
\eps\log(1+h)-\int_\dSphe  K_A(\cdot,y)h(y)\d\mubar(y).
\]
After projecting onto mean-zero functions,
\[
\Phi_\mu
=
\eps h-\Pi_0^\mubar\Bigl[\int_\dSphe  K_A(\cdot,y)h(y)\d\mubar(y)\Bigr]
+
\eps\Pi_0^\mubar\bigl[\log(1+h)-h\bigr].
\]
The first two terms are exactly $\Hbar h$, hence
\[
R_F(h)=\eps\Pi_0^\mubar\bigl[\log(1+h)-h\bigr].
\]
Since $|\log(1+z)-z|\le C|z|^2$ for $|z|\le r_0$, \eqref{eq:RF-bound} follows from the prefactor \(\eps\). The Sobolev bound \eqref{eq:RF-Hs} is the standard composition estimate for $C^\infty$ functions on the compact sphere $\dSphe$.
\end{proof}

\begin{corollary}[local quadratic equivalence]\label{cor:quadratic-equivalence}
There exists \(r_1\in(0,r_0]\) such that, whenever \(\Vab{h}_{L^\infty}\le r_1\),
\begin{equation}\label{eq:E-equivalence}
c_1\Vab{h}_{L^2(\mubar)}^2\le \Energygap(\mu)\le c_2\Vab{h}_{L^2(\mubar)}^2,
\end{equation}
and
\begin{equation}\label{eq:Phi-lower}
\Vab{\Phi_\mu}_{L^2(\mubar)}\ge c_3\Vab{h}_{L^2(\mubar)}.
\end{equation}
After decreasing the small-noise threshold and the chart radius if necessary,
one may take
\begin{equation}\label{eq:c3-explicit-eps-scale}
c_3=\frac{\gamma_A}{4}\eps,
\end{equation}
and there is a constant \(C_{\textup{up}}=C_{\textup{up}}(A,d,r_1)\), independent
of \(\eps\) in the small-noise range, such that
\begin{equation}\label{eq:c2-upper-eps-scale}
c_2\le C_{\textup{up}}\eps.
\end{equation}
\end{corollary}

\begin{proof}
By \eqref{eq:H-coercive}, after decreasing the small-noise threshold if
necessary,
\[
\aab{\Hbar h,{h}}_{L^2(\mubar)}
\ge \frac{\gamma_A}{2}\eps\Vab{h}_{L^2(\mubar)}^2,
\qquad h\in\Xred.
\]
Since \(\Hbar\) is self-adjoint and nonnegative on \(\Xred\), the same spectral
lower bound gives
\[
\Vab{\Hbar h}_{L^2(\mubar)}
\ge \frac{\gamma_A}{2}\eps\Vab{h}_{L^2(\mubar)}.
\]
Apply \zcref{lem:energy-taylor} and \zcref{lem:firstvar-taylor}, and
choose \(r_1\) so that their remainder terms are absorbed by
\(\gamma_A\eps/8\).  This gives the lower energy bound and
\eqref{eq:c3-explicit-eps-scale}.  The upper bound in \eqref{eq:E-equivalence}
follows from the Taylor expansion, the positivity of \(\Kbar\), and
\(\Hbar=\eps I-\Kbar\), together with the remainder estimate
\eqref{eq:RE-bound}; after possibly decreasing \(r_1\), this yields
\eqref{eq:c2-upper-eps-scale}.
\end{proof}

\subsubsection{Local PL, strict dissipativity, and forced localization}\label{sec:app:local-pl}

Define the local dissipation by
\begin{equation}\label{eq:dissipation}
\Dissipation(\mu)\coloneqq \int_\dSphe  |\nabla \Phi_\mu|^2\d\mu.
\end{equation}

\begin{lemma}[dissipation controls the $L^2$ distance]\label{lem:D-lower-h}
There exists $r_2\in(0,r_1]$ and $c_D>0$ such that, whenever $\mu=(1+h)\mubar$ and $\Vab{h}_{L^\infty}\le r_2$,
\begin{equation}\label{eq:D-lower-h}
\Dissipation(\mu)\ge c_D\Vab{h}_{L^2(\mubar)}^2.
\end{equation}
\end{lemma}

\begin{proof}
If $\Vab{h}_{L^\infty}\le r_2<1$, then
\[
(1-r_2)\d\mubar\le \d \mu\le (1+r_2)\d\mubar.
\]
Hence
\[
\Dissipation(\mu)=\int_\dSphe  |\nabla \Phi_\mu|^2\d\mu
\ge
(1-r_2)\int_\dSphe  |\nabla \Phi_\mu|^2\d\mubar.
\]
By \zcref{prop:poincare},
\[
\int_\dSphe  |\nabla \Phi_\mu|^2\d\mubar
=-\aab{\Pbar \Phi_\mu,\Phi_\mu}_{L^2(\mubar)}
\ge \kappa_P(\eps)\Vab{\Phi_\mu}_{L^2(\mubar)}^2.
\]
Now use \eqref{eq:Phi-lower} from \zcref{cor:quadratic-equivalence}.
\end{proof}

\begin{theorem}[local PL / local entropy-dissipation]\label{thm:local-PL}
There exist $r_2\in(0,r_1]$ and $c_{\textup{PL}}>0$ such that, whenever $\mu=(1+h)\mubar$ and $\Vab{h}_{L^\infty}\le r_2$,
\begin{equation}\label{eq:local-PL}
\Dissipation(\mu)\ge c_{\textup{PL}}\,\Energygap(\mu).
\end{equation}
\end{theorem}

\begin{proof}
Combine \zcref{lem:D-lower-h} with the upper bound in \eqref{eq:E-equivalence}.
\end{proof}

\begin{proposition}[controlled strict dissipativity]\label{prop:strict-diss}
Let $(\mu_t,W_t)$ solve \eqref{eq:ODE-Net}, and assume $\mu_t$ remains in the neighborhood from \zcref{thm:local-PL}. Then
\begin{equation}\label{eq:energy-identity}
\frac{\d }{\d t}\Energygap(\mu_t)
=
-\Dissipation(\mu_t)
+
\int_\dSphe  \nabla\frac{\delta\Energy}{\delta\mu}(\mu_t)\cdot u_{W_t}\d\mu_t,
\end{equation}
and hence
\begin{equation}\label{eq:strict-diss}
\frac{\d }{\d t}\Energygap(\mu_t)
\le
-\frac12\Dissipation(\mu_t)+\frac12\Vab{\sigma}_{L^\infty(\dSphe)}^2\Vab{W_t}_{\frob}^2
\le
-a\,\Energygap(\mu_t)+C_\sigma\Vab{W_t}_{\frob}^2,
\end{equation}
where $a\coloneqq c_{\textup{PL}}/2$ and $C_\sigma\coloneqq \frac12\Vab{\sigma}_{L^\infty(\dSphe)}^2$.
\end{proposition}

\begin{proof}
Differentiate $\Energy(\mu_t)$ along \eqref{eq:gradient_flow_form}. Since $\partial_t\mu_t+\Div(\mu_t v_t)=0$ with
\[
v_t=-\nabla\frac{\delta\Energy}{\delta\mu}(\mu_t)+u_{W_t},
\]
the chain rule and integration by parts give
\[
\frac{\d }{\d t}\Energy(\mu_t)
=
\int_\dSphe  \nabla\frac{\delta\Energy}{\delta\mu}(\mu_t)\cdot v_t\d\mu_t,
\]
which is exactly \eqref{eq:energy-identity}. Apply Young's inequality to the mixed term:
\[
\int_\dSphe  \nabla\fdv{\Energy}{\mu}(\mu_t)\cdot u_{W_t}\d\mu_t
\le
\frac12\Dissipation(\mu_t)+\frac12\int_\dSphe  |u_{W_t}|^2\d\mu_t.
\]
Since $|u_W(x)|=|P_x^\perp W\sigma(x)|\le\Vab{W}_{\frob}\,\Vab{\sigma(x)}$, one has
\[
\int_\dSphe  |u_{W_t}|^2\d\mu_t\le \Vab{\sigma}_{L^\infty(\dSphe)}^2\Vab{W_t}_{\frob}^2.
\]
Insert \eqref{eq:local-PL}.
\end{proof}

\begin{corollary}[forced state localization, not yet turnpike]\label{cor:forced-localization}
Under the assumptions of \zcref{prop:strict-diss}, every $0\le s\le t\le T$ satisfies
\begin{equation}\label{eq:forced-localization}
\Energygap(\mu_t)\le \e^{-a(t-s)}\Energygap(\mu_s)+C_\sigma\int_s^t \e^{-a(t-\tau)}\Vab{W_\tau}_{\frob}^2\d\tau.
\end{equation}
\end{corollary}

\begin{proof}
Apply Gronwall to \eqref{eq:strict-diss}.
\end{proof}

\begin{remark}[explicit safe scale of the localization exponent]
\label{rem:explicit-localization-exponent-onewell}
Tracing the proof of \zcref{thm:local-PL} gives
\begin{equation}\label{eq:cPL-traced-onewell}
c_{\textup{PL}}\ge \frac{(1-r_2)\kappa_P(\eps)c_3^2}{c_2},
\qquad
 a=\frac12c_{\textup{PL}}.
\end{equation}
The parameter \(r_2\) is only the radius of the relative-density chart:
\begin{equation}\label{eq:r2-chart-meaning}
\mu=(1+h)\mubar,
\qquad
\int_\dSphe h\d\mubar=0,
\qquad
\Vab{h}_{L^\infty(\mubar)}\le r_2.
\end{equation}
It enters the estimate through the measure comparison
\begin{equation}\label{eq:r2-measure-comparison}
(1-r_2)\d\mubar\le \d \mu\le (1+r_2)\d\mubar.
\end{equation}
Thus \(1-r_2\) is a chart-safety factor rather than a spectral or dynamical
parameter.  By shrinking the local neighborhood one may make \(r_2\) smaller,
at the cost of requiring a more localized initial state.

With the choice \(c_3=\gamma_A\eps/4\) from
\eqref{eq:c3-explicit-eps-scale}, \zcref{prop:one-well-poincare-scale}
gives
\begin{equation}\label{eq:a-onewell-before-c2}
a\ge
\frac{(1-r_2)c_{\mathrm{1w}}(A,d)\gamma_A^2\eps}{32c_2}.
\end{equation}
Using the upper quadratic scale \eqref{eq:c2-upper-eps-scale}, we get
\begin{equation}\label{eq:a-onewell-explicit-scale}
a\ge
 a_{\textup{safe}}
\coloneqq 
\frac{(1-r_2)c_{\mathrm{1w}}(A,d)\gamma_A^2}{32C_{\textup{up}}}
\ge
\frac{(1-r_2)c_d}{32C_{\textup{up}}}
\left(1-\frac{\lambda_2}{\lambda_1}\right)^2
\lambda_1\e^{-\lambda_1}.
\end{equation}
This is a lower bound on the uniform global state-localization exponent in the
relative-density chart.  It should not be read as a sharp equivalence for the
bulk relaxation rate near the bottom of the selected well.
\end{remark}

\begin{remark}[uniform localization time from the energy gap]
\label{rem:uniform-localization-time}
Assume first that the control is absent, \(W_t\equiv0\), and that the trajectory
remains in the local chart of \zcref{thm:local-PL}.  Set
\begin{equation}\label{eq:Delta-init-local}
\Delta\coloneqq \Energygap(\mu_0)
=\Energy(\mu_0)-\Energy(\mubar).
\end{equation}
Then \zcref{cor:forced-localization} gives
\begin{equation}\label{eq:unforced-localization-gap}
\Energygap(\mu_t)\le \e^{-at}\Delta.
\end{equation}
Consequently, for any target level \(0<\eta<\Delta\), the sufficient uniform
localization time
\begin{equation}\label{eq:Tuniform-def}
T_{\textup{uniform}}(\eta)
\coloneqq 
\inf\{t\ge0:\ \Energygap(\mu_s)\le\eta\ \text{for all }s\ge t\}
\end{equation}
satisfies
\begin{equation}\label{eq:Tuniform-a}
T_{\textup{uniform}}(\eta)
\le
\frac1a\log\frac{\Delta}{\eta}
\le
\frac1{a_{\textup{safe}}}\log\frac{\Delta}{\eta}.
\end{equation}
Using \eqref{eq:a-onewell-explicit-scale}, this yields the rough scale
\begin{equation}\label{eq:Tuniform-rough-scale}
T_{\textup{uniform}}(\eta)
\lesssim
\frac{\e^{\lambda_1}}
{\left(1-\frac{\lambda_2}{\lambda_1}\right)^2\lambda_1}
\log\frac{\Delta}{\eta},
\end{equation}
where the suppressed constant depends on the chart constants and on the
one-well Poincare comparison constants.  This is a worst-case global
relative-density estimate.  It may overestimate the effective relaxation time
for perturbations supported in the bulk of the selected well.
\end{remark}

At this stage the local one-well theory above is a toolkit around one fixed minimizer branch.  The global problem is different: the free energy has an antipodal symmetry, and a branch-specific Wasserstein statement is not stable under this symmetry.  The robust target is therefore the scalar free-energy gap.

Recall from \zcref{prop:energy-landscape} that, for the fixed value of \(\eps\),
\begin{equation}\label{eq:minimizer-set}
\Minimizers
=
\argmin_{\mu\in\PdSphe}\Energy(\mu)
=
\{\mubar^+,\mubar^-\}.
\end{equation}
We use the global energy gap
\begin{equation}\label{eq:global-energy-gap-def}
\Energygap(\mu)
\coloneqq
\Energy(\mu)-\min_{\nu\in\PdSphe}\Energy(\nu).
\end{equation}
Since the selected branch \(\mubar\) belongs to \(\Minimizers\), this definition of \(\Energygap\) agrees with the local one used above.  Globally, \(\Energygap\) measures distance to the minimizer set in energy and not distance to a chosen element of that set.

\subsubsection{Static double-well landscape}\label{sec:app:two-well}

\zcref{thm:package} now supplies the finite-\(\eps\) double-well minimizer
statement.  The zero-temperature equality case gives only separation into the
two caps; the branch count itself is obtained in the proof of
\zcref{thm:package} from the Euler--Lagrange difference identity and the
chord estimate \eqref{eq:chord-K-op-main}.  The next proposition records the static double-well structure.  The local
basins below are density-chart basins around the two smooth minimizers.

\begin{proposition}[static double-well local package]\label{prop:two-well-static}
For every sufficiently small fixed \(\eps>0\), the following hold.
\begin{enumerate}[label=(\roman*),leftmargin=7mm]
\item The global minimizer set consists of exactly two antipodal branches:
\begin{equation}\label{eq:two-minimizers}
\Minimizers=
\{\mubar^+,\mubar^-\},
\qquad
\mubar^-=S_\#\mubar^+,
\qquad
\mubar^+=\mubar.
\end{equation}

\item There exist disjoint local density-chart basins \(\mathcal U_+\) and
\(\mathcal U_-\) of \(\mubar^+\) and \(\mubar^-\), and constants
\(a_{\textup{loc}}>0\), \(b_{\textup{loc}}>0\), such that the local
PL/strict-dissipativity estimate above holds on both basins:
\begin{equation}\label{eq:local-dissipativity-two-well}
\frac{d}{dt}\Energygap(\mu_t)
\le
-a_{\textup{loc}}\,\Energygap(\mu_t)
+b_{\textup{loc}}\Vab{W_t}_{\frob}^2
\end{equation}
whenever the controlled solution remains in one of the two basins.  By
symmetry one may take the constants from \zcref{prop:strict-diss} on
the selected branch, namely
\(a_{\textup{loc}}=a\) and \(b_{\textup{loc}}=C_\sigma\).

\item On each basin there is a branchwise metric-energy comparison
\begin{equation}\label{eq:branchwise-W2-from-E-two}
W_2(\mu,\mubar^\pm)^2
\le
C_{W_2,\Energygap}\,\Energygap(\mu)
\qquad \forall \mu\in\mathcal U_\pm.
\end{equation}

\item There exists \(\eta_{\textup{loc}}>0\) such that, for every
\(0<\eta<\eta_{\textup{loc}}\), the branchwise chart sublevels
\[
\{\Energygap \le \eta\}\cap\mathcal U_+,
\qquad
\{\Energygap \le \eta\}\cap\mathcal U_-
\]
are connected.
\end{enumerate}
\end{proposition}

\begin{proof}
Item (i) is exactly \zcref{thm:package}(iv).  Choose
\(r_*>0\) so small that
\[
 r_*\le r_2,
 \qquad
 \frac{1}{1+r_*}>\frac{\lambda_2}{\lambda_1},
\]
where \(r_2\) is the radius in \zcref{thm:local-PL}, and then decrease
\(r_*\) if necessary so that the two relative-density charts below are
disjoint.  Define
\begin{equation}\label{eq:local-chart-basins-def-app}
\mathcal U_+
\coloneqq
\left\{(1+h)\mubar^+:
 h\in L^\infty(\mubar^+),
 \int_\dSphe h\d \mubar^+=0,
 \Vab{h}_{L^\infty}<r_*\right\},
\end{equation}
and set \(\mathcal U_-\coloneqq S_\#\mathcal U_+\).  By \zcref{thm:local-PL} and
\zcref{prop:strict-diss}, the dissipativity estimate
\eqref{eq:local-dissipativity-two-well} holds on \(\mathcal U_+\).  The
antipodal symmetry transfers the same estimate, with the same constants, to
\(\mathcal U_-\).

For the metric-energy comparison on \(\mathcal U_+\), write
\(\mu=(1+h)\mubar^+\) with \(\Vab{h}_{L^\infty}<r_*\), and connect
\(\mubar^+\) to \(\mu\) by the linear density path
\(\mu_s=(1+sh)\mubar^+\).  Since \(\mu_s\) is uniformly equivalent to the
smooth positive measure \(\mubar^+\), the elliptic problem
\[
-\Div(\mu_s\nabla\phi_s)=h\mubar^+,
\qquad
\int_\dSphe\phi_s\d \mu_s=0,
\]
has the standard estimate
\[
\int_\dSphe|\nabla\phi_s|^2\d \mu_s
\le C\Vab{h}_{L^2(\mubar^+)}^2,
\]
with \(C\) uniform for \(0\le s\le1\).  The Benamou--Brenier formula along this
path gives
\[
W_2(\mu,\mubar^+)^2
\le C\Vab{h}_{L^2(\mubar^+)}^2.
\]
Combining this with the local quadratic lower bound
\eqref{eq:E-equivalence} gives \eqref{eq:branchwise-W2-from-E-two}.  The
negative branch is again identical by symmetry.

It remains to prove the connectedness of small branchwise chart sublevels.  Let
\(\mu_i=(1+h_i)\mubar^+\in\mathcal U_+\), \(i=0,1\), and set
\(h_t=(1-t)h_0+th_1\), \(g=h_1-h_0\).  The chart is convex, so
\((1+h_t)\mubar^+\in\mathcal U_+\).  For
\(f(t)\coloneqq\Energy((1+h_t)\mubar^+)\), the exact second derivative is
\[
f''(t)
=
\eps\int_\dSphe\frac{g^2}{1+h_t}\d \mubar^+
-
\iint K_A(x,y)g(x)g(y)\d \mubar^+(x)\d \mubar^+(y).
\]
Since \(\int g\d \mubar^+=0\), \zcref{thm:K-op} yields
\[
f''(t)
\ge
\left(\frac{\eps}{1+r_*}-\Vab{\Kbar}_{\op}\right)
\Vab{g}_{L^2(\mubar^+)}^2
\ge c\eps\Vab{g}_{L^2(\mubar^+)}^2
\]
for small \(\eps\).  Hence \(\Energy\) is convex along every density
chord in \(\mathcal U_+\), and every sublevel intersection
\(\{\Energygap \le\eta\}\cap\mathcal U_+\) is connected.  The proof in
\(\mathcal U_-\) is the antipodal image of this one.  Taking
\(\eta_{\textup{loc}}>0\) arbitrary, or smaller if desired for later uses,
finishes the proof.
\end{proof}

\begin{proposition}[weak double-well barrier]\label{prop:weak-two-well-barrier}
Fix a sufficiently small \(\eps>0\).  Let \(V_+^{\rm stat}\) and
\(V_-^{\rm stat}\) denote the two static wells from
\zcref{thm:package}, and choose disjoint narrow open neighborhoods
\(\mathcal V_+\), \(\mathcal V_-\) of \(\mubar^+\), \(\mubar^-\) such that
\[
S_\#\mathcal V_+=\mathcal V_-,
\qquad
\mathcal U_\pm\subset\mathcal V_\pm,
\qquad
\overline{\mathcal V_+}\subset V_+^{\rm stat},
\qquad
\overline{\mathcal V_-}\subset V_-^{\rm stat}.
\]
This choice is possible after decreasing the chart radius in
\zcref{prop:two-well-static}, because the relative-density charts
then lie in an arbitrarily small total-variation, hence \(W_2\), neighborhood
of the corresponding minimizer.
Then
\begin{equation}\label{eq:weak-two-well-barrier}
\beta_{\textup w}
\coloneqq
\inf_{\PdSphe\setminus(\mathcal V_+\cup\mathcal V_-)}\Energygap 
>
0.
\end{equation}
\end{proposition}

\begin{proof}
The functional \(\Energy\) is narrowly lower semicontinuous: the entropy
is narrowly lower semicontinuous and the interaction term is narrowly
continuous because \(K_A\) is continuous and bounded.  Since \(\dSphe\) is compact,
\(\PdSphe\) is narrowly compact.  If \(\beta_{\textup w}=0\), there is a
sequence \(\mu_n\notin\mathcal V_+\cup\mathcal V_-\) such that
\(\Energygap(\mu_n)\to0\).  Passing to a subsequence,
\(\mu_n\rightharpoonup\mu_*\).  Lower semicontinuity gives \(\Energygap(\mu_*)=0\),
so \zcref{prop:two-well-static}(i) gives
\(\mu_*\in\{\mubar^+,\mubar^-\}\).  Since \(\mathcal V_\pm\) are open, this
contradicts \(\mu_n\notin\mathcal V_+\cup\mathcal V_-\) for all large \(n\).
\end{proof}

\begin{lemma}[free-energy gradient-flow deformation inside one weak well]
\label{lem:gf-deformation-one-well}
Let \(0<\eta<\min\{\beta_{\textup w},\eta_{\textup{crit}}\}\), where
\(\eta_{\textup{crit}}\) is the one-well critical-point uniqueness threshold in
\zcref{thm:package}(iv).  Define
\[
\mathcal C_\pm(\eta)
\coloneqq
\{\mu\in\PdSphe:\Energygap(\mu)\le\eta\}\cap\mathcal V_\pm.
\]
Then each \(\mathcal C_\pm(\eta)\) is path connected.  More precisely, every
\(\mu\in\mathcal C_\pm(\eta)\) can be joined to \(\mubar^\pm\) by a continuous
path contained in \(\mathcal C_\pm(\eta)\).
\end{lemma}

\begin{proof}
We prove the statement in the positive well.  Let \(G_t\mu\) be the
uncontrolled Wasserstein gradient flow of \(\Energy\), equivalently the
smooth solution for \(t>0\) of
\begin{equation}\label{eq:uncontrolled-gf-static-topology}
\partial_t\nu_t+\Div\bigl(\nu_t P_x^\perp\chi_A[\nu_t]\bigr)=\eps\Laplacian\nu_t,
\qquad
\nu_{t=0}=\mu.
\end{equation}
For finite-energy initial data this uniformly parabolic equation is well posed
for positive times, is continuous in the narrow topology down to \(t=0\), and
satisfies the energy identity
\begin{equation}\label{eq:uncontrolled-energy-identity}
\frac{d}{dt}\Energygap(G_t\mu)=-\Dissipation (G_t\mu)
\end{equation}
for \(t>0\).  These are the standard parabolic-gradient-flow facts for entropy
plus a smooth interaction on a compact sphere; see, for example,
\parencite{BOGACHEV20193681} for related Fokker--Planck--Kolmogorov stability
and convergence arguments.

Since energy is nonincreasing, \(G_t\mu\in\{\Energygap \le\eta\}\) for all \(t\ge0\).
Because \(\eta<\beta_{\textup w}\), \zcref{prop:weak-two-well-barrier}
gives
\[
G_t\mu\in\mathcal V_+\cup\mathcal V_-
\qquad\forall t\ge0.
\]
The path \(t\mapsto G_t\mu\) is continuous and starts in \(\mathcal V_+\); the
sets \(\mathcal V_+\) and \(\mathcal V_-\) are disjoint open sets whose union
contains the path.  Hence connectedness of \([0,T]\) implies
\[
G_t\mu\in\mathcal V_+
\qquad\forall t\ge0.
\]

It remains to identify the limit as \(t\to\infty\).  From
\eqref{eq:uncontrolled-energy-identity},
\(\int_0^\infty\Dissipation (G_t\mu)\,dt<\infty\).  Parabolic regularity gives
relative compactness of \(\{G_t\mu:t\ge t_0\}\) in smooth topologies for every
\(t_0>0\).  Thus every \(\omega\)-limit point is a critical point of
\(\Energy\), has energy at most \(\eta\), and belongs to
\(\overline{\mathcal V_+}\subset V_+^{\rm stat}\).  Since
\(\eta<\eta_{\textup{crit}}\), \zcref{thm:package}(iv) implies that the
only such critical point is \(\mubar^+\).  Therefore
\[
G_t\mu\rightharpoonup\mubar^+
\qquad\text{as }t\to\infty.
\]
On compact \(\dSphe\), narrow convergence is equivalent to \(W_2\)-convergence.
Thus
\[
\gamma(s)\coloneqq
\begin{cases}
G_{s/(1-s)}\mu, & 0\le s<1,\\
\mubar^+, & s=1,
\end{cases}
\]
is a continuous path in \(\mathcal C_+(\eta)\) joining \(\mu\) to \(\mubar^+\).
The negative well follows by applying the antipodal map.
\end{proof}

\begin{proposition}[low-energy sublevel has exactly two weak components]
\label{prop:two-low-components}
For every sufficiently small fixed \(\eps>0\), there exists
\(\eta_{\textup{sep}}>0\) such that, for every
\(0<\eta<\eta_{\textup{sep}}\),
\begin{equation}\label{eq:two-components-decomposition}
\{\mu:\Energygap(\mu)\le \eta\}
=
\mathcal C_+(\eta)\,\dot\cup\,\mathcal C_-(\eta),
\end{equation}
where
\begin{equation}\label{eq:components-def}
\mathcal C_\pm(\eta)
\coloneqq
\{\mu:\Energygap(\mu)\le\eta\}\cap\mathcal V_\pm.
\end{equation}
Moreover \(\mathcal C_+(\eta)\) and \(\mathcal C_-(\eta)\) are precisely the
two connected components of the low-energy sublevel.
\end{proposition}

\begin{proof}
Let \(\eta_{\textup{crit}}=\eta_{\textup{crit}}(\eps)\) be the critical-point
uniqueness threshold from \zcref{thm:package}(iv).  Set
\[
\eta_{\textup{sep}}
<
\min\{\beta_{\textup w},\eta_{\textup{crit}}\}.
\]
Then any \(\mu\) with \(\Energygap(\mu)\le\eta<\eta_{\textup{sep}}\) lies in
\(\mathcal V_+\cup\mathcal V_-\) by
\zcref{prop:weak-two-well-barrier}.  This proves the disjoint
decomposition \eqref{eq:two-components-decomposition}.  The two sets are
nonempty because they contain \(\mubar^+\) and \(\mubar^-\).  They are path
connected by \zcref{lem:gf-deformation-one-well}.  Since they are separated
inside the sublevel by the disjoint open sets \(\mathcal V_+\), \(\mathcal V_-\),
they are exactly the two connected components.
\end{proof}

\begin{lemma}[well-distance modulus on weak low-energy components]
\label{lem:well-distance-modulus}
For \(0<\eta<\eta_{\textup{sep}}\), define
\begin{equation}\label{eq:well-distance-modulus-def}
\Theta_{\textup{well}}(\eta)
\coloneqq
\max_{\varsigma\in\{+,-\}}
\sup_{\mu\in\mathcal C_\varsigma(\eta)}
W_2(\mu,\mubar^\varsigma).
\end{equation}
Then
\begin{equation}\label{eq:well-distance-modulus-vanish}
\Theta_{\textup{well}}(\eta)\to0
\qquad\text{as }\eta\downarrow0.
\end{equation}
\end{lemma}

\begin{proof}
If not, then for some \(\delta>0\) there are \(\eta_n\downarrow0\), signs
\(\varsigma_n\), and \(\mu_n\in\mathcal C_{\varsigma_n}(\eta_n)\) such that
\(W_2(\mu_n,\mubar^{\varsigma_n})\ge\delta\).  Passing to a subsequence fixes
the sign and gives \(\mu_n\rightharpoonup\mu_*\).  Lower semicontinuity implies
\(\Energygap(\mu_*)=0\), hence \(\mu_*=\mubar^{\varsigma}\).  Since \(\dSphe\) is compact,
narrow convergence implies \(W_2\)-convergence, contradicting the lower bound
by \(\delta\).
\end{proof}

\subsubsection{Global energy budget and set-valued entry}\label{sec:app:energy-budget}

The part of the local dissipativity argument that does not use the local PL inequality is global.  Define the Wasserstein slope dissipation
\begin{equation}\label{eq:global-dissipation}
\Dissipation (\mu)
\coloneqq
\int_\dSphe
\left|\nabla\frac{\delta\Energy}{\delta\mu}(\mu)\right|^2\d \mu.
\end{equation}
For every smooth controlled solution of \eqref{eq:ODE-Net}, the chain rule and Young inequality give
\begin{equation}\label{eq:global-EDI-forced}
\frac{d}{dt}\Energygap(\mu_t)
=
-\Dissipation (\mu_t)
+
\int_\dSphe
\nabla\frac{\delta\Energy}{\delta\mu}(\mu_t)\cdot u_{W_t}\d \mu_t
\le
-\frac12\Dissipation (\mu_t)+C_\sigma\Vab{W_t}_{\frob}^2,
\end{equation}
where
\[
C_\sigma=\frac12\Vab{\sigma}_{L^\infty(\dSphe)}^2.
\]

Since \(\osc\ell<\infty\) in the standing problem formulation, comparing the optimal control with the zero control gives the $T$-uniform bound
\begin{equation}\label{eq:global-control-budget}
\frac{\regparam}{2}\int_0^T\Vab{W_t^{*,T}}_{\frob}^2\d t
\le
\osc\ell.
\end{equation}
For later use set
\begin{equation}\label{eq:Ropt-def}
R_{\textup{opt}}
\coloneqq
\frac{2}{\regparam}\osc\ell,
\end{equation}
so that every subinterval of an optimal trajectory has control energy at most
\(R_{\textup{opt}}\).  Combining \eqref{eq:global-EDI-forced} and
\eqref{eq:global-control-budget} yields
\begin{equation}\label{eq:Emax-def}
\sup_{t\in[0,T]}\Energygap(\mu_t^{*,T})
\le
\mathfrak E_{\max}
\coloneqq
\Energygap(\mu_0)+\frac{2C_\sigma}{\regparam}\osc\ell,
\end{equation}
\begin{equation}\label{eq:Dmax-def}
\int_0^T\Dissipation (\mu_t^{*,T})\d t
\le
\mathfrak D_{\max}
\coloneqq
2\Energygap(\mu_0)+\frac{4C_\sigma}{\regparam}\osc\ell.
\end{equation}
These constants are independent of the horizon.

For $0<\eta<\eta_{\textup{sep}}$, with the components \(\mathcal C_\pm(\eta)\) of \zcref{prop:two-low-components}, define the low-energy neighborhood of the minimizer set
\begin{equation}\label{eq:Neta-def}
\mathcal N_\eta
\coloneqq
\mathcal C_+(\eta)\cup\mathcal C_-(\eta).
\end{equation}
The correct global slope gap is a gap outside this set:
\begin{equation}\label{eq:set-gap-def}
d_{\textup{set}}(\eta)
\coloneqq
\inf\Bigl\{
\Dissipation (\mu):\;
\Energygap(\mu)\le\mathfrak E_{\max},
\mu\notin\mathcal N_\eta
\Bigr\}.
\end{equation}

\paragraph{Consequences of \zcref{ass:nolargegap}.}
Set
\begin{equation}\label{eq:QA-Delta0-main}
Q_A\coloneqq\max\{\e^{\lambda_2},\cosh\lambda_1\},
\qquad
\Delta_0\coloneqq\frac12(\e^{\lambda_1}-Q_A)>0.
\end{equation}
We use the first-coordinate order parameter and the finite-temperature sign-changing barrier
\begin{equation}\label{eq:eta-w-def-main}
m_1(\mu)\coloneqq \int_\dSphe x_1\d \mu(x),
\qquad
\eta_{\textup w,\eps}\coloneqq\inf\{\Energygap(\mu):m_1(\mu)=0\}.
\end{equation}
By \zcref{ass:nolargegap} and \eqref{eq:Emax-def},
\begin{equation}\label{eq:explicit-barrier-margin-main}
\gamma_{\textup{bar}}
\coloneqq
\Delta_0-\mathfrak E_{\max}>0.
\end{equation}
The second inequality in \zcref{ass:nolargegap} is
\begin{equation}\label{eq:terminal-oscillation-switch-small}
\osc\ell
<
\frac{\regparam}{8C_\sigma}\Delta_0.
\end{equation}

\begin{remark}[how the slope gap and no-switching are obtained]\label{rem:set-gap-condition}
The entry argument below uses a threshold \(\eta\in(0,\eta_{\textup{sep}})\) for which
\begin{equation}\label{eq:set-gap-positive}
d_{\textup{set}}(\eta)>0.
\end{equation}
This follows from \zcref{ass:nolargegap}.  \zcref{prop:explicit-small-temperature-excited-critical-gap,prop:slope-gap-from-critical-exclusion,cor:slope-gap-from-explicit-barrier} show that its first inequality excludes all non-minimal Gibbs critical points with \(0<\Energygap \le\mathfrak E_{\max}\) for all sufficiently small fixed \(\eps\), and a fixed-temperature compactness argument then gives \eqref{eq:set-gap-positive} for every \(0<\eta<\eta_{\textup{sep}}\).  Its second inequality, together with the sign-changing barrier in \zcref{cor:finite-temp-sign-changing-barrier}, excludes switching between the two branches.
\end{remark}

\subsubsection{Sign barrier and slope-gap inputs}\label{sec:app:sign-barrier}
\begin{lemma}[zero-temperature sign-changing barrier]
\label{lem:zero-temp-sign-changing-barrier}
Let
\[
\mathcal P_0\coloneqq \{\mu\in\mathcal P(\dSphe):m_1(\mu)=0\}.
\]
Then
\begin{equation}\label{eq:zero-temp-sign-changing-barrier}
\sup_{\mu\in\mathcal P_0} I_A(\mu)
=
Q_A.
\end{equation}
\end{lemma}

\begin{proof}
The lower bound follows from the two admissible measures \(\delta_{e_2}\) and
\(\frac12(\delta_{e_1}+\delta_{-e_1})\), whose interaction energies are
\(\e^{\lambda_2}\) and \(\cosh\lambda_1\), respectively.  For the upper bound,
use that \(K_A\) is positive definite, hence \(I_A\) is convex in \(\mu\).  The
set \(\mathcal P_0\) is convex and compact.  The maximum of a continuous convex
functional over this moment set is attained at an extreme point.  Since
\(\mathcal P_0\) is described by the two affine moment constraints
\(\int1\d\mu=1\) and \(\int x_1\d\mu=0\), its extreme points are supported
on at most two points.  Hence it suffices to consider
\(\mu=t\delta_x+(1-t)\delta_y\) with \(t x_1+(1-t)y_1=0\).  If this measure is a
single Dirac mass, then its support point satisfies \(x_1=0\), and
\(I_A(\delta_x)\le\e^{\lambda_2}\).  Otherwise \(x_1y_1\le0\), and the value of
\(t\) is the top-coordinate-cancelling weight used in
\zcref{lem:two-point-bound-across-equator}.  Therefore
\[
I_A(t\delta_x+(1-t)\delta_y)\le Q_A.
\]
This proves the upper bound.
\end{proof}

\begin{corollary}[finite-temperature sign-changing barrier]
\label{cor:finite-temp-sign-changing-barrier}
Let \(\eta_{\textup w,\eps}\) be defined by \eqref{eq:eta-w-def-main}.  Then
there exists \(C_{\textup w}<\infty\), depending only on the fixed model data,
such that, for all sufficiently small \(\eps>0\),
\begin{equation}\label{eq:finite-temp-sign-changing-barrier}
\eta_{\textup w,\eps}
\ge
\Delta_0-C_{\textup w}\eps|\log\eps|.
\end{equation}
In particular, after decreasing \(\eps_0\),
\begin{equation}\label{eq:finite-sign-barrier-main-use}
\eta_{\textup w,\eps}\ge\frac34\Delta_0.
\end{equation}
\end{corollary}

\begin{proof}
Let \(\mu\) satisfy \(m_1(\mu)=0\).  By
\zcref{lem:zero-temp-sign-changing-barrier}, \(I_A(\mu)\le Q_A\).  Since
\(\eps\KL(\mu\mid\omega)\ge0\),
\[
\Energy(\mu)\ge -\frac12Q_A.
\]
On the other hand, the test-cap construction in the proof of
\zcref{thm:package} gives
\[
m_\eps\le -\frac12\e^{\lambda_1}+C_{\textup w}\eps|\log\eps|.
\]
Therefore
\[
\Energygap(\mu)=\Energy(\mu)-m_\eps
\ge
\frac12(\e^{\lambda_1}-Q_A)-C_{\textup w}\eps|\log\eps|
=
\Delta_0-C_{\textup w}\eps|\log\eps|.
\]
Taking the infimum over \(m_1(\mu)=0\) proves the first claim.  The second claim
follows by decreasing \(\eps_0\).
\end{proof}

\begin{proposition}[fixed-temperature slope gap from critical-point exclusion]
\label{prop:slope-gap-from-critical-exclusion}
Fix \(0<\eps\) and \(0<\eta<\eta_{\textup{sep}}\).  Suppose that
\begin{equation}\label{eq:no-budget-critical-points}
\Crit\Energy\cap
\{\mu:0<\Energygap(\mu)\le\mathfrak E_{\max}\}
\subset \Minimizers.
\end{equation}
Equivalently, there is no non-minimal stationary Gibbs state with
\(0<\Energygap\le\mathfrak E_{\max}\).  Then
\begin{equation}\label{eq:slope-gap-from-critical-exclusion}
d_{\textup{set}}(\eta)>0.
\end{equation}
\end{proposition}

\begin{proof}
Suppose not.  Then there are \(\mu_n=\rho_n\omega\) such that
\[
\Energygap(\mu_n)\le\mathfrak E_{\max},
\qquad
\mu_n\notin\mathcal N_\eta,
\qquad
\Dissipation (\mu_n)\to0.
\]
Since \(0<\eta<\eta_{\textup{sep}}\), \zcref{prop:two-low-components}
implies that \(\mu_n\notin\mathcal N_\eta\) forces \(\Energygap(\mu_n)>\eta\).  Passing
to a subsequence, \(\Energygap(\mu_n)\to L\in[\eta,\mathfrak E_{\max}]\).

Set \(b_n\coloneqq \nabla(K_A*\mu_n)\).  Smoothness of \(K_A\) on the compact set \(\dSphe\times\dSphe\)
gives \(\Vab{b_n}_{L^\infty}\le C_A\).  From the definition of \(\Dissipation \),
\[
\eps^2\int_\dSphe  \frac{|\nabla\rho_n|^2}{\rho_n}\d\omega
=
\int_\dSphe  |\eps\nabla\log\rho_n|^2\d\mu_n
\le 2\Dissipation (\mu_n)+2C_A^2.
\]
Thus \(\sqrt{\rho_n}\) is bounded in \(H^1(\dSphe)\).  After passing to a subsequence,
\(\sqrt{\rho_n}\to g\) strongly in \(L^2(\dSphe)\), and hence
\(\rho_n\to\rho_*\coloneqq g^2\) strongly in \(L^1(\dSphe)\).  Put
\(\mu_*\coloneqq \rho_*\omega\).  The entropy is continuous along this sequence by
uniform integrability, and the interaction is continuous under narrow
convergence; hence \(\Energygap(\mu_*)=L\in[\eta,\mathfrak E_{\max}]\).

For every \(\phi\in C^\infty(\dSphe)\),
\[
\left|
\int_\dSphe  \nabla\phi\cdot(\eps\nabla\log\rho_n-b_n)\d\mu_n
\right|
\le
\Vab{\nabla\phi}_{L^\infty}\Dissipation (\mu_n)^{1/2}\to0.
\]
Passing to the limit, using \(b_n\to b_*\coloneqq \nabla(K_A*\mu_*)\) uniformly, gives
\[
-\eps\int_\dSphe  \Delta\phi\d\mu_*
-
\int_\dSphe  \nabla\phi\cdot b_*\d\mu_*=0.
\]
Thus
\[
\Div(\eps\nabla\rho_* - \rho_*\nabla(K_A*\mu_*))=0
\]
in distributions.  Elliptic regularity and the strong maximum principle imply
\(\rho_*\in C^\infty(\dSphe)\) and \(\rho_*>0\).  Therefore
\[
\eps(\log\rho_*+1)-K_A*\mu_*
\]
is constant on \(\dSphe\), i.e. \(\mu_*\in\Crit\Energy\).  Since
\(\Energygap(\mu_*)\in[\eta,\mathfrak E_{\max}]\), this contradicts
\eqref{eq:no-budget-critical-points}.  Hence \(d_{\textup{set}}(\eta)>0\).
\end{proof}

\begin{corollary}[slope gap from the explicit barrier]
\label{cor:slope-gap-from-explicit-barrier}
Assume that there exists \(\gamma_{\textup{bar}}>0\) such that
\[
\mathfrak E_{\max}\le\Delta_0-\gamma_{\textup{bar}}.
\]
Then, after decreasing \(\eps_0\) depending on \(\gamma_{\textup{bar}}\), for every
\(0<\eps<\eps_0\) and every \(0<\eta<\eta_{\textup{sep}}\),
\[
d_{\textup{set}}(\eta)>0.
\]
\end{corollary}

\begin{proof}
Apply \zcref{prop:explicit-small-temperature-excited-critical-gap} with
\(\gamma=\gamma_{\textup{bar}}/2\).  For all sufficiently small \(\eps\), every
non-minimal Gibbs critical point has excess energy at least
\(\Delta_0-\gamma_{\textup{bar}}/2>\mathfrak E_{\max}\).  Hence the critical-point
exclusion hypothesis in \zcref{prop:slope-gap-from-critical-exclusion}
holds, and the slope gap follows.
\end{proof}

\subsubsection{Parabolic chart entry and exclusion of branch switching}\label{sec:app:chart-entry}

\begin{lemma}[parabolic weak-to-chart upgrade for fixed \(\eps\)]
\label{lem:parabolic-chart-entry}
Fix the local density-chart basins \(\mathcal U_\pm\) from
\zcref{prop:two-well-static}.  Let
\(\tau_{\textup{reg}}>0\) and \(R_{\textup{reg}}<\infty\) be fixed.  Then there
exists
\[
0<\eta_{\textup{chart}}=\eta_{\textup{chart}}(\eps,\tau_{\textup{reg}},
R_{\textup{reg}},\mathcal U_\pm)<\eta_{\textup{sep}},
\]
possibly very small as \(\eps\downarrow0\), such that the following holds.  Let
\(\mu_t=\rho_t\omega\) be a controlled solution of \eqref{eq:ODE-Net} on
\([t-\tau_{\textup{reg}},t]\).  If, for some \(\varsigma\in\{+,-\}\),
\begin{equation}\label{eq:parabolic-chart-entry-hyp}
\int_{t-\tau_{\textup{reg}}}^t\Vab{W_s}_{\frob}^2\d s
\le R_{\textup{reg}},
\qquad
\mu_t\in\mathcal C_\varsigma(\eta_{\textup{chart}}),
\end{equation}
then
\begin{equation}\label{eq:parabolic-chart-entry-concl}
\mu_t\in\mathcal U_\varsigma.
\end{equation}
\end{lemma}

\subsubsection{Parabolic regularization and same-branch entry}\label{sec:app:regularization}

This section proves the regularization input used in
\zcref{lem:parabolic-chart-entry}.  The heat-kernel input is recalled in
\zcref{rem:grigoryan-heat-kernel-input}, following Grigor'yan
\parencite{Grigoryan_1999}.  For related nonlinear Fokker--Planck--Kolmogorov
stability and convergence results, see \parencite{BOGACHEV20193681}.  The drift
is treated by an explicit parametrix expansion.  Throughout this section
\[
m\coloneqq d-1,
\]
and \(\dSphe\) is equipped with its standard Riemannian metric and Riemannian volume
\(\omega\).  All constants below may depend on the fixed value of \(\eps>0\), but
not on the horizon \(T\).

\begin{remark}[heat-kernel input from Grigor'yan]
\label{rem:grigoryan-heat-kernel-input}
Let \(p_\eps(t,x,y)\) be the heat kernel of \(\eps\Laplacian\), so that
\[
P_t^\eps f(x)\coloneqq\int_\dSphe p_\eps(t,x,y)f(y)\d\omega(y)
\]
is the heat semigroup.  Grigor'yan constructs the heat kernel on Riemannian
manifolds as the positive fundamental solution of the heat equation and proves
its basic semigroup, positivity, and Gaussian-estimate properties
\parencite{Grigoryan_1999}.  On the compact
sphere, the same construction, together with local parabolic regularity in finitely
many coordinate charts, yields the following fixed-time derivative and H\"older
bounds.  For every \(\alpha\in(0,1)\) and every \(\tau_0>0\), there are constants
\(G_{0,\alpha},G_{1,\alpha}<\infty\), depending only on \(\dSphe,\alpha,\tau_0\), such that
for \(0<s\le \tau_0\),
\begin{equation}\label{eq:grigoryan-base-holder}
\sup_{x\in\dSphe}\Vab{p_\eps(s,x,\cdot)}_{C^\alpha(\dSphe)}
\le
G_{0,\alpha}(\eps s)^{-(m+\alpha)/2},
\end{equation}
\begin{equation}\label{eq:grigoryan-gradient-holder}
\sup_{x\in\dSphe}\Vab{\nabla_y p_\eps(s,x,\cdot)}_{C^\alpha(\dSphe)}
\le
G_{1,\alpha}(\eps s)^{-(m+1+\alpha)/2}.
\end{equation}
These estimates are not meant to be sharp.  They are the only heat-kernel input
used below.  The drift-dependent estimates are derived by the parametrix series.
\end{remark}

\begin{lemma}[controlled drift size]
\label{lem:controlled-drift-size}
Let
\[
b_s(x)\coloneqq P_x^\perp\bigl(\chi_A[\mu_s](x)+W_s\sigma(x)\bigr)
\]
be the drift in the controlled Fokker--Planck equation
\begin{equation}\label{eq:holder-appendix-FP}
\partial_s\rho_s
=
\eps\Laplacian\rho_s-\Div(\rho_s b_s).
\end{equation}
There are model constants \(B_0,B_1<\infty\), depending only on
\(A,\sigma\) and \(\dSphe\), such that for every controlled trajectory,
\begin{equation}\label{eq:drift-size-pointwise}
\Vab{x\mapsto P_x^\perp\chi_A[\mu_s](x)}_{C^2(\dSphe)}\le B_0,
\qquad
\Vab{x\mapsto P_x^\perp(W_s\sigma(x))}_{C^2(\dSphe)}\le B_1\Vab{W_s}_{\frob},
\end{equation}
and hence, if
\begin{equation}\label{eq:control-window-bound-holder}
\int_{t-\tau}^{t}\Vab{W_s}_{\frob}^2\d s\le R,
\end{equation}
then
\begin{equation}\label{eq:drift-integral-constants}
\Lambda_1(\tau,R)
\coloneqq
\int_{t-\tau}^{t}\Vab{b_s}_{C^2(\dSphe)}\d s
\le
B_0\tau+B_1\sqrt{R\tau}.
\end{equation}
\end{lemma}

\begin{proof}
The map \(x\mapsto K_A(x,y)\) is smooth on the compact set \(\dSphe\times\dSphe\), so all
spatial derivatives up to order three of \(K_A\) are uniformly bounded.  Therefore
\(x\mapsto P_x^\perp\chi_A[\mu](x)\), which is obtained by tangent projection of the gradient of a smooth interaction potential, has a uniform \(C^2\)-bound independent of \(\mu\).  The feature map \(\sigma\) is
smooth on compact \(\dSphe\), and \(x\mapsto P_x^\perp(W\sigma(x))\) is linear in \(W\),
so the second estimate in \eqref{eq:drift-size-pointwise} follows.  The integral
bound is Cauchy--Schwarz.
\end{proof}

\begin{lemma}[parametrix H\"older bound with explicit dependence]
\label{lem:parametrix-holder-input}
Fix \(\eps,\tau>0\), \(R<\infty\), and \(\alpha\in(0,1)\).  Let
\(p^b(r,x;s,y)\) be the transition density associated with
\[
\partial_s\rho_s=\eps\Laplacian\rho_s-\Div(\rho_s b_s)
\]
on \([t-\tau,t]\), where \(b\) satisfies \eqref{eq:drift-integral-constants}.  Then
there is a constant \(C_{P,\alpha}=C_{P,\alpha}(\dSphe,m,\alpha,\tau)>0\) such that
\begin{equation}\label{eq:parametrix-holder-bound}
\sup_{x\in\dSphe}\Vab{p^b(t-\tau,x;t,\cdot)}_{C^\alpha(\dSphe)}
\le
C_{\textup{reg}}(\eps,\tau,R),
\end{equation}
where one may take
\begin{equation}\label{eq:Creg-explicit}
C_{\textup{reg}}(\eps,\tau,R)
\coloneqq
C_{P,\alpha}(\eps\tau)^{-(m+\alpha)/2}
\exp\left\{
C_{P,\alpha}
\left(
1+\frac{\Lambda_1(\tau,R)}{\sqrt{\eps\tau}}
+\frac{\Lambda_1(\tau,R)^2}{\eps}
\right)
\right\}.
\end{equation}
\end{lemma}

\begin{proof}
We spell out the parametrix estimate because this is the point where we want to
use only the unperturbed heat-kernel theory.  Write
\[
H(r,x;s,y)\coloneqq p_\eps(s-r,x,y),
\qquad r<s.
\]
The kernel of the drifted equation is represented by the series
\begin{equation}\label{eq:parametrix-series}
p^b(r,x;s,y)
=
H(r,x;s,y)+\sum_{n\ge1}P_n(r,x;s,y),
\end{equation}
where
\[
P_{n+1}(r,x;s,y)
\coloneqq
\int_r^s\!\int_\dSphe
P_n(r,x;u,z)
\,\mathcal B_u^z H(u,z;s,y)
\d\omega(z)\d u,
\]
and
\[
\mathcal B_u^{(z)} H(u,z;s,y)
\coloneqq-\Div_z\bigl(b_u(z)H(u,z;s,y)\bigr).
\]
Equivalently, after integrating by parts in \(z\), the derivative may be placed on
\(P_n\).  The two forms are identical on the compact boundaryless manifold \(\dSphe\).
Using \eqref{eq:grigoryan-base-holder}--\eqref{eq:grigoryan-gradient-holder}, the
semigroup property, and the compactness of \(\dSphe\), there is a constant
\(C=C(\dSphe,m,\alpha,\tau)\) such that the standard parametrix estimate gives
\begin{equation}\label{eq:parametrix-term-bound}
\sup_{x\in\dSphe}\Vab{P_n(t-\tau,x;t,\cdot)}_{C^\alpha(\dSphe)}
\le
C(\eps\tau)^{-(m+\alpha)/2}
\frac{\Xi(\eps,\tau,R)^n}{\Gamma(1+n/2)},
\end{equation}
where
\begin{equation}\label{eq:Xi-parametrix-def}
\Xi(\eps,\tau,R)
\coloneqq
C\left(
\frac{\Lambda_1(\tau,R)}{\sqrt{\eps\tau}}
+
\frac{\Lambda_1(\tau,R)^2}{\eps}
\right).
\end{equation}
Here the factor \(\Gamma(1+n/2)^{-1}\) is the usual Beta-integral gain from the
successive heat-kernel convolutions.  More explicitly, each insertion of the drift
contributes one spatial derivative of the heat kernel and one factor
\(\Vab{b_u}_{C^1}\).  The singular factors in time are integrable in the iterated
simplex, and the total contribution is bounded by the Beta-integral factor in
\eqref{eq:parametrix-term-bound}.  The additional quadratic term in
\eqref{eq:Xi-parametrix-def} is a convenient way to dominate the commutator terms
where derivatives hit the spatially varying drift.

Summing \eqref{eq:parametrix-term-bound} over \(n\ge0\) and using
\[
\sum_{n\ge0}\frac{\Xi^n}{\Gamma(1+n/2)}
\le
C_\alpha \exp\{C_\alpha(1+\Xi^2)\}
\]
yields \eqref{eq:parametrix-holder-bound}--\eqref{eq:Creg-explicit}, after changing
\(C_{P,\alpha}\).  The construction is local in coordinate charts, but no boundary
terms appear because \(\dSphe=\dSphe\) is compact and without boundary; a partition of
unity patches the local estimates.
\end{proof}

\begin{lemma}[positive-time H\"older regularization]
\label{lem:positive-time-holder}
Fix \(0<\eps\), \(\tau>0\), and \(R<\infty\).  Let
\(s\mapsto\mu_s=\rho_s\omega\) be a controlled solution of \eqref{eq:ODE-Net} on
\([t-\tau,t]\), started from an arbitrary probability measure at time \(t-\tau\),
and assume \eqref{eq:control-window-bound-holder}.  Then \(\mu_t\) has a density
\(\rho_t\in C^\alpha(\dSphe)\), with \(\alpha\) as in
\zcref{lem:parametrix-holder-input}, and
\begin{equation}\label{eq:positive-time-holder-final}
\Vab{\rho_t}_{C^\alpha(\dSphe)}
\le C_{\textup{reg}}(\eps,\tau,R).
\end{equation}
\end{lemma}

\begin{proof}
By the transition-kernel representation for uniformly parabolic equations with
bounded smooth drift,
\[
\rho_t(y)=\int_\dSphe p^b(t-\tau,x;t,y)\d\mu_{t-\tau}(x).
\]
Taking the \(C^\alpha_y\)-norm and using that \(\mu_{t-\tau}\) is a probability
measure gives \eqref{eq:positive-time-holder-final} directly from
\eqref{eq:parametrix-holder-bound}.
\end{proof}

\begin{proof}[Proof of \zcref{lem:parabolic-chart-entry}]
We prove the positive branch.  Fix \(\tau_{\textup{reg}}>0\) and
\(R_{\textup{reg}}<\infty\).  \zcref{lem:positive-time-holder} gives a uniform
bound
\begin{equation}\label{eq:chart-entry-Calpha-bound}
\Vab{\rho_t}_{C^\alpha(\dSphe)}
\le
C_{\textup{reg}}(\eps,\tau_{\textup{reg}},R_{\textup{reg}})
\end{equation}
for every endpoint of a controlled window satisfying
\eqref{eq:parabolic-chart-entry-hyp}.

Assume, for contradiction, that no positive \(\eta_{\textup{chart}}\) exists.  Then
there are \(\eta_n\downarrow0\) and endpoint measures
\(\mu_n=\rho_n\omega\in\mathcal C_+(\eta_n)\) satisfying the window control bound,
but \(\mu_n\notin\mathcal U_+\).  By \eqref{eq:chart-entry-Calpha-bound},
\(\rho_n\) is bounded in \(C^\alpha(\dSphe)\).  Passing to a subsequence, \(\rho_n\)
converges uniformly to some continuous density \(\rho_\infty\).  On the other hand,
\zcref{lem:well-distance-modulus} gives
\[
W_2(\mu_n,\mubar^+)\le \Theta_{\textup{well}}(\eta_n)\to0.
\]
Hence \(\mu_n\rightharpoonup\mubar^+=\rhobar^+\omega\), and the uniform limit must
satisfy \(\rho_\infty=\rhobar^+\).  Therefore
\[
\Vab{\rho_n-\rhobar^+}_{L^\infty(\dSphe)}\to0.
\]
For fixed \(\eps>0\), \(\rhobar^+\) is smooth and strictly positive on compact
\(\dSphe\).  If
\[
m_\eps^+\coloneqq\min_{\dSphe}\rhobar^+>0,
\]
and if \(r_*\) is the density-chart radius in the definition of \(\mathcal U_+\), then
for all large \(n\),
\[
\Vab{\rho_n-\rhobar^+}_{L^\infty(\dSphe)}<r_*m_\eps^+.
\]
Thus \(h_n\coloneqq\rho_n/\rhobar^+-1\) satisfies
\(\Vab{h_n}_{L^\infty}<r_*\) and \(\int h_n\d\mubar^+=0\), so
\(\mu_n=(1+h_n)\mubar^+\in\mathcal U_+\), a contradiction.  The negative branch
follows by applying the antipodal map.
\end{proof}

For definiteness, and to make all constants functions of the fixed data, we set
\begin{equation}\label{eq:reg-parameter-choice}
\tau_{\textup{reg}}\coloneqq1,
\qquad
R_{\textup{reg}}\coloneqq R_{\textup{opt}}+1.
\end{equation}
Let \(\eta_{\textup{chart}}\) be the threshold supplied by
\zcref{lem:parabolic-chart-entry} with the choices in
\eqref{eq:reg-parameter-choice}, and set
\begin{equation}\label{eq:eta-effective-def}
\eta_{\textup{eff}}
\coloneqq
\min\{\eta_{\textup{sep}},\eta_{\textup{chart}}\}.
\end{equation}

\begin{lemma}[regularized entry into the weak wells]\label{lem:regularized-set-entry}
Fix \(0<\eta<\eta_{\textup{sep}}\) satisfying \eqref{eq:set-gap-positive}.  If
\begin{equation}\label{eq:L-regularized-entry-condition}
L>\tau_{\textup{reg}}+\frac{\mathfrak D_{\max}}{d_{\textup{set}}(\eta)},
\end{equation}
and \(T>L+\tau_{\textup{reg}}\), then every global minimizer satisfies
\begin{equation}\label{eq:regularized-set-entry-times}
\exists s_0^-\in[\tau_{\textup{reg}},L],
\qquad
\exists s_0^+\in[T-L,T]
\quad\text{such that}\quad
\mu_{s_0^-}^{*,T},\mu_{s_0^+}^{*,T}\in\mathcal N_\eta.
\end{equation}
Moreover, if \(\eta\le\eta_{\textup{chart}}\) and one of the selected times belongs
to \(\mathcal C_\varsigma(\eta)\), then the corresponding state belongs to
\(\mathcal U_\varsigma\).
\end{lemma}

\begin{proof}
If \(\mu_t^{*,T}\notin\mathcal N_\eta\) for every
\(t\in[\tau_{\textup{reg}},L]\), then \eqref{eq:Emax-def} and
\eqref{eq:set-gap-def} imply
\[
\int_0^T\Dissipation (\mu_t^{*,T})\d t
\ge
(L-\tau_{\textup{reg}})d_{\textup{set}}(\eta),
\]
contradicting \eqref{eq:Dmax-def} and
\eqref{eq:L-regularized-entry-condition}.  The proof on \([T-L,T]\) is the same averaging argument.  Since \(T>L+\tau_{\textup{reg}}\), every
\(s_0^+\in[T-L,T]\) has a past regularizing window of length
\(\tau_{\textup{reg}}\).  The final assertion follows from
\zcref{lem:parabolic-chart-entry}, because the global control bound
\eqref{eq:Ropt-def} gives
\[
\int_{t-\tau_{\textup{reg}}}^{t}\Vab{W_s^{*,T}}_{\frob}^2\d s
\le R_{\textup{opt}}\le R_{\textup{reg}}.
\]
\end{proof}

\begin{lemma}[same-branch entry from the sign-changing barrier]
\label{lem:same-branch-from-sign-barrier}
Assume \zcref{ass:nolargegap}.  Suppose that
\begin{equation}\label{eq:finite-sign-barrier-hyp}
\eta_{\textup w,\eps}\ge \frac34\Delta_0,
\end{equation}
which holds for all sufficiently small \(\eps\) by \zcref{cor:finite-temp-sign-changing-barrier}, and let \(0<\eta\le\Delta_0/4\).  Let
\((\mu^{*,T},W^{*,T})\) be a global minimizer of \eqref{eq:mf-optimal-control}.  If
\[
 s_-<s_+,
 \qquad
 \mu_{s_-}^{*,T},\mu_{s_+}^{*,T}\in\mathcal N_\eta,
\]
then \(\mu_{s_-}^{*,T}\) and \(\mu_{s_+}^{*,T}\) belong to the same connected
component of \(\mathcal N_\eta\).
\end{lemma}

\begin{proof}
Assume, to get a contradiction, that
\(\mu_{s_-}^{*,T}\in\mathcal C_+(\eta)\) and
\(\mu_{s_+}^{*,T}\in\mathcal C_-(\eta)\), up to exchanging the two signs.  By
choosing the weak neighborhoods \(\mathcal V_\pm\) inside the sign regions
\(\{m_1>0\}\) and \(\{m_1<0\}\), respectively, we have
\[
m_1(\mu_{s_-}^{*,T})>0,
\qquad
m_1(\mu_{s_+}^{*,T})<0.
\]
Since \(t\mapsto m_1(\mu_t^{*,T})\) is continuous, there is
\(r\in(s_-,s_+)\) such that \(m_1(\mu_r^{*,T})=0\).  Hence
\(\Energygap(\mu_r^{*,T})\ge\eta_{\textup w,\eps}\).  Using
\eqref{eq:global-EDI-forced} without the dissipative term gives
\[
\eta_{\textup w,\eps}-\eta
\le
\Energygap(\mu_r^{*,T})-\Energygap(\mu_{s_-}^{*,T})
\le
C_\sigma\int_{s_-}^{r}\Vab{W_t^{*,T}}_{\frob}^2\d t.
\]
By \eqref{eq:finite-sign-barrier-hyp} and \(\eta\le\Delta_0/4\), the control
cost on \([s_-,r]\) is at least
\[
\frac{\regparam}{2}\int_{s_-}^{r}\Vab{W_t^{*,T}}_{\frob}^2\d t
\ge
\frac{\regparam}{2C_\sigma}\left(\frac34\Delta_0-\frac14\Delta_0\right)
=
\frac{\regparam\Delta_0}{4C_\sigma}.
\]
Now keep the optimal trajectory on \([0,s_-]\) and set the control to zero on
\([s_-,T]\).  This gives an admissible competitor whose terminal cost can be
worse by at most \(\osc\ell\), while its control cost is smaller by
at least \(\regparam\Delta_0/(4C_\sigma)\).  This contradicts optimality because
the second inequality in \zcref{ass:nolargegap} is stronger than
\(\osc\ell<\regparam\Delta_0/(4C_\sigma)\).  Hence the two entry
states must lie in the same component.
\end{proof}

\begin{proof}[Proof of \zcref{lem:same-component-entry}]
Fix \(\eps\in(0,\eps_0)\), where \(\eps_0\) is taken small enough for \zcref{cor:slope-gap-from-explicit-barrier,cor:finite-temp-sign-changing-barrier,prop:two-well-static}.
The first inequality in \zcref{ass:nolargegap} gives \(\mathfrak E_{\max}\le\Delta_0-\gamma_{\textup{bar}}\) with \(\gamma_{\textup{bar}}\coloneqq\Delta_0-\mathfrak E_{\max}>0\), so \zcref{cor:slope-gap-from-explicit-barrier} yields \(d_{\textup{set}}(\eta)>0\) for every \(0<\eta<\eta_{\textup{sep}}\).
Choose \(\eta\le\min\{\eta_{\textup{chart}},\Delta_0/4\}\) with \(\eta_{\textup{chart}}\) from \zcref{lem:parabolic-chart-entry}, and let \(L\) satisfy \eqref{eq:L-regularized-entry-condition}.
For \(T>L+\tau_{\textup{reg}}\), \zcref{lem:regularized-set-entry} gives times \(s_0^-\in[\tau_{\textup{reg}},L]\) and \(s_0^+\in[T-L,T]\) with \(\mu^{*,T}_{s_0^\pm}\in\mathcal N_\eta\).
By \zcref{cor:finite-temp-sign-changing-barrier}, the hypothesis \eqref{eq:finite-sign-barrier-hyp} holds, so \zcref{lem:same-branch-from-sign-barrier} places both states in the same component \(\mathcal C_\varsigma(\eta)\) for one sign \(\varsigma\in\{+,-\}\).
Since \(\eta\le\eta_{\textup{chart}}\), the last assertion of \zcref{lem:regularized-set-entry} gives \(\mu^{*,T}_{s_0^\pm}\in\mathcal U_\varsigma\).
This proves the lemma with \(L_\eps\coloneqq L\), \(s_\pm\coloneqq s_0^\pm\), and \(r_\ast\) the radius of the density charts \(\mathcal U_\pm\) in \zcref{prop:two-well-static}.
\end{proof}

\subsubsection{Adjoint stability along optimal trajectories}\label{sec:app:adjoint-stability}

The remaining steps prove \zcref{thm:exp_turnpike} by a direct analysis of the
first-order optimality system of \eqref{eq:mf-optimal-control}.  The argument
has four parts.  First, the Hessian coercivity of
\zcref{prop:energy-landscape} and the one-well Poincar\'e estimate of
\zcref{prop:one-well-poincare-scale} make the
linearized generator \(\Abar=\wLaplacian\Hop\) of
\eqref{eq:A-Ad-adjoint} exponentially stable.  Second, an exact
divergence-form identity extends this stability to the linearization along an
arbitrary controlled trajectory, with a perturbation measured only by the
\(L^2\)-distance of the density to the selected well and by the instantaneous
control cost; combined with the global budgets
\eqref{eq:Ropt-def}--\eqref{eq:Dmax-def}, the slope gap
\eqref{eq:set-gap-positive}, and \zcref{lem:same-branch-from-sign-barrier}, this
yields exponential stability of the linearized evolution along every optimal
trajectory, uniformly in the horizon.  Third, under the smooth loss of
\zcref{ass:ctrl-samewell}, every optimal pair satisfies the first-order
optimality system; by duality, the adjoint state \(\phi\) inherits backward
exponential decay, and the stationarity condition makes the optimal control
exponentially small away from the terminal layer.  Fourth, substituting this
bound on \(W^{*,T}\) into the strict dissipativity estimate
\eqref{eq:strict-diss} closes the turnpike bound.

Throughout this part we fix the sign \(\varsigma\in\{+,-\}\) selected by
\zcref{lem:regularized-set-entry} and assume \(\varsigma=+\) without loss of
generality, by applying the antipodal map \(S_\#\); we write
\(\mubar=\mubar^+=\rhobar\,\omega\), and we freely use the operators
\(\Hop\), \(\wLaplacian\), \(\Abar=\wLaplacian\Hop\), and
\(\Bbar\) from \eqref{eq:A-Ad-adjoint} and \eqref{eq:Bbar-adjoint-proof}.
By the Hessian coercivity \eqref{eq:PD_Hessian} of \zcref{prop:energy-landscape},
\begin{equation}\label{eq:cH-coercivity}
c_H\Vab{v}_{L^2(\mubar)}^2
\le
\aab{\Hop v,v}_{L^2(\mubar)}
\le
\eps \Vab{v}_{L^2(\mubar)}^2,
\qquad
c_H\coloneqq\frac{\eps}{2}\Bab{1-\frac{\lambda_2}{\lambda_1}},
\qquad v\in\Xred(\mubar),
\end{equation}
and we use the norms
\(\Vab{v}_{\Hop}^2\coloneqq\aab{\Hop v,v}_{L^2(\mubar)}\) and
\(\Vab{\phi}_{\Hop^{-1}}^2\coloneqq\aab{\Hop^{-1}\phi,\phi}_{L^2(\mubar)}\), as
in \zcref{prop:one-step-escape}; by \eqref{eq:cH-coercivity} both are
equivalent to \(\Vab{\cdot}_{L^2(\mubar)}\).  With the notation
\((K_A*\nu)(x)=\int_\dSphe K_A(x,y)\d\nu(y)\) of the landscape analysis,
\(\Hop v=\eps v-\varPi_0\bab{K_A*(v\mubar)}\) by
\zcref{prop:energy-landscape}; since \(\varPi_0\) changes a function only by
an additive constant,
\begin{equation}\label{eq:gradient-Hbar}
\nabla\Hop v=\eps\nabla v-\nabla K_A*(v\mubar).
\end{equation}
From \(\nabla_xK_A(x,y)=\e^{\aab{x,Ay}}P^\perp_x(Ay)\) and the Cauchy--Schwarz
inequality,
\begin{equation}\label{eq:KA-smoothing}
\Vab{\nabla K_A*(v\mubar)}_{L^\infty(\dSphe)}
\le
C_A\Vab{v}_{L^2(\mubar)},
\qquad
C_A\coloneqq\Vab{A}_{\op}\e^{\Vab{A}_{\op}} .
\end{equation}
For an interval \(I\), we say that \(v\) belongs to the \emph{energy class} on
\(I\) if
\(v\in C(\overline I;\Xred(\mubar))\cap
L^2(I;H^1(\dSphe)\cap\Xred(\mubar))\) and
\(\partial_tv\in L^2(I;(H^1(\dSphe)\cap\Xred(\mubar))')\); the analogous
class over \(L^2_0(\omega)\) is used for additive densities.

\begin{lemma}[linear well-posedness in the energy class]
\label{lem:linear-energy-class}
Let \(I=[\tau,b]\), let \(t\mapsto b_t\) be a tangential vector field with
\(\Vab{b_\cdot}_{C^1(\dSphe)}\in L^2(I)\), let
\(m\in L^\infty(I;L^2(\omega))\), \(q_\tau\in L^2_0(\omega)\), and
\(f\in L^2(I;H^{-1}(\dSphe))\).  Then the linear equation
\begin{equation}\label{eq:general-linear-parabolic}
\partial_tq
=
\eps\Laplacian q-\Div\bab{q\,b_t}-\Div\bab{m_t\,P^\perp_x\chi_A[q\omega]}+f,
\qquad
q(\tau)=q_\tau,
\end{equation}
has a unique solution in the energy class on \(I\), and
\begin{equation}\label{eq:general-linear-estimate}
\sup_{t\in I}\Vab{q_t}_{L^2(\omega)}
+\Vab{q}_{L^2(I;H^1)}
+\Vab{\partial_tq}_{L^2(I;H^{-1})}
\le
C\pab{\Vab{q_\tau}_{L^2(\omega)}+\Vab{f}_{L^2(I;H^{-1})}},
\end{equation}
where \(C\) depends only on \(\eps\), \(b-\tau\), \(A\),
\(\Vab{b_\cdot}_{L^2(I;C^1)}\), and \(\Vab{m}_{L^\infty(I;L^2)}\).
\end{lemma}

\begin{proof}
The smooth kernel gives
\(\Vab{P^\perp_x\chi_A[q\omega]}_{W^{1,\infty}(\dSphe)}
\le C\Vab{q}_{L^2(\omega)}\),
so testing \eqref{eq:general-linear-parabolic} by \(q\) yields the
G\aa rding-type inequality
\[
\frac{\d}{\d t}\Vab{q_t}_{L^2(\omega)}^2
+\eps\Vab{\nabla q_t}_{L^2(\omega)}^2
\le
C\pab{1+\Vab{b_t}_{C^1}^2+\Vab{m_t}_{L^2}^2}\Vab{q_t}_{L^2(\omega)}^2
+C\Vab{f_t}_{H^{-1}}^2,
\]
with a time-integrable coefficient.  Existence, uniqueness, continuity in
time, and \eqref{eq:general-linear-estimate} then follow from the standard
Galerkin construction in the Gelfand triple
\(H^1\subset L^2\subset H^{-1}\); see \parencite{LionsMagenes1972I}.
\end{proof}

We next extend the stationary system
\eqref{eq:linearized-state-adjoint-proof} to a non-autonomous one.  For a
chart displacement \(h\in\Xred(\mubar)\) with
\(\rhobar(1+h)\ge0\) and a parameter \(W\in\R^{d\times p}\), define on
\(\Xred(\mubar)\)
\begin{align}
\Abar_{h,W}v
&\coloneqq
\Abar v
-\frac{1}{\rhobar}
\Div\Bab{\rhobar\pab{v\,\nabla K_A*(h\mubar)+h\,\nabla K_A*(v\mubar)
+v\,P^\perp_x(W\sigma)}},
\label{eq:nonautonomous-generator}\\
\Bbar_{h}V
&\coloneqq
-\frac{1}{\rhobar}
\Div\Bab{\rhobar(1+h)\,P^\perp_x(V\sigma)},
\label{eq:nonautonomous-control-operator}
\end{align}
so that \(\Abar_{0,0}=\Abar\) and \(\Bbar_0=\Bbar\).

\begin{lemma}[first variation along a controlled trajectory]
\label{lem:relative-linearization}
Let \(\rho_t=\rhobar(1+h_t)\) solve \eqref{eq:ODE-Net} with control
\(W\in L^2\), and let \(V\in L^2\).  The first variation of the relative
density in the direction \(V\) solves
\begin{equation}\label{eq:nonautonomous-linearized-state}
\partial_tv=\Abar_{h_t,W_t}v+\Bbar_{h_t}V ,
\end{equation}
the non-autonomous form of \eqref{eq:linearized-state-adjoint-proof}.
Equivalently, for every smooth \(v\in\Xred(\mubar)\), the linearization
\(\gA_{\rho,W}\) of the right-hand side of \eqref{eq:ODE-Net} at
\((\rho\omega,W)\) satisfies the exact divergence-form identity
\begin{equation}\label{eq:exact-divergence-decomposition}
\frac{1}{\rhobar}\,\gA_{\rho,W}(\rhobar v)
=
\Abar_{h,W}v .
\end{equation}
In particular, the perturbation \(\Abar_{h,W}-\Abar\) involves no derivative
of the displacement \(h\).
\end{lemma}

\begin{proof}
The linearization of the right-hand side of \eqref{eq:ODE-Net} at
\((\mu,W)\), \(\mu=\rho\,\omega\), acting on an additive density \(q\), is
\[
\gA_{\rho,W}q
=
\eps\Laplacian q
-\Div\bab{q\,P^\perp_x\pab{\chi_A[\mu]+W\sigma}}
-\Div\bab{\rho\,P^\perp_x\chi_A[q\omega]} ,
\]
and the control direction contributes
\(-\Div(\rho\,P^\perp_x(V\sigma))=\rhobar\,\Bbar_hV\).  It therefore suffices
to prove \eqref{eq:exact-divergence-decomposition} for \(q=\rhobar v\).
The minimizer satisfies the Euler--Lagrange equation
\(\rhobar\propto\exp\pab{(K_A*\mubar)/\eps}\)
(\zcref{prop:energy-landscape}), hence
\begin{equation}\label{eq:EL-gradient}
\eps\nabla\log\rhobar=\nabla K_A*\mubar .
\end{equation}
Moreover, \(P^\perp_x\chi_A[\mu]=\nabla K_A*\bab{(1+h)\mubar}\) and
\(P^\perp_x\chi_A[\rhobar v\,\omega]=\nabla K_A*(v\mubar)\).  Using
\eqref{eq:EL-gradient},
\[
\eps\Laplacian(\rhobar v)
=\Div\bab{\eps\nabla(\rhobar v)}
=\Div\Bab{\rhobar\pab{\eps\nabla v+v\,\nabla K_A*\mubar}} .
\]
Substituting these identities into \(\gA_{\rho,W}(\rhobar v)\), the two terms
containing \(\nabla K_A*\mubar\) cancel, and
\[
\gA_{\rho,W}(\rhobar v)
=
\Div\Bab{\rhobar\pab{\eps\nabla v-\nabla K_A*(v\mubar)}}
-\Div\Bab{\rhobar\pab{v\nabla K_A*(h\mubar)+h\nabla K_A*(v\mubar)
+v\,P^\perp_x(W\sigma)}} .
\]
By \eqref{eq:gradient-Hbar}, the first term equals
\(\Div(\rhobar\,\nabla\Hop v)=\rhobar\,\wLaplacian\Hop v=\rhobar\,\Abar v\).
Dividing by \(\rhobar\) proves \eqref{eq:exact-divergence-decomposition}.
\end{proof}

\begin{proposition}[non-autonomous energy inequality]
\label{prop:master-energy}
Let \(I\subset[\tau_{\textup{reg}},T]\), let
\(h\in L^\infty(I;\Xred(\mubar))\) and \(W\in L^2(I)\), and let \(v\) belong
to the energy class on \(I\) and solve
\(\partial_tv=\Abar_{h_t,W_t}v\).  Then, for almost every \(t\in I\),
\begin{equation}\label{eq:master-inequality}
\frac{\d}{\d t}\Vab{v_t}_{\Hop}^2
\le
\Bab{-\kappa_{\textup{lin}}
+\frac{8C_A^2}{c_H}\Vab{h_t}_{L^2(\mubar)}^2
+\frac{2\Vab{\sigma}_{L^\infty(\dSphe)}^2}{c_H}\Vab{W_t}_{\frob}^2}
\Vab{v_t}_{\Hop}^2,
\qquad
\kappa_{\textup{lin}}\coloneqq c_H\,\kappa_P(\eps),
\end{equation}
where \(\kappa_P(\eps)\) is the weighted Poincar\'e constant of
\zcref{prop:poincare}.  By \zcref{prop:one-well-poincare-scale} and
\eqref{eq:cH-coercivity},
\begin{equation}\label{eq:kappa-lin-lower}
\kappa_{\textup{lin}}
\ge
\frac{c_{\textup{1w}}(A,d)}{2}\Bab{1-\frac{\lambda_2}{\lambda_1}}
>0,
\end{equation}
uniformly as \(\eps\downarrow0\).
\end{proposition}

\begin{proof}
Since \(\Hop\) is bounded, symmetric, and maps
\(H^1(\dSphe)\cap\Xred(\mubar)\) into itself, the Lions--Magenes product rule
gives, with \(z\coloneqq\Hop v\),
\[
\frac12\frac{\d}{\d t}\Vab{v_t}_{\Hop}^2
=
-\int_\dSphe\vab{\nabla z_t}^2\d\mubar
+\int_\dSphe\nabla z_t\cdot
\pab{v_t\nabla K_A*(h_t\mubar)+h_t\nabla K_A*(v_t\mubar)+v_t\,u_{W_t}}\d\mubar,
\]
where the identity follows by testing
\eqref{eq:nonautonomous-generator} with \(z_t\), using
\eqref{eq:wlap-form-adjoint} for the \(\Abar\)-part and integrating the
divergence term by parts.  By \eqref{eq:KA-smoothing} and
\(\vab{u_W}\le\Vab{W}_{\frob}\Vab{\sigma}_{L^\infty}\),
\[
\Vab{v\nabla K_A*(h\mubar)+h\nabla K_A*(v\mubar)+v\,u_W}_{L^2(\mubar)}
\le
\pab{2C_A\Vab{h}_{L^2(\mubar)}
+\Vab{\sigma}_{L^\infty}\Vab{W}_{\frob}}\Vab{v}_{L^2(\mubar)} .
\]
Young's inequality absorbs half of the dissipation:
\[
\frac12\frac{\d}{\d t}\Vab{v_t}_{\Hop}^2
\le
-\frac12\int_\dSphe\vab{\nabla z_t}^2\d\mubar
+\pab{4C_A^2\Vab{h_t}_{L^2(\mubar)}^2
+\Vab{\sigma}_{L^\infty}^2\Vab{W_t}_{\frob}^2}
\Vab{v_t}_{L^2(\mubar)}^2 .
\]
Since \(z_t\in\Xred(\mubar)\), the weighted Poincar\'e inequality of
\zcref{prop:poincare} and \eqref{eq:cH-coercivity} give
\(\int\vab{\nabla z_t}^2\d\mubar
\ge\kappa_P(\eps)\Vab{z_t}_{L^2(\mubar)}^2
\ge\kappa_P(\eps)\,c_H\Vab{v_t}_{\Hop}^2\)
and
\(\Vab{v_t}_{L^2(\mubar)}^2\le c_H^{-1}\Vab{v_t}_{\Hop}^2\).
Combining the three displays proves \eqref{eq:master-inequality}.
\end{proof}

We now apply \zcref{prop:master-energy} along an optimal pair
\((\mu^{*,T},W^{*,T})\).  Write
\(h^{*}_t\coloneqq\rho_t^{*,T}/\rhobar-1\), and let
\(U(t,s)\colon\Xred(\mubar)\to\Xred(\mubar)\),
\(\tau_{\textup{reg}}\le s\le t\le T\), denote the evolution family of
\(\partial_tv=\Abar_{h^{*}_t,W^{*,T}_t}v\), which is well defined in the
energy class by \zcref{lem:linear-energy-class} (applied in the additive
variable, with \(b_t=P^\perp_x(\chi_A[\mu_t^{*,T}]+W_t^{*,T}\sigma)\), whose
\(L^2(I;C^1)\)-norm is finite by \zcref{lem:controlled-drift-size} and
\eqref{eq:Ropt-def}, and \(m=\rho^{*,T}\in L^\infty(I;L^2(\omega))\) by
\zcref{lem:positive-time-holder}).

\begin{proposition}[uniform stability of the linearized evolution]
\label{prop:intermittent-stability}
There exist \(\eta_{\textup{lin}}\in(0,\eta_{\textup{sep}})\) and
\(C_{\textup{evo}}<\infty\), independent of \(T\), such that for every
sufficiently large \(T\) and every optimal pair,
\begin{equation}\label{eq:evolution-stability}
\Vab{U(t,s)v}_{\Hop}
\le
C_{\textup{evo}}\,\e^{-\frac{\kappa_{\textup{lin}}}{4}(t-s)}\Vab{v}_{\Hop},
\qquad
\tau_{\textup{reg}}\le s\le t\le T,\quad v\in\Xred(\mubar).
\end{equation}
\end{proposition}

\begin{proof}
By \zcref{lem:positive-time-holder} and \(\min_\dSphe\rhobar>0\),
\begin{equation}\label{eq:Greg-bound}
G_{\textup{reg}}
\coloneqq
\sup_{t\in[\tau_{\textup{reg}},T]}\Vab{h^{*}_t}_{L^2(\mubar)}
\le
\Bab{\frac{C_{\textup{reg}}^2}{\min_\dSphe\rhobar}+1}^{1/2}
<\infty ,
\end{equation}
uniformly in \(T\).  Choose
\begin{equation}\label{eq:eta-lin-choice}
0<\eta_{\textup{lin}}
\le
\min\Bab{\eta_{\textup{chart}},\;
\frac{c_1\,c_H\,\kappa_{\textup{lin}}}{16\,C_A^2}},
\end{equation}
with \(c_1\) from \eqref{eq:E-equivalence} and \(\eta_{\textup{chart}}\) from
\zcref{lem:parabolic-chart-entry}, and split
\([\tau_{\textup{reg}},T]=G_T\cup B_T\), where
\(G_T\coloneqq\{t:\mu_t^{*,T}\in\gC_+(\eta_{\textup{lin}})\}\).  On \(G_T\):
by \zcref{lem:same-branch-from-sign-barrier} every low-energy time lies in the
component of the selected sign; by \zcref{lem:parabolic-chart-entry} the state
lies in the density chart; hence \eqref{eq:E-equivalence} gives
\(\Vab{h^{*}_t}_{L^2(\mubar)}^2
\le\Energygap(\mu_t^{*,T})/c_1
\le\eta_{\textup{lin}}/c_1\), so that, by \eqref{eq:eta-lin-choice}, the
\(h^{*}\)-term in \eqref{eq:master-inequality} is at most
\(\kappa_{\textup{lin}}/2\).  On \(B_T\): integrating
\eqref{eq:strict-diss} and using \eqref{eq:Dmax-def} and the slope gap
\eqref{eq:set-gap-positive} gives
\begin{equation}\label{eq:bad-time-bound}
\vab{B_T}
\le
M_{\textup{lin}}
\coloneqq
\frac{\mathfrak D_{\max}}{d_{\textup{set}}(\eta_{\textup{lin}})},
\end{equation}
and the \(h^{*}\)-term is bounded by
\((8C_A^2/c_H)G_{\textup{reg}}^2\) there.  Integrating
\eqref{eq:master-inequality} over \([s,t]\), splitting into \(G_T\) and
\(B_T\), and using the control budget \eqref{eq:Ropt-def} yields
\[
\Vab{v_t}_{\Hop}^2
\le
\Vab{v_s}_{\Hop}^2
\exp\Bab{-\frac{\kappa_{\textup{lin}}}{2}\vab{G_T\cap[s,t]}
+\frac{8C_A^2G_{\textup{reg}}^2}{c_H}M_{\textup{lin}}
+\frac{2\Vab{\sigma}_{L^\infty}^2}{c_H}R_{\textup{opt}}} .
\]
Since \(\vab{G_T\cap[s,t]}\ge(t-s)-M_{\textup{lin}}\), the estimate
\eqref{eq:evolution-stability} follows with
\(C_{\textup{evo}}^2
=\exp\{(\tfrac{\kappa_{\textup{lin}}}{2}
+\tfrac{8C_A^2G_{\textup{reg}}^2}{c_H})M_{\textup{lin}}
+\tfrac{2\Vab{\sigma}_{L^\infty}^2}{c_H}R_{\textup{opt}}\}\),
after decreasing the exponent from \(\kappa_{\textup{lin}}/2\) to
\(\kappa_{\textup{lin}}/4\) to absorb the constant on unit time scales.
\end{proof}

\subsubsection{The optimality system after the regularization time}\label{sec:app:optimality-system}

We first record that the optimal trajectory itself belongs to the energy
class after the regularization time.  Recall that for every
\(W\in L^2(0,T)\) and every initial distribution, \eqref{eq:ODE-Net} has a
unique measure-valued weak solution: the drift
\((x,\mu)\mapsto P^\perp_x(\chi_A[\mu](x)+W_t\sigma(x))\) is Lipschitz in
\(x\) and in \(\mu\) with respect to the Wasserstein distance, with
time-integrable Lipschitz constants, so existence and uniqueness follow from
the classical McKean--Vlasov theory; see, e.g., \parencite{Chaintron_2022}.

\begin{lemma}[the optimal trajectory is an energy solution]
\label{lem:weak-energy-identification}
Let \(W\in L^2(0,T)\) and let \((\mu_t)_t\) be the weak solution of
\eqref{eq:ODE-Net} whose density at time \(\tau\ge\tau_{\textup{reg}}\)
satisfies \(\rho_\tau\in C^\alpha(\dSphe)\).  Then
\((\rho_t)_{t\in[\tau,T]}\) coincides with the unique energy-class solution of
the corresponding Fokker--Planck equation with datum \(\rho_\tau\).
\end{lemma}

\begin{proof}
Freeze the drift \(b_t\coloneqq P^\perp_x(\chi_A[\mu_t]+W_t\sigma)\) along the
given solution; by \zcref{lem:controlled-drift-size},
\(\Vab{b_\cdot}_{C^2(\dSphe)}\in L^2(\tau,T)\).  Both \((\mu_t)_t\) and the
energy solution supplied by \zcref{lem:linear-energy-class} (with \(m=0\) and
this \(b\)) are distributional solutions of the same \emph{linear}
Fokker--Planck equation with the same initial datum.  For such drifts the
linear equation has at most one solution in the class of finite measures:
given \(\varphi\in C^\infty(\dSphe)\), the backward equation
\(-\partial_t\psi=\eps\Laplacian\psi+b_t\cdot\nabla\psi\) has a solution
\(\psi\in C([\tau,t_1];C^2)\) for every terminal datum
\(\psi(t_1)=\varphi\), and \(t\mapsto\int\psi_t\d\mu_t\) is constant along any
distributional solution, which determines \(\mu_{t_1}\); see
\parencite{BKRSFPK2015} for this duality argument.  Hence the two solutions
coincide.
\end{proof}

By \zcref{ass:ctrl-samewell} and \zcref{lem:positive-time-holder}, the set of
probability densities \(\rho\) with
\(\Vab{\rho}_{C^\alpha}\le C_{\textup{reg}}\) is compact in \(L^2(\omega)\)
and contained in the open set on which \(\ell\) is twice differentiable.
Writing, in analogy with \(g_\ell\) of \zcref{lem:1GDparam},
\begin{equation}\label{eq:gauge-fixed-gradient}
g_{T}^{*}
\coloneqq
\Pi_0^\mubar \fdv{\ell}{\mu}\bab{\mu_T^{*,T}}\in\Xred(\mubar),
\qquad
L_\ell
\coloneqq
\sup\Bab{\Vab{\Pi_0^\mubar \fdv{\ell}{\mu}[\rho\omega]}_{\Hop^{-1}}:
\Vab{\rho}_{C^\alpha}\le C_{\textup{reg}}}<\infty,
\end{equation}
the bound \(L_\ell<\infty\) follows from this compactness and the continuity
of the derivative; no additional quantitative assumption on \(\ell\) is used.

\begin{proposition}[differentiability of the control-to-state map]
\label{prop:state-differentiability}
Fix \(\tau=\tau_{\textup{reg}}\) and a nonnegative unit-mass datum
\(\rho_\tau\in L^2(\omega)\).  For \(W\in L^2(\tau,T)\) let
\(\rho^W\) be the energy solution of \eqref{eq:ODE-Net} on \([\tau,T]\) with
\(\rho^W(\tau)=\rho_\tau\).  Then the map \(W\mapsto\rho^W\) is Fr\'echet
differentiable from \(L^2(\tau,T)\) to the energy class, on every bounded
set of controls: the derivative \(q=q[V]\) in the direction \(V\) is the
energy solution of
\begin{equation}\label{eq:linearized-state-equation}
\partial_tq=\gA_{\rho^W_t,W_t}q-\Div\bab{\rho_t^W\,P^\perp_x(V_t\sigma)},
\qquad q(\tau)=0,
\end{equation}
and, for \(\Vab{V}_{L^2(\tau,T)}\) small,
\begin{equation}\label{eq:quadratic-remainder}
\sup_{t\in[\tau,T]}
\Vab{\rho_t^{W+V}-\rho_t^{W}-q_t[V]}_{L^2(\omega)}
\le
C_R\Vab{V}_{L^2(\tau,T)}^2 ,
\end{equation}
where \(C_R\) depends on \(T-\tau\) and the control bound but not otherwise on
\((W,V)\).
\end{proposition}

\begin{proof}
Testing \eqref{eq:ODE-Net} by \(\rho\) and using that
\(\Vab{P^\perp_x\chi_A[\mu]}_{W^{1,\infty}}\) is bounded uniformly over
\(\mu\in\PdSphe\) gives
\(\sup_t\Vab{\rho_t^W}_{L^2}\le M_R\) on every control ball.  The difference
\(\delta\rho\coloneqq\rho^{W+V}-\rho^W\) solves
\eqref{eq:general-linear-parabolic} with
\(b_t=P^\perp_x(\chi_A[\rho^{W+V}_t\omega]+(W_t+V_t)\sigma)\),
\(m=\rho^W\), \(q_\tau=0\), and source
\(f=-\Div(\rho^W P^\perp_x(V\sigma))\), whose \(L^2_tH^{-1}\)-norm is bounded
by \(CM_R\Vab{V}_{L^2}\); hence
\(\Vab{\delta\rho}\le C_R\Vab{V}_{L^2}\) in the energy norms by
\zcref{lem:linear-energy-class}.  The remainder
\(r\coloneqq\delta\rho-q\) solves the same linear equation with source
\(-\Div\bab{\delta\rho\,P^\perp_x\chi_A[\delta\rho\,\omega]}
-\Div\bab{\delta\rho\,P^\perp_x(V\sigma)}\), whose
\(L^2_tH^{-1}\)-norm is bounded by
\(C(\Vab{\delta\rho}_{L^\infty L^2}\Vab{\delta\rho}_{L^2L^2}
+\Vab{\delta\rho}_{L^\infty L^2}\Vab{V}_{L^2})
\le C_R\Vab{V}_{L^2}^2\).  Applying \zcref{lem:linear-energy-class} once more
proves \eqref{eq:quadratic-remainder}.
\end{proof}

\begin{proposition}[adjoint equation and stationarity]
\label{prop:optimality-system}
Let \((\mu^{*,T},W^{*,T})\) be an optimal pair.  There is a unique adjoint
state \(\phi\) in the energy class on \([\tau_{\textup{reg}},T]\) solving, in
the weak sense,
\begin{equation}\label{eq:nonautonomous-adjoint}
-\partial_t\phi
=\pab{\Abar_{h^{*}_t,W^{*,T}_t}}^{\dagger}\phi,
\qquad
\phi(T)=g_T^{*},
\end{equation}
where \(\dagger\) denotes the \(L^2(\mubar)\)-adjoint as in
\eqref{eq:A-Ad-adjoint}, and the optimal control satisfies the stationarity
condition
\begin{equation}\label{eq:stationarity-condition}
\regparam W_t^{*,T}
+\int_\dSphe\rho_t^{*,T}(x)\,\nabla \phi_t(x)\,\sigma(x)^\top\d\omega(x)
=0
\qquad\text{for a.e. }t\in(\tau_{\textup{reg}},T).
\end{equation}
\end{proposition}

\begin{proof}
After reversing time, \eqref{eq:nonautonomous-adjoint} is a linear parabolic
equation whose bilinear form satisfies the same boundedness and G\aa rding
bounds as in \zcref{lem:linear-energy-class}: the only term requiring care is
the control drift, for which an integration by parts moves the derivative off
the test function at the cost of the factor
\(\Vab{u_{W_t}}_{L^\infty}
+\Vab{\Div(\rhobar\,u_{W_t})/\rhobar}_{L^\infty}
\le C\Vab{W_t}_{\frob}\in L^2(\tau_{\textup{reg}},T)\).
Hence the time-reversed Galerkin construction of
\zcref{lem:linear-energy-class} applies verbatim and yields the unique
backward energy solution \(\phi\).

For the stationarity condition, take \(V\in L^2(\tau_{\textup{reg}},T)\),
extended by zero to \((0,\tau_{\textup{reg}})\), so that the perturbed
trajectory agrees with the optimal one up to \(\tau_{\textup{reg}}\); by
\zcref{lem:weak-energy-identification}, on \([\tau_{\textup{reg}},T]\) the
states are the energy solutions of
\zcref{prop:state-differentiability}.  For small perturbations the terminal
density stays in the open set of \zcref{ass:ctrl-samewell}, so the chain
rule, \eqref{eq:quadratic-remainder}, and the substitution \(q=\rhobar v\)
of \zcref{lem:relative-linearization} give
\[
DJ(W^{*,T})[V]
=
\int_\dSphe g_T^{*}\,v_T\d\mubar
+\regparam\int_{\tau_{\textup{reg}}}^T W_t^{*,T}:V_t\d t .
\]
The Lions--Magenes product rule shows that
\(t\mapsto\int v_t\phi_t\d\mubar\) is absolutely continuous with derivative
\(\aab{\Bbar_{h^{*}_t}V_t,\phi_t}_{L^2(\mubar)}
=\pab{\int_\dSphe\rho_t^{*,T}\nabla\phi_t\,\sigma^\top\d\omega}:V_t\),
whence
\(\int g_T^{*}v_T\d\mubar
=\int_{\tau_{\textup{reg}}}^T
\pab{\int_\dSphe\rho_t^{*,T}\nabla\phi_t\,\sigma^\top\d\omega}:V_t\d t\).
Since \(DJ(W^{*,T})[V]=0\) for all such \(V\), the fundamental lemma of the
calculus of variations gives \eqref{eq:stationarity-condition}.
\end{proof}

\subsubsection{Adjoint decay and proof of \texorpdfstring{\zcref{thm:exp_turnpike}}{Theorem 3.9}}\label{sec:app:adjoint-decay}

\begin{proposition}[backward decay of the adjoint state]
\label{prop:adjoint-decay}
With \(\phi\) as in \zcref{prop:optimality-system},
\begin{equation}\label{eq:adjoint-decay}
\Vab{\phi_t}_{\Hop^{-1}}
\le
C_{\textup{evo}}\,L_\ell\,
\e^{-\frac{\kappa_{\textup{lin}}}{4}(T-t)},
\qquad
\Vab{\phi_t}_{L^2(\mubar)}
\le
\pab{\eps/c_H}^{1/2}\,C_{\textup{evo}}\,L_\ell\,
\e^{-\frac{\kappa_{\textup{lin}}}{4}(T-t)},
\qquad t\in[\tau_{\textup{reg}},T].
\end{equation}
\end{proposition}

\begin{proof}
By duality, \(\phi_t=U(T,t)^{\dagger}g_T^{*}\), and the Hilbert dual of
\((\Xred(\mubar),\Vab{\cdot}_{\Hop})\) under the \(L^2(\mubar)\) pairing is
\((\Xred(\mubar),\Vab{\cdot}_{\Hop^{-1}})\).  The first bound follows from
\zcref{prop:intermittent-stability} and \eqref{eq:gauge-fixed-gradient}; the
second from the norm equivalence \eqref{eq:cH-coercivity}.
\end{proof}

\begin{lemma}[unit-window gradient estimate]
\label{lem:adjoint-caccioppoli}
There is \(C_\nabla<\infty\), independent of \(T\), such that for
\(s\in[\tau_{\textup{reg}},T-1]\),
\begin{equation}\label{eq:caccioppoli}
\int_s^{s+1}\Vab{\nabla \phi_t}_{L^2(\mubar)}^2\d t
\le
C_\nabla\sup_{t\in[s,s+1]}\Vab{\phi_t}_{L^2(\mubar)}^2 .
\end{equation}
\end{lemma}

\begin{proof}
Testing \eqref{eq:nonautonomous-adjoint} by \(\phi\) gives, by the G\aa rding
bound of the adjoint form,
\[
\frac{\eps}{2}\int_s^{s+1}\Vab{\nabla \phi_t}_{L^2(\mubar)}^2\d t
\le
\frac12\Vab{\phi_{s+1}}_{L^2(\mubar)}^2
+C\int_s^{s+1}
\pab{1+\Vab{W_t^{*,T}}_{\frob}^2}\Vab{\phi_t}_{L^2(\mubar)}^2\d t,
\]
and the control budget \eqref{eq:Ropt-def} bounds the time integral of the
coefficient uniformly in \(T\).
\end{proof}

\begin{corollary}[exponential smallness of the optimal control]
\label{cor:control-tail}
There is \(C_{\textup{tail}}<\infty\), independent of \(T\), such that
\begin{equation}\label{eq:control-window-decay}
\int_s^{s+1}\Vab{W_t^{*,T}}_{\frob}^2\d t
\le
C_{\textup{tail}}\,\e^{-\frac{\kappa_{\textup{lin}}}{2}(T-s)},
\qquad s\in[\tau_{\textup{reg}},T-1],
\end{equation}
and hence, for every \(L\ge1\),
\(\int_{\tau_{\textup{reg}}}^{T-L}\Vab{W_t^{*,T}}_{\frob}^2\d t
\le C_{\textup{tail}}'\,\e^{-\frac{\kappa_{\textup{lin}}}{2}L}\).
\end{corollary}

\begin{proof}
By \eqref{eq:stationarity-condition}, the Cauchy--Schwarz inequality, the
uniform bound \(\Vab{\rho_t^{*,T}}_{L^\infty}\le C_{\textup{reg}}\), and
\(\min_\dSphe\rhobar>0\),
\(\regparam^2\Vab{W_t^{*,T}}_{\frob}^2
\le C\Vab{\nabla \phi_t}_{L^2(\mubar)}^2\).  Integrate over the unit window
and apply \zcref{lem:adjoint-caccioppoli,prop:adjoint-decay}.  Summing the
window bound over unit intervals gives the tail bound.
\end{proof}

\begin{lemma}[persistence in the chart away from the terminal layer]
\label{lem:chart-persistence}
There is \(L_{\textup{end}}\ge1\), independent of \(T\), such that for all
sufficiently large \(T\), the optimal trajectory satisfies
\(\mu_t^{*,T}\in\gC_+(\eta_{\textup{chart}})\) and lies in the density chart
for every \(t\in[s_0^-,T-L_{\textup{end}}]\), where
\(s_0^-\le L\) is the entry time of \zcref{lem:regularized-set-entry}.
\end{lemma}

\begin{proof}
Let \(a_{\textup{loc}}\) and \(C_\sigma\) denote the constants \(a\) and
\(C_\sigma\) in \eqref{eq:strict-diss}.  Choose \(L_{\textup{end}}\) so large
that
\(\eta_{\textup{lin}}+C_\sigma C_{\textup{tail}}'\,
\e^{-\frac{\kappa_{\textup{lin}}}{2}L_{\textup{end}}}<\eta_{\textup{chart}}\).
On the maximal interval on which the trajectory stays in the chart, the
variation-of-constants form of \eqref{eq:strict-diss} and
\zcref{cor:control-tail} give
\(\Energygap(\mu_t^{*,T})
\le\eta_{\textup{lin}}
+C_\sigma\int_{s_0^-}^{T-L_{\textup{end}}}\Vab{W_r^{*,T}}_{\frob}^2\d r
<\eta_{\textup{chart}}\)
for \(t\le T-L_{\textup{end}}\).  Since every such time has a regularizing
window of length \(\tau_{\textup{reg}}\) with control budget
\eqref{eq:Ropt-def}, \zcref{lem:parabolic-chart-entry} keeps the trajectory in
the chart, and
\zcref{prop:two-low-components,lem:same-branch-from-sign-barrier} exclude a
switch of component.  Hence the maximal interval extends to
\(T-L_{\textup{end}}\).
\end{proof}

\begin{proof}[Proof of \zcref{thm:exp_turnpike}]
Let \(s_0^-\le L\) be the entry time of \zcref{lem:regularized-set-entry} and \(L_{\textup{end}}\) the constant of \zcref{lem:chart-persistence}.  For
\(t\in[s_0^-,T-L_{\textup{end}}]\), the variation-of-constants form of
\eqref{eq:strict-diss} on the chart gives
\[
\Energygap(\mu_t^{*,T})
\le
\e^{-a_{\textup{loc}}(t-s_0^-)}\,\eta_{\textup{lin}}
+C_\sigma\int_{s_0^-}^{t}
\e^{-a_{\textup{loc}}(t-r)}\Vab{W_r^{*,T}}_{\frob}^2\d r .
\]
Splitting the convolution into unit windows and applying
\eqref{eq:control-window-decay} on each yields
\(\int_{s_0^-}^{t}\e^{-a_{\textup{loc}}(t-r)}
\Vab{W_r^{*,T}}_{\frob}^2\d r
\le C\,\e^{-\frac12\min\{a_{\textup{loc}},\kappa_{\textup{lin}}/2\}(T-t)}
+C\,\e^{-\frac{a_{\textup{loc}}}{2}(t-s_0^-)}\),
by summing two geometric series.  Since \(s_0^-\le L\) is bounded and
\(\Energygap\le\mathfrak E_{\max}\) on the two boundary layers
\([0,s_0^-]\) and \([T-L_{\textup{end}},T]\) of \(T\)-independent length, the
bound of \zcref{thm:exp_turnpike} follows with
\(a\coloneqq\tfrac12\min\{a_{\textup{loc}},\kappa_{\textup{lin}}/2\}\) and a
constant \(C\) depending only on \(\mathfrak E_{\max}\),
\(a_{\textup{loc}}\), \(\kappa_{\textup{lin}}\), \(C_\sigma\),
\(C_{\textup{tail}}\), \(L\), and \(L_{\textup{end}}\).
\end{proof}

\begin{remark}[explicit size of the turnpike constants]
\label{rem:explicit-turnpike-constants}
All constants above are independent of the horizon \(T\).  The exponential
rate is \(a=\tfrac12\min\{a_{\textup{loc}},\kappa_{\textup{lin}}/2\}\), where
\(a_{\textup{loc}}\) is the local PL rate of \eqref{eq:strict-diss} and, by
\eqref{eq:kappa-lin-lower},
\(\kappa_{\textup{lin}}
\ge\tfrac{c_{\textup{1w}}(A,d)}{2}(1-\lambda_2/\lambda_1)\)
uniformly as \(\eps\downarrow0\); the \(A\)-dependence of both quantities is
governed by the spectral-gap factor \(1-\lambda_2/\lambda_1\) and by the
one-well scale \(c_{\textup{1w}}(A,d)\gtrsim\lambda_1\e^{-\lambda_1}\)
(\zcref{prop:one-well-poincare-scale,rem:uniform-localization-time}).  The
prefactor is controlled by \(C_{\textup{evo}}\), that is, by the bad-time
measure \(M_{\textup{lin}}\), the budgets
\eqref{eq:Ropt-def}--\eqref{eq:Dmax-def}, and the terminal-gradient bound
\(L_\ell\) of \eqref{eq:gauge-fixed-gradient}.
\end{remark}

%% file: contents/numerical_added_cr.tex
\section{Supplementary numerical experiments}

\subsection{Implementation details}
\label{app:numerical-details}

We solve the particle dynamics in \eqref{eq:noisy-particle-dynamics} with
Diffrax \parencite{kidger2022neuraldifferentialequations} and optimize the
control parameters with Adam \parencite{KingmaB14}.  Unless otherwise stated,
we use horizon \(T=80\), inverse temperature \(\beta=1\), and \(N=64\)
particles in the projected \(d=2\) setting with
\(A=\diag(1,\lambda_2)\).  The control input \(u_W\) is
piecewise constant with bin width
\(\Delta t_{\textup{control}}=0.25\), trained for 600 steps with learning rate
\(0.01\).  For the numerical optimization runs reported here we set
\(\regparam=0\); this choice is used only for the exploratory particle
simulations and is separate from the positive-regularization assumptions in the
turnpike theorem.

The initial particles are sampled by first drawing independent standard normal
vectors \(\xi_i\), then projecting
\[
  X_i(0)=\frac{e_1+s\xi_i}{\Vab{e_1+s\xi_i}}
\]
onto the sphere, with \(s=0.3\).
Thus, the initial empirical law is
concentrated in a cap around \(e_1\).  For trajectory diagnostics on \(\Sphe^1\),
we choose the stationary reference direction of each run and represent each
particle by its signed angle \(\theta_i(t)\) relative to that direction.
The energy plots use \(\Energy\) from \eqref{eq:gradient_flow_form}.

For the rate comparison panels, we fit the late-time energy lift to
\(C\exp(-a(T-t))\).  The theoretical rate is computed by solving the stationary
problem on an \(\Sphe^1\) grid, linearizing the dynamics around the stationary
density, and evaluating the rate diagnostic from \zcref{prop:one-step-escape}.
For numerical stability, the rate-comparison runs and the theoretical rate calculation use \(\eps=0.1\), while the trajectory and energy-overlay panels use \(\eps=10^{-4}\) unless otherwise stated.
All experiments were run on an Apple M4 Max MacBook with 128GB RAM.
Code will be released upon publication.

\subsection{Experiment 0: Learning induces a turnpike}
\label{app:numerical-experiment0}
We verify that the turnpike profile is a learned effect rather than a property of the uncontrolled dynamics.
Panels (a)--(c) of \zcref{fig:numerical-turnpike} plot every particle trajectory on $S^1$ in the signed angular coordinate $\theta$, measured relative to a stationary reference angle.
Panel~(a) shows the untrained setting, while panels~(b) and~(c) show the trained counterparts: (b) the align-target trajectory and (c) the BoW trajectory. All three panels share the common parameters $\lambda_2=0.65$ and $\eps=10^{-4}$.
The BoW run additionally uses $\kappa_{\textup{cand}}=2$.
In the no-train case, the particles start from a scattered initial configuration, rapidly concentrate to a single location, and subsequently remain near $\theta=0$.
In contrast, the align-target trajectory shows particles that, after concentrating, remain clustered until the late layers before undergoing a coherent displacement at the terminal time.
The BoW trajectory likewise shows an initial concentration followed by a dispersion near the terminal time.
Panels (d)--(f) display the corresponding energy $\Energy$ in \eqref{eq:gradient_flow_form}.
Panel (d) shows the no-train energy trajectory, while panels (e) and (f) compare the observed energy curve with the theoretical bound in \zcref{prop:one-step-escape}.
Here, to ensure numerical stability, the lower bound in \zcref{prop:one-step-escape} is computed at $\eps=0.1$.
For readability, we display normalized values for both the observed energy and the theoretical bound.
These results demonstrate that training gives rise to the turnpike profile.

\begin{figure}[t]
  \centering
  \begin{subfigure}[t]{0.32\linewidth}
    \centering
    \includegraphics[width=\linewidth]{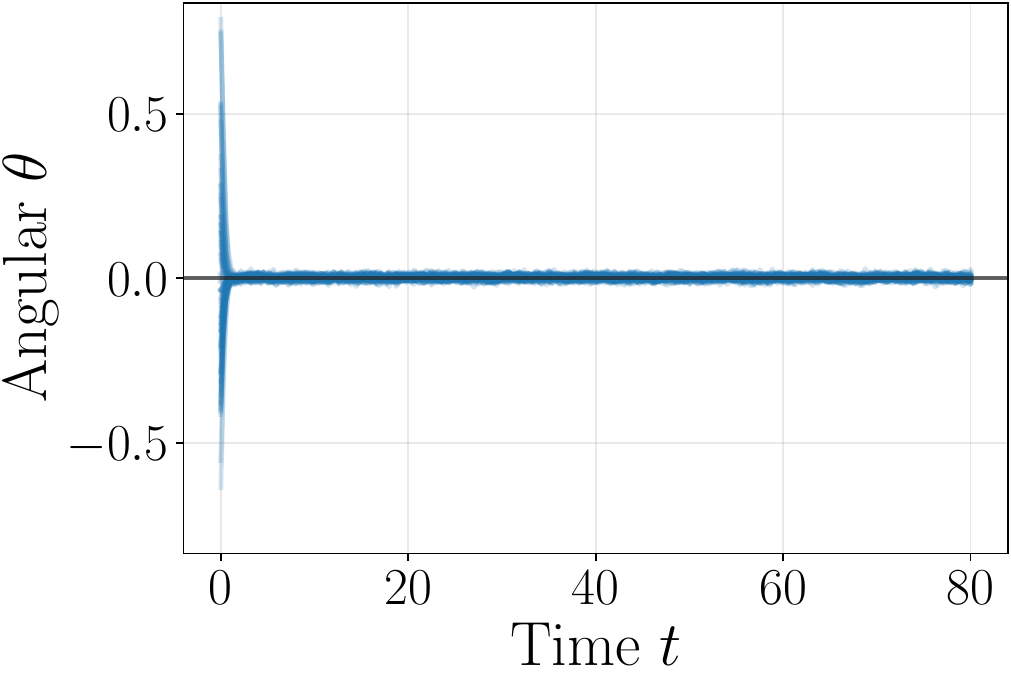}
    \caption{Trajectory: no-train.}
  \end{subfigure}\hfill
  \begin{subfigure}[t]{0.32\linewidth}
    \centering
    \includegraphics[width=\linewidth]{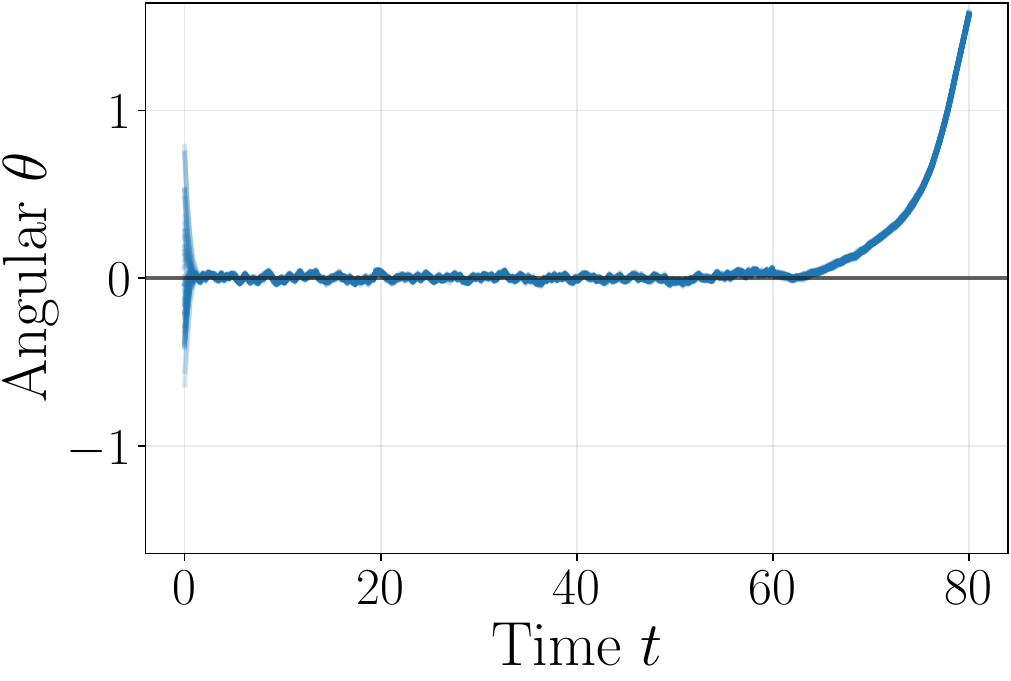}
    \caption{Trajectory: align-target.}
  \end{subfigure}\hfill
  \begin{subfigure}[t]{0.32\linewidth}
    \centering
    \includegraphics[width=\linewidth]{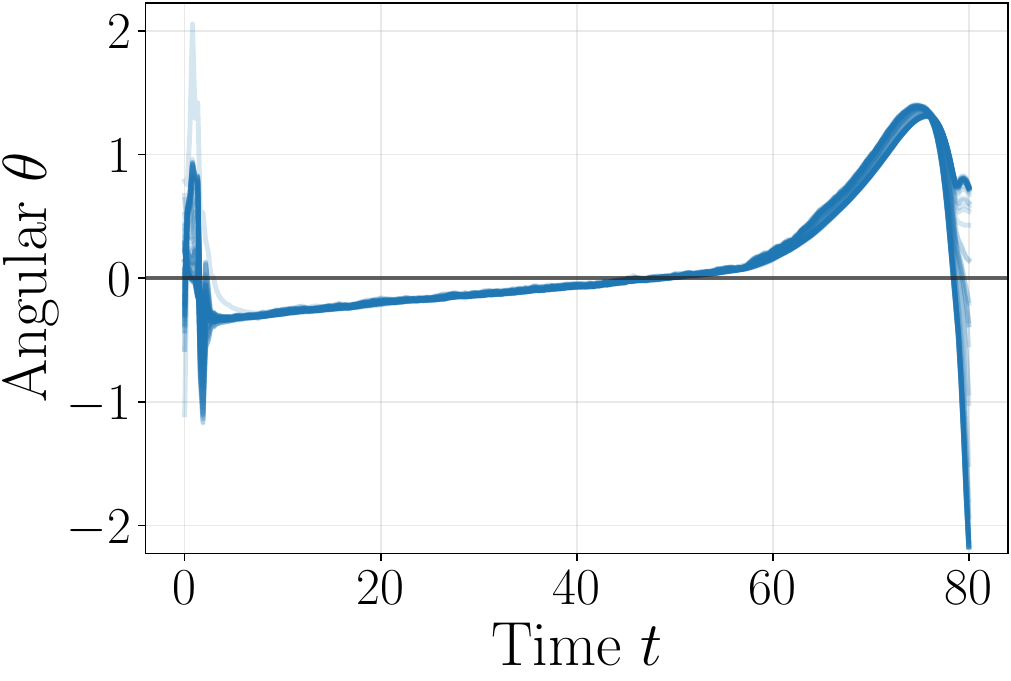}
    \caption{Trajectory: BoW.}
  \end{subfigure}
  \\
  \vspace{0.5em}
  \begin{subfigure}[t]{0.32\linewidth}
    \centering
    \includegraphics[width=\linewidth]{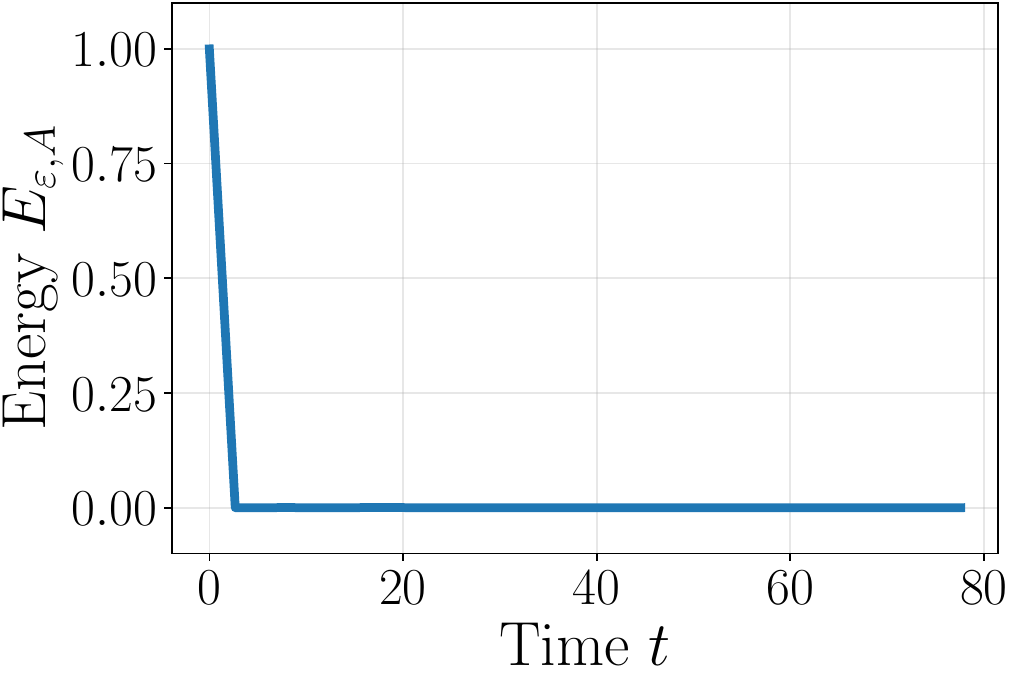}
    \caption{Energy: no-train.}
  \end{subfigure}\hfill
  \begin{subfigure}[t]{0.32\linewidth}
    \centering
    \includegraphics[width=\linewidth]{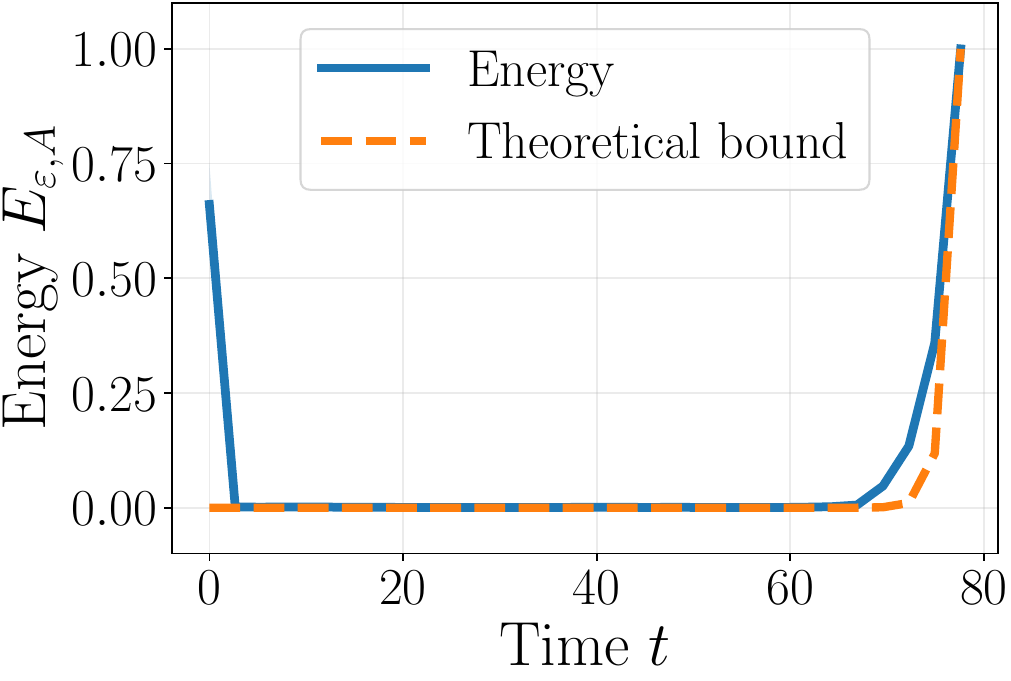}
    \caption{Energy: align-target.}
  \end{subfigure}\hfill
  \begin{subfigure}[t]{0.32\linewidth}
    \centering
    \includegraphics[width=\linewidth]{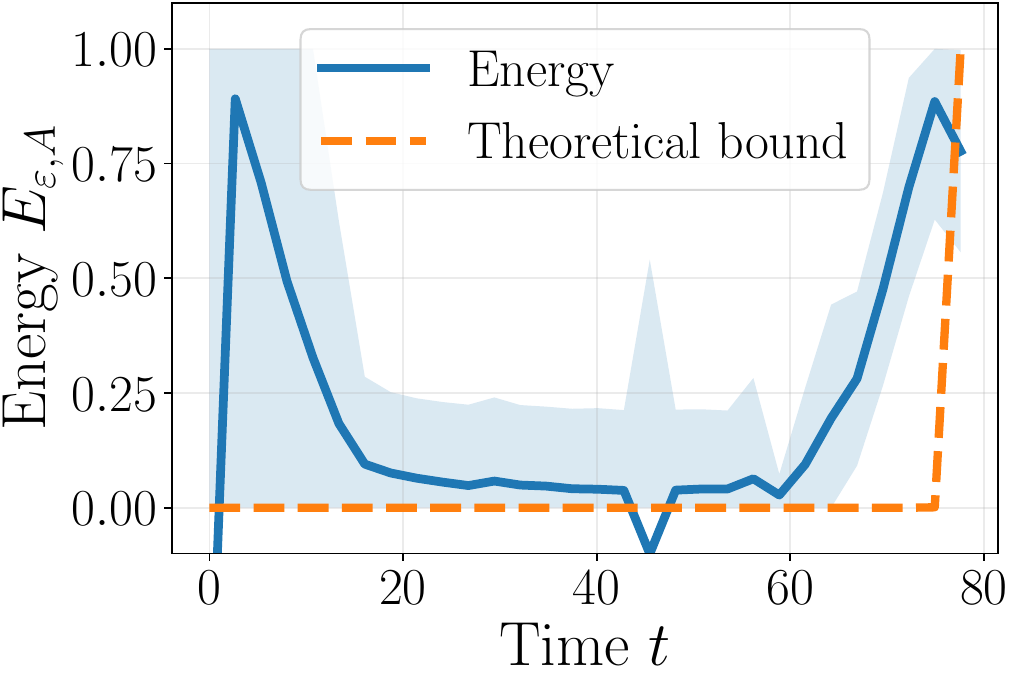}
    \caption{Energy: BoW.}
  \end{subfigure}
  \caption{Particle trajectories and energy time evolution. Panels (a)--(c) show all-particle angular trajectories in stationary-reference coordinates for no-train, align-target, and BoW runs. Panel (d) shows the corresponding no-train energy trajectory. Panels (e) and (f) show the normalized observed energy against the theoretical bound for representative align-target and BoW runs.}
  \label{fig:numerical-turnpike}
\end{figure}

\subsection{Robustness protocol and full tables}
\label{app:numerical-robustness}
This appendix records the protocol behind the robustness summary in \zcref{sec:numerical}.
The sweeps are summarized in \zcref{tab:robustness}; the per-phase counts and the fitted quantities follow.
At the reference setting, 20/20 runs satisfy the criterion, with fitted terminal rate $0.246\pm0.007$ and terminal angular displacement $1.554\pm0.010$ (mean $\pm$ SD).
At $\eps=10^{-1}$ the clustering and the plateau degrade because the equilibrium cap width $\sigma_\eps=\sqrt{\eps/(\e^{\lambda_1}\lambda_1)}\approx0.19$ becomes comparable to the initial spread $\approx0.3$, while the terminal response persists in 20/20 runs.
Across horizons the plateau duration grows ($32.0\to71.2\to109.7$) while the initial and terminal widths vary mildly, consistent with boundary layers of $T$-independent width; across targets the plateau is robust whereas the strength and timing of the escape are target-dependent; under a global random initialization, the profile requires early entry into a single low-energy component, and regimes with additional metastable states lie outside our theorems.
\begin{table}[h]
\centering
\caption{Robustness of the three-phase profile (20 runs per configuration; the fixed criterion is stated in \zcref{app:numerical-robustness}).}
\label{tab:robustness}
\small
\begin{tabular}{lll}
\toprule
Sweep & Range & Full profile \\
\midrule
Seeds (reference setting) & 20 seeds & 20/20 \\
Diffusion $\eps$ & $\Bab{0,\,10^{-4},\,10^{-3},\,10^{-2}}$ & 79/80 \\
 & $10^{-1}$ & 0/20 (terminal response 20/20) \\
Horizon $T$ & $\Bab{40,80,120}$ & 60/60 \\
Particles $N$ & $\Bab{32,64,128}$ & 60/60 \\
Regularization $\regparam$ & $0.01$ & 20/20 \\
Targets $\kappa_{\textup{cand}}$ & $\Bab{2,3,5,7,8,9,10}$ & 138/140 (plateau 140/140) \\
\bottomrule
\end{tabular}
\end{table}

\paragraph{Fixed three-phase criterion.}
Let $E(t)$ denote the energy $\Energy$ reconstructed along a run, $R_E\coloneqq\max_tE(t)-\min_tE(t)$, and $E_{\textup{pl}}\coloneqq\operatorname{median}_{t\in[0.4T,0.6T]}E(t)$.
A run satisfies the criterion when
\[
\frac{E(0)-E_{\textup{pl}}}{R_E}\ge0.1,
\qquad
\frac{\operatorname{SD}_{t\in[0.4T,0.6T]}E(t)}{R_E}\le0.05,
\qquad
\frac{E(T)-E_{\textup{pl}}}{R_E}\ge0.1,
\]
where the standard deviation is the sample standard deviation on the interior window; no smoothing is applied, and the criterion is used without retuning across all sweeps.
The initial width is the first time in $[0,0.5T]$ at which $90$\% of the initial-to-plateau decrease has occurred; the terminal width is $T-t_{\textup{esc}}$, where $t_{\textup{esc}}$ is the first time in $[0.5T,T]$ at which $10$\% of the plateau-to-terminal increase is reached; the plateau duration is the interval between these crossings.
The terminal rate $a$ is fitted on $[0.7T,T]$ from $E(t)-E_{\textup{pl}}\approx C\e^{-a(T-t)}$ whenever the terminal lift is positive.
Since $\eps$ enters $\Energy$, only within-run normalized quantities are compared across different $\eps$.

\paragraph{Full tables.}
Diffusion sweep (20 runs per value; per-phase counts):
\begin{center}
\small
\begin{tabular}{r|rrrr}
\toprule
$\eps$ & initial clustering & plateau & terminal response & full profile \\
\midrule
$0$ & 20/20 & 20/20 & 20/20 & 20/20 \\
$10^{-4}$ & 20/20 & 20/20 & 20/20 & 20/20 \\
$10^{-3}$ & 20/20 & 20/20 & 20/20 & 20/20 \\
$10^{-2}$ & 20/20 & 19/20 & 20/20 & 19/20 \\
$10^{-1}$ & 0/20 & 13/20 & 20/20 & 0/20 \\
\bottomrule
\end{tabular}
\end{center}
Horizon sweep (20 runs per value, all satisfying the full criterion):
\begin{center}
\small
\begin{tabular}{r|rrr}
\toprule
$T$ & plateau duration & initial width & terminal width \\
\midrule
$40$ & $32.0$ & $\approx2.67$ & $5.33$ \\
$80$ & $71.2$ & $\approx2.67$ & $6.15$ \\
$120$ & $109.7$ & $\approx2.67$ & $7.62$ \\
\bottomrule
\end{tabular}
\end{center}
Target sweep: for $\kappa_{\textup{cand}}\in\Bab{2,3,5,7,8,9,10}$ (20 runs each), the interior-plateau criterion holds in $140/140$ runs and the full criterion in $138/140$, while the mean terminal angular displacement decreases from $0.763\pm0.343$ at $\kappa_{\textup{cand}}=2$ to $0.169\pm0.006$ at $\kappa_{\textup{cand}}=10$.
Basin sensitivity: replacing the cap initialization by a global random initialization and classifying plateaus by the plateau-center norm (single well: $\ge0.9$; mixed well: $\le0.5$; intermediate otherwise), $13$ of $20$ runs entered a single-well plateau, $6$ remained mixed, and $1$ was intermediate; the full criterion held in $11/13$ single-well and $0/6$ mixed-well runs.
Supplementary checks: $N\in\Bab{32,64,128}$ give $20/20$ each; $\regparam=0.01$ gives $20/20$ with terminal displacement $1.560\pm0.008$ (versus $1.554\pm0.010$ unregularized); an exploratory particle-model screening over $d\in\Bab{2,3,8,16}$ gives $5/5$ at each dimension.